\documentclass{article}

\usepackage{fullpage}
\usepackage[square,comma,authoryear]{natbib}

\author{%
  Yuxing Liu\thanks{Equal contribution.},
  \quad
  Rui Pan\footnotemark[1],
  \quad
  Tong Zhang
  \\
  \\
  University of Illinois Urbana-Champaign \\
  \small\texttt{yuxing6,ruip4,tozhang@illinois.edu}
}

\date{}

\usepackage[utf8]{inputenc} % allow utf-8 input
\usepackage[T1]{fontenc}    % use 8-bit T1 fonts
\usepackage{hyperref}       % hyperlinks
\usepackage{url}            % simple URL typesetting
\usepackage{booktabs}       % professional-quality tables
\usepackage{amsfonts}       % blackboard math symbols
\usepackage{nicefrac}       % compact symbols for 1/2, etc.
\usepackage{microtype}      % microtypography
\usepackage{xcolor}         % colors

\usepackage{color,colortbl} 
\definecolor{lightblue}{rgb}{0.88,1,1}

\usepackage{amsmath}
\usepackage{amssymb}
\usepackage{mathtools}
\usepackage{amsthm}
\usepackage{multirow}
\usepackage{bm}
\usepackage{layouts}
\usepackage{wrapfig}

\usepackage{algorithm}
\usepackage{algorithmic}

\usepackage{relsize}
\usepackage{anyfontsize}

\usepackage{makecell}
\usepackage{threeparttable}

\theoremstyle{plain}
\newtheorem{theorem}{Theorem}[section]

\newtheorem{lemma}[theorem]{Lemma}

\theoremstyle{definition}

\newtheorem{assumption}[theorem]{Assumption}
\theoremstyle{remark}
\newtheorem{remark}[theorem]{Remark}

\theoremstyle{definition}
\newtheorem{example}[theorem]{Example}

\newcommand{\Norm}[1]{\left\|#1\right\|}

\newcommand{\dotprod}[2]{\left\langle#1,#2\right\rangle}
\newcommand{\Abs}[1]{\left\vert #1 \right\vert}
\newcommand{\tr}[1]{{\rm tr}\left( #1 \right)}

\newcommand{\cW}{\mathcal{W}}
\newcommand{\cF}{\mathcal{F}}
\newcommand{\cB}{\mathcal{B}}

\newcommand{\bw}{\mathbf{w}}
\newcommand{\bx}{\mathbf{x}}
\newcommand{\bI}{\mathbf{I}}
\newcommand{\bH}{\mathbf{H}}

\newcommand{\bv}{\mathbf{v}}
\newcommand{\bLa}{\mathbf{\Lambda}}
\newcommand{\bla}{\mathbf{\lambda}}
\newcommand{\bLat}{\mathbf{\Tilde{\Lambda}}}
\newcommand{\blat}{\mathbf{\Tilde{\lambda}}}
\newcommand{\bz}{\mathbf{z}}
\newcommand{\bn}{\mathbf{n}}
\newcommand{\bg}{\mathbf{g}}

\newcommand{\EE}{\mathbb{E}}
\newcommand{\cO}{\mathcal{O}}
\newcommand{\RR}{\mathbb{R}}

\newcommand{\bL}{\mathbf{L}}
\newcommand{\bsigma}{\bm{\sigma}}

\title{\textbf{AdaGrad under Anisotropic Smoothness}}

\begin{document}

\maketitle

\begin{abstract}
Adaptive gradient methods have been widely adopted in training large-scale deep neural networks, especially large foundation models. Despite the huge success in practice, their theoretical advantages over classical gradient methods with uniform step sizes across all coordinates (e.g. SGD) have not been fully understood, especially in the large batch-size setting commonly used in practice. This is because the only theoretical result that can demonstrate this benefit was obtained in the original paper of Adagrad for convex nonsmooth objective functions, which is insufficient for large batch algorithms. 
In this work, we attempt to resolve this gap between theory and practice
by proposing a novel anisotropic generalized smoothness assumption and providing corresponding analyses of Adagrad. It is shown that under anisotropic smoothness and noise conditions, AdaGrad can achieve faster convergence guarantees in terms of better dimensional dependence than algorithms with uniform step sizes across all coordinates.
% We further present detailed comparisons between SGD and Adagrad to provide a better understanding of the benefits of adaptive gradient methods.
Experiments in logistic regression and instruction following fine-tuning tasks provide strong evidence to support our novel assumption and theoretical analysis.
\end{abstract}
% TODO: check abstract, about convex and the smoothness condition.

\section{Introduction}
% \textcolor{blue}{We may need to change the title.}

To solve the stochastic optimization problem
\begin{align}\label{eq:problem_formulation}
    \min_{\bw\in\RR^d} f(\bw) \triangleq \EE_\xi[f(\bw;\xi)],
\end{align}
adaptive gradient methods~\citep{duchi2011adaptive,zeiler2012adadelta,tieleman2012lecture,kingma2014adam,loshchilov2017decoupled} are among the most popular methods. These methods have gained incredible importance from their superior efficiency, especially in training large foundation models, where large batch sizes are commonly employed. 
One of the most important features of adaptive gradient methods is the coordinate-wise anisotropic step size.
Take Adagrad~\citep{duchi2011adaptive,streeter2010less}, the first adaptive gradient method as an example, which writes
\begin{align*}
    \bw_{t+1} = \bw_t - \eta_t \bLa_t^{-1}\bg_t,
\end{align*}
where $\bg_t$ notes the stochastic gradient estimation obtained at $\bw_t$ and $\bLa_t \neq \bI_d$ is a diagonal matrix, representing the square root of a coordinate-wise summation of former gradient estimations. 
Despite the huge success in practice, as also noted in~\citet{li2021frequency,kunstner2023noise,li2024convergence}, the theoretical understanding of when and why adaptive gradient methods enjoy acceleration over classical gradient algorithms with uniform step size across all coordinates such as SGD, is still limited.
On the theory side, the original Adagrad paper~\cite{duchi2011adaptive,streeter2010less} shows the superiority of Adagrad over SGD in the convex non-smooth scheme, suggesting that if the gradients are sparse and the predictor is limited in an appropriate convex set, Adagrad can converge faster in terms of better dimensional dependence.
However, as the convergence rates in nonsmooth settings rely on the scale of stochastic gradients, their results can be insufficient for the smooth and large-batch training scheme, % TODO
which is a realistic setting gaining extensive focus. This is because the scale of stochastic gradients does not decrease linearly as the batch size $M$ increases. Therefore, if we fix the stochastic gradient computation number $N=MT$, where $M$ is the batch size and $T$ is the total iteration number, such that increasing $M$ leads to a linear decrease in $T$, the original convergence rate $\cO( 1/\sqrt{T} )$ may be unsatisfying. 
Fine analysis in smooth settings can solve this problem as the obtained convergence results in this case depend on the gradient variance, which decreases linearly as the batch size $M$ increases.
Many theoretical papers have also conducted analysis of popular adaptive gradient methods (e.g. Adagrad, Adam) in smooth settings under smoothness and noise assumptions with respect to $\Norm{\cdot}_2$. % cite
However, to the best of our knowledge, their proven results have no better or even worse dimensional dependence than the standard convergence results of SGD in the same settings. The existence of this gap between theory and practice makes us wonder: \textit{Can we obtain better theoretical guarantees of adaptive gradient methods to explain their practical success, especially in large batch settings?}

To answer this question, it would be a good starting point to revisit the well-known insight of adaptive gradient methods.
Intuitively, Adagrad shines when the problem is highly imbalanced, i.e., coordinates of the gradients have very different magnitudes. It schedules a larger learning rate (compared to SGD) for coordinates with small gradients and thus converges faster.
This implies that the performance of adaptive gradient methods relies largely on the anisotropic structure of the problem. However, existing standard assumptions fail to describe this property. Take the standard smoothness~\citep{nesterov2018lectures} as an example, which assumes a constant $L>0$ such that $-L\bI_d \preceq \nabla^2 f(\bw) \preceq L\bI_d$ for all $\bw$. It is evident that this assumption is coordinate-wise equivalent and cannot reflect the imbalance between coordinates, thus the benefits of adaptive gradient methods are hidden, as in the results of~\citet{vaswani2020adaptive,defossez2020simple,wang2023convergence}. 

To better explore the provable benefits of adaptive gradient methods, it is necessary to employ appropriate assumptions that can better describe the structure of models.
In this paper, we consider the anisotropic smoothness and noise assumptions. Based on these assumptions, we give novel convergence analysis of Adagrad and show that Adagrad can adapt well to the problem's anisotropic nature. 
To extend our results to a more general setting, we additionally introduce a novel generalized anisotropic smoothness assumption and provide corresponding analysis of AdaGrad.
By comparisons between the convergence results of Adagrad and gradient methods with step sizes across all coordinates (SGD and AdaGrad-Norm~\citep{streeter2010less} as two representatives), we justify the power of adaptive gradient methods when the problem has highly anisotropic nature.
\footnote{Note that the comparison between upper bounds is commonly adopted in past literature, e.g.~\citep{bernstein2018signsgd,allen2018katyusha}, though theoretically it only suggests the worst-case performance comparison of algorithms.
}

Our contributions are summarized as follows:

\begin{enumerate}
    \item We present a fine-grained analysis of Adagrad under anisotropic smoothness and noise assumptions, leading to novel theoretical convergence guarantees for Adagrad. We further introduce a generalized form of anisotropic smoothness, extending these results to more practical settings.
    \item We discuss how the convergence results indicate the potential benefits of AdaGrad compared to algorithms with coordinate-wisely uniform step sizes such as SGD and AdaGrad-Norm.
    \item Experiments on logistic regressions on real-world datasets and instruction-following fine-tuning with GPT-2 provide concrete empirical evidence to support our claims.
\end{enumerate}

\section{Related Work}

\paragraph{Adaptive gradient methods.} Adaptive gradient methods are popular optimizers for training neural networks~\citep{choi2019empirical,vani2019experimental}. Among them, Adagrad~\citep{duchi2011adaptive,streeter2010less} is considered to be the first adaptive gradient method in this branch, which was originally proposed to solve non-smooth online convex optimization problems. Ever since its first appearance, numerous adaptive gradient methods have emerged, such as RMSProp~\citep{tieleman2012lecture}, AdaDelta~\citep{zeiler2012adadelta}, Adam~\citep{kingma2014adam}, SC-Adagrad~\citep{mukkamala2017variants}, AdamW~\citep{loshchilov2017decoupled}, WNGrad~\citep{wu2018wngrad}, AMSGrad~\citep{reddi2019convergence}, SAdam~\citep{wang2019sadam} to name a few. They have revolutionized the field of deep neural network training, and are still widely adopted in the literature of large language models~\citep{radford2019gpt2,touvron2023llama,touvron2023llama2}.

\paragraph{Convergence results of SGD.} Stochastic gradient descent (SGD), a popular optimizer for many real-world tasks, has been extensively studied in the literature~\citep{robbins1951stochastic,nemirovski1978cezari,nemirovskij1983problem,nemirovski2009robust,hazan2007logarithmic,rakhlin2011making,shamir2013stochastic}. 
In the smooth stochastic optimization scheme,~\citet{moulines2011non} and~\citet{ghadimi2013stochastic} gave analysis of SGD under standard assumptions in the convex and nonconvex settings separately. More recently, \citet{zhang2019gradient} introduced the $(L_0,L_1)$-smoothness to relax the standard smoothness assumption and studied the convergence of SGD and clipped SGD under this assumption. A line of following work has been conducted based on this assumption after that~\citep{zhang2020improved,chen2020understanding,gorbunov2020stochastic,qian2021understanding,koloskova2023revisiting}.

\paragraph{Convergence results of Adagrad.} As the pioneering work of Adagrad, \citet{duchi2011adaptive} provided an analysis for Adagrad's convergence guarantees in online convex optimization settings, showing acceleration over SGD when the gradients are sparse. However, the presented analysis is only for general nonsmooth objectives, which cannot explain Adagrad's effectiveness under large batch settings. 
\citet{levy2018online} proved Adagrad-Norm~\citep{streeter2010less}, a non-coordinate-wise variant of Adagrad, achieves the same convergence rate as SGD in the convex smooth optimization setting. \citet{vaswani2020adaptive} followed this result and studied Adagrad in the interpolation scheme. However, their results show no better or even worse dimensional dependence than Adagrad-Norm or SGD.
On another side, the convergence of non-convex Adagrad or its close variants in the smooth nonconvex setting has been extensively studied~\citep{li2019convergence,defossez2020simple,ward2020adagrad,faw2022power,faw2023beyond,wang2023convergence,kavis2022high,liu2023high,attia2023sgd,hong2024revisiting}. Among them, \citet{ward2020adagrad} obtained the $\cO(1/\sqrt{T})$ rate of Adagrad-Norm under global bounded gradient assumption. More recently, \citet{faw2023beyond,wang2023convergence} further improved the result to hold under the $(L_0,L_1)$-smoothness assumptions. 
There are also extensive studies focused on the convergence of Adam and its close variants, we list some of them here for reference: \citep{reddi2019convergence,de2018convergence,defossez2020simple,guo2021novel,zhang2022adam,wang2022provable,li2024convergence,wang2024closing,hong2024convergence}.

\paragraph{Theoretical understanding of adaptive gradient methods:}
Surprisingly, though the community tends to have a common sense that adaptive gradient methods converge much faster than algorithms with uniform step sizes in specific tasks, theoretical explanations on when and why this happens are relatively rare compared to extensive empirical studies. It is worth pointing out that among all the above-mentioned works on theoretical analysis of adaptive gradient methods, only the initial Adagrad analysis~\citep{duchi2011adaptive} clearly shows this acceleration in terms of possibly lower dimensional dependence in the online convex programming setting. 
How adaptive gradient methods can help accelerate convergence in smooth or large batch stochastic optimization still remains unclear.
We also notice that \citet{cesa2007improved,orabona2015scale,orabona2018scale,zhuang2022understanding} mentioned the intuition of scale-free algorithms and \cite{zhuang2022understanding} demonstrates the connection between scale-freeness of adaptive gradient methods and better condition number dependence. More recently, \citet{zhang2024transformers,das2024towards} also investigated why Adam is effective in certain tasks compared to SGD and gives theoretical results in quadratics settings.
However, their results are more intuitive with very restrictive analysis that are insufficient to explain real-world tasks.
In another line of work, \citet{bernstein2018signsgd,wu2020dissecting,kunstner2023noise,liu2023sophia} suggests a relation between the benefits of adaptive gradient methods and their sign-based nature. These intuitions may shed light on the theoretical understanding of adaptive gradient methods. In parallel with our work, a concurrent study~\citep{maladkar2024convergence} investigates a similar approach, compared to which we further generalize the assumptions to obtain a more broadly applicable analysis.

\paragraph{Large batch training:}
Large batch training enjoys extensive focus for its practical impact. It has been observed that large batch sizes are beneficial for accelerating large model training in practice~\citep{you2017large,you2017100,you2018imagenet,you2019large,pan2022extremebert}. 
Furthermore, large batch training is a valuable acceleration technique in distributed machine learning~\citep{verbraeken2020survey} and pretraining~\citep{zhou2023comprehensive}, where adaptive gradient methods are popular. In particular, it is a common practice in pretraining of large language models to combine large batch sizes with adaptive gradient methods~\citep{radford2019gpt2,touvron2023llama,touvron2023llama2}, where thousands of GPUs can be utilized and a fixed number of batch sizes will be assigned to each GPU, rendering the total batch size extremely large. 

\section{Preliminaries}
\subsection{Notations}
We use $\odot$ to denote the coordinate-wise product of vectors and without leading to confusion, $\sqrt{\cdot}$ is sometimes used to denote the coordinate-wise square root of a vector or diagonal matrix. 
Let $\bH\in \RR^{d\times d}$ be a symmetric positive definite matrix, we denote the vector norm induced by $\bH$ that
$
    \Norm{\bw}_\bH^2 \triangleq \bw^\top \bH \bw.
$
With a slightly abuse of notation, for a vector $\mathbf{h}\in \RR^d_+$, we denote $\Norm{\bw}_{\mathbf{h}}^2 = \sum_{j=1}^d \mathbf{h}_j \bw_j^2$. For a symmetric positive definite matrix $\bH\in \RR^{d\times d}$ and convex set $\mathcal{\cW}$ we introduce the $\bH$-based projection operator $\Pi^\bH_\cW(\cdot)$ such that
$
    \Pi^\bH_\cW(\bw) = {\text{argmin}}_{\bz\in\cW} \Norm{\bz - \bw}_\bH^2.
$
As discussed in~\citet{hazan2007logarithmic}, the projection is a convex program and can be solved efficiently.

Let us denote $\nabla_\bw f(\bw;\xi)$ the stochastic gradient oracle at $\bw$ and $\bg_t$ the gradient estimation employed at $\bw_t$. $\cF_t\triangleq\sigma(\bg_0,\cdots,\bg_{t-1})$ stands for the sigma field of the gradient estimators from the first iteration to the $t-1$ iteration. We use $\EE[\cdot]$ to denote total expectation over $\cF_{T}$ where $T$ is the maximum iteration number  and $\EE_t[\cdot]$ as an abbreviation of the conditional expectation $\EE[\cdot|\cF_t]$.

\subsection{Problem Settings and Assumptions}\label{sec:assumptions}
We study the stochastic optimization problem~\eqref{eq:problem_formulation},
where we can only access the stochastic gradient oracle $\nabla f(\bw;\xi)$ at $\bw$.
Throughout this paper, we consider the following assumptions.
\begin{assumption}[Convexity]\label{asm:convex_set}
    $f(\cdot)$ is convex. For convex cases, we search solution in a closed convex set $\cW\subseteq \RR^d$ such that there exists at least one optimal solution $\bw_*\in\cW$ and
    \begin{equation}\label{eq:asm_convex_set}
    \begin{aligned}
        &\max_{\bw, \bw'\in\cW} \Norm{\bw-\bw'}_\infty \le D_\infty \quad \text{and} \quad
        \max_{\bw, \bw'\in\cW} \Norm{\bw-\bw'}_2 \le D_2.
    \end{aligned}
    \end{equation}
\end{assumption}

\begin{assumption}[Lower bounded]\label{asm:f_lower_bound}
    There exists constant $f^*$ such that for $\bw \in \RR^d$, $f(\bw) \ge f^*$.
\end{assumption}

\begin{assumption}[Anisotropic $\bL$-smoothness]\label{asm:smooth}
    There exists a positive vector $\bL = [L_1,\dots, L_d] \in \RR_+^d$ such that $f(\cdot)$ is $\bL$-smooth, namely, for $\bw,\bw' \in \RR^d$,
    \begin{align}\label{eq:asm_smooth}
        \Norm{\nabla f(\bw) - \nabla f(\bw')}_{\bL^{-1}} \le \Norm{\bw - \bw'}_\bL .
    \end{align}
\end{assumption}

Assumption~\ref{asm:convex_set} and \ref{asm:f_lower_bound} are standard for convex and nonconvex problems, respectively.
Assumption~\ref{asm:smooth} is an anisotropic generalization of the smoothness condition, which has also been employed in a line of work on SignSGD~\citep{bernstein2018signsgd,bernstein2019signsgd}. Its intuition can be understood in the following manner. When the loss function $f(\cdot)$ is twice-differentiable, Assumption~\ref{asm:smooth} is equivalent to $\nabla^2 f \preceq \text{diag}(\bL)$ that implies the standard smoothness assumption by $L=\Norm{\bL}_\infty$.
When the Hessian of $f(\cdot)$ is imbalanced, namely, coordinates have very different scales, $\bL$ can be adapted to this imbalanced distribution and can describe a tighter upper bound of the Hessian of $f(\cdot)$, resulting in $\Norm{\bL}_1 \ll Ld$. This benefit is realistic as the highly imbalanced spectrum distribution of the Hessian has been widely observed in multiple circumstances~\citep{sagun2016eigenvalues,arjevani2020analytic,pan2021eigencurve}.
The power of this adaptation shines when adaptive gradient methods are employed.

We consider the standard stochastic approximation framework~\citep{kushner2012stochastic} and denote the gradient noise at $\bw$ to be
$
    \bn(\bw;\xi) \triangleq \nabla f(\bw) - \nabla_\bw f(\bw;\xi).
$
We assume the following assumptions on gradient noise throughout this paper.
\begin{assumption}[Unbiased Independent gradient]\label{asm:unbiased_gradient}
    Each $\nabla f(\bw;\xi)$ is independently drawn and 
    \begin{align}\label{eq:asm_unbiased_gradient}
        \EE\left[ \bn(\bw;\xi) \right] = 0.
    \end{align}
\end{assumption}

\begin{assumption}[Anisotropic noise]\label{asm:anisotropic_noise}
    There exists positive vector $\bsigma = [\bsigma_1,\dots,\bsigma_d]\in\RR_+^{d}$ such that
    \begin{align}\label{eq:asm_anisotropic_noise}
        \EE\left[ \bn_j(\bw;\xi)^2 \right] \le \bsigma_j^2 \quad \text{for all} \quad j\in[d] .
    \end{align}
\end{assumption}
Note that Assumption~\ref{asm:anisotropic_noise} implies the standard bounded noise assumption by $\EE[ \Norm{\bn(\bw;\xi)}_2^2 ] \le \Norm{\bsigma}_2^2$. Intuitively, it upper bounds all the coordinates of $\bn(\bw;\xi)$ instead of only the norm and gives more detailed information on the scale of noise. Generally, Combining Assumption~\ref{asm:smooth} and \ref{asm:anisotropic_noise} allows more fine-grained analysis, which can take the sparsity of the model into account. 
Note that the anisotropic noise assumption has also been explored in other lines of studies on sign-based methods~\citep{bernstein2018signsgd,bernstein2019signsgd,crawshaw2022robustness} and quadratics~\citep{dieuleveut2017harder,ge2019step,pan2021eigencurve,pan2023accelerated}. It is also closely related to Assumption 2 in~\citet{zhang2020adaptive}, where the authors attempt to model the heavy-tailedness of neural networks.

\begin{algorithm}[tb]
   \caption{Adagrad}
   \label{alg:adagrad}
\begin{algorithmic}[1]
   \STATE {\bfseries Input:} $\bw_0\in \RR^d$, $\{\eta_t\}_{t=0}^{T-1}\in\RR$, $\epsilon\in\RR$, and batch size $M\in \mathbb{N}$
   \STATE Initialize $\bv_{-1}=\epsilon^2 \mathbf{1}_d$
   \FOR{$t=0$ {\bfseries to} $T-1$}
   \STATE Sample mini-batch $\cB_t$ with $\Abs{\cB_t} \equiv M$ uniformly
   \STATE $\bg_t = \frac{1}{M} \sum_{\xi\in\cB_t} \nabla_\bw f(\bw_t;\xi)$
   \STATE $\bv_t = \bv_{t-1} + (\bg_t \odot \bg_t)$
   \STATE $\bLa_t = \text{diag}(\sqrt{\bv_t})$

   \STATE 
   \textbf{Option I: } 
   $\bw_{t+1} = \Pi^{\bLa_t}_\cW(\bw_t - \eta_t \bLa_t^{-1}\bg_t)$
   \STATE 
   \textbf{Option II: } 
   $\bw_{t+1} = \bw_t - \eta_t \bLa_t^{-1}\bg_t$
   % \IF{$x_i > x_{i+1}$}
   % \STATE Swap $x_i$ and $x_{i+1}$
   % \STATE $noChange = false$
   % \ENDIF
   \ENDFOR
   \STATE {\bfseries Output:} $1/T\sum_{t=0}^{T-1}\bw_t$
\end{algorithmic}
\end{algorithm}

\section{AdaGrad with Anisotropic Assumptions}\label{sec:AdaGrad}
How to accelerate the convergence of SGD has been a fundamental problem in training machine learning models. One possible approach is to lower the implicit dependence of dimension $d$. 
As discussed in~\citet{nguyen2019tight}, in smooth strongly convex settings, the lower bound of SGD can be $d$ times larger than a wider class of algorithms including adaptive gradient methods.
The original Adagrad paper~\citep{duchi2011adaptive} showed that Adagrad has better dimension dependence than SGD when the stochastic gradients are generally sparse in the non-smooth convex scheme.
However, in the stochastic smooth optimization scheme, to the best of our knowledge, existing theoretical results are insufficient to account for the benefit of adaptive gradient methods.
We attempt to fill the gap, equipped with the anisotropic Assumptions~\ref{asm:smooth}~and~\ref{asm:anisotropic_noise}.

\subsection{Convex Cases}
For a warmup, we present results in the convex case to first gain some intuition on how the anisotropic assumptions can better describe the convergence of AdaGrad in the large batch setting. 

\begin{theorem}[Convex convergence of Adagrad]\label{thm:smooth_convex_adagrad}
    Under Assumptions~\ref{asm:convex_set}, \ref{asm:smooth}, \ref{asm:unbiased_gradient}, \ref{asm:anisotropic_noise} with $\bL_1 = 0$, for the sequence $\{\bw_t\}_{t=1}^T$ generated by Adagrad (Algorithm~\ref{alg:adagrad} with option I) with constant step size $\eta_t\equiv\eta = D_\infty$,
    it holds that for $\Bar{\bw}_T=(1/T)\sum_{t=0}^{T-1} f(\bw_t)$,
    \begin{align*}
        \EE\left[ f(\bar{\bw}_T) - f(\bw_*) \right] =
        \cO\left( \frac{D_\infty \Norm{\bsigma}_1}{\sqrt{MT}} + \frac{\Norm{\bL}_1 D_\infty^2}{T} \right) 
        + \cO\left( \frac{\epsilon D_2^2}{D_\infty T} \right) .
    \end{align*}
\end{theorem}

Note that $\epsilon$ is employed mainly for numerical stability and is commonly very small (in order $10^{-10}$ by default).
Theorem~\ref{thm:smooth_convex_adagrad} shows that the convergence of AdaGrad depends on $\Norm{\bL}_1$ and $\Norm{\bsigma}_1$, which require Assumptions~\ref{asm:smooth} and \ref{asm:anisotropic_noise} to describe. In contrast, if we only use the standard $L$-smooth and $\sigma_2^2$-bounded gradient variance assumption, explicit dimensional dependence will be inevitably involved, resulting in worse results than coordinate-wise uniform step sizes like~\citet{vaswani2020adaptive} %\textcolor{blue}{TODO: related work}.
To better see how Theorem~\ref{thm:adagrad_nonconvex} can describe the potential benefits of AdaGrad, we also include the convergence rates of SGD and AdaGrad-Norm as representatives of algorithms with coordinate-wisely uniform step sizes here\footnote{We also include the standard proof of SGD in Appendix~\ref{sec:sgd} for completeness. The result of AdaGrad-Norm can be found in \citet{levy2018online}.}:
\begin{align}\label{eq:sgd_convex_rate}
    \text{SGD \& AdaGrad-Norm:}& \quad \EE\left[ f(\bw) - f(\bw_*) \right] = \cO\left( \frac{D_2 \Norm{\bsigma}_2}{\sqrt{MT}} + \frac{\Norm{\bL}_\infty D_2^2}{T}  \right) .
\end{align}
By comparing the results in Theorem~\ref{thm:adagrad_nonconvex} and \eqref{eq:sgd_convex_rate}, we can find that whether AdaGrad is better than SGD or AdaGrad-Norm largely relies on the ratios $\frac{D_\infty \Norm{\bsigma}_1}{D_2\Norm{\bsigma}_2}$ and $\frac{\Norm{\bL}_1 D_\infty^2}{\Norm{\bL}_\infty D_2^2}$, which reflect the sparsity of the curvature, noise, and the geometry of $\cW$. 
When $M$ is small such that the variance term is dominant, our conclusion is consistent with that presented in \citet{duchi2011adaptive} for nonsmooth cases. Generally, when (1) $\bsigma$ is sparse, i.e. has very different scales in different coordinates, which implies that $\Norm{\bsigma}_1 \ll \sqrt{d} \Norm{\bsigma}_2$; (2) $\cW$ satisfies that $D_2$ is close to $\sqrt{d} D_\infty$, which can be satisfied by setting $\cW$ to be a hypercube, the variance term of Adagrad might be much smaller than that of SGD. 
When a large batch size is employed, the bias term can also be important and thus the superior performance of Adagrad in this case additionally requires that (3) $\bL$ is sparse such that $\Norm{\bL}_1 \ll d \Norm{\bL}_\infty$. 
Following~\citet{duchi2011adaptive}, we also provide a concrete example for better understanding the quantities in Appendix~\ref{appendix:example_convex}. 
\begin{remark}
    It is also worth pointing out that the ratio between bias terms of Theorem~\ref{thm:smooth_convex_adagrad} and \eqref{eq:sgd_convex_rate} can be in order $\Theta(1/d)$ in extreme cases, while the variance ratio can only be $\Theta(1/\sqrt{d})$. This suggests an even sharper possible gap when $M$ is large and might provide some intuition on the observation that adaptive gradient methods benefit more from large batch size than SGD~\citep{kunstner2023noise}.
\end{remark}

\subsection{Nonconvex Cases}

Next we consider the more general nonconvex scheme.
\begin{theorem}[Nonconvex convergence of Adagrad]\label{thm:adagrad_nonconvex}
    Under Assumptions~\ref{asm:f_lower_bound}, \ref{asm:smooth}, \ref{asm:unbiased_gradient}, \ref{asm:anisotropic_noise}, for the sequence $\{\bw_t\}_{t=1}^{T-1}$ generated by Adagrad (Algorithm~\ref{alg:adagrad} with option II) with constant step size $\eta_t \equiv \eta = \sqrt{ \frac{\Norm{\bL}_1}{\Delta} } $, it holds that
    \begin{align*}
        \frac{1}{T} \left( \EE\left[ \sqrt{\sum_{t=0}^{t-1} \Norm{\nabla f(\bw_t)}_1^2 } \right] \right)^2  =& \Tilde{\cO}\left( \frac{\sqrt{\Norm{\bL}_1 \Delta } \Norm{\bsigma}_1 }{\sqrt{MT}} + \frac{\Norm{\bsigma}_1^2}{M\sqrt{T}} + \frac{ \Norm{\bL}_1 \Delta }{T} \right) 
        \\
        &+ \Tilde{\cO}\left( \frac{d\epsilon \sqrt{\Norm{\bL}_1 \Delta} }{T} + \frac{d\epsilon \Norm{\bsigma}_1 }{\sqrt{M}T}\right) ,
    \end{align*}
    where $\Delta = f(\bw_0) - f^*$ and we use $\Tilde{\cO}(\cdot)$ to hide logarithmic factors.
\end{theorem}
It is worth pointing out that Theorem~\ref{thm:adagrad_nonconvex} obtains convergence of $\Norm{\nabla f(\bw)}_1$ instead of the common $\Norm{\nabla f(\bw)}_2$ by algorithms with coordinate-wise uniform step sizes like SGD, which indicates at most $\sqrt{d}$ times tighter results when the gradients are generally dense. 
Similar to the convex case, the introduction of the anisotropic assumptions \ref{asm:smooth} and \ref{asm:anisotropic_noise} removes the explicit dependence on dimension $d$ compared to existing results on adaptive gradient methods like \citet{defossez2020simple,liu2023high,zhang2022adam,wang2022provable}, rendering AdaGrad at least comparable to algorithms with coordinate-wisely uniform step sizes like SGD, which has the following results:
\begin{align}\label{eq:sgd_nonconvex_rate}
    \text{SGD:} \quad \EE\left[ \Norm{\nabla f(\bw)}_2^2 \right]  \le \cO\left( \frac{\sqrt{\Norm{\bL}_\infty \Delta } \Norm{\bsigma}_2 }{\sqrt{MT}} +  \frac{ \Norm{\bL}_\infty \Delta }{T}  \right)
\end{align}
We can find that when the batch size is large enough such that $M \ge \Norm{\bsigma}_1^2 / ( \Norm{\bL}_1 \Delta )$, the comparison between Theorem~\ref{thm:adagrad_nonconvex} and \eqref{eq:sgd_nonconvex_rate} is generally consistent with the comparison between SGD and SignSGD~\citep{bernstein2018signsgd}. It mainly relies on the sparsity of $\nabla f(\bw_t)$, $\bL$, and $\bsigma$, which can determine the ratio between the upper bounds. Generally speaking, when $\bL$ and $\bsigma$ are sparse and the gradients $\nabla f(\bw_t)$ are relatively dense, Adagrad shines compared to SGD in terms of having a tighter convergence guarantee. % \textcolor{blue}{TODO: rephrase}
We note that this can be a realistic case, as $\bL$ can be extremely sparse based on many observations on multiple scenarios~\citep{sagun2016eigenvalues,arjevani2020analytic,pan2021eigencurve}, and on the other side, $\frac{\Norm{\bsigma}_1}{\Norm{\bsigma}_2}$ and $\frac{\Norm{\nabla f(\bw_t)}_1}{\Norm{\nabla f(\bw_t)}_2}$ can be mild constants as examined by \citet{bernstein2018signsgd}. Therefore, with all these conditions satisfied, we show the potentially faster convergence of AdaGrad compared to SGD.
Also, the consistency between the convergence results of AdaGrad and SignSGD provides theoretical insights into the close relation between adaptive gradient methods and sign-based methods.

\begin{table*}[t]
\caption{We summarize the convergence rates of algorithms under the nonconvex smooth settings, with large enough batch size such that $M \ge \Norm{\bsigma}_1^2 / ( \Norm{\bL}_1 \Delta )$. Note that we omit logarithmic terms here. The convergence of SGD in the smooth setting is a standard result, and we include the theorem in the appendix for reference. Also, we leave the convergence of SGD under generalized smoothness blank as it is generally incomparable with other results, which we will discuss in Section~\ref{sec:AdaGrad_generalized_smooth}. 
}
  \label{tab:comparison_bounds}
  \centering
  \small
  \begin{tabular}{cccc}
    \addlinespace
    \hline 
    Algorithms & \makecell{ Convergence \\ Objective } & \makecell{ Anisotropic Smooth } & \makecell{ Anisotropic Generalized Smooth} \\\hline\hline\addlinespace
    \makecell{ SGD } & \makecell{ $\EE\left[ \Norm{\nabla f(\bw)}^2_{2} \right]$ } & \makecell{ $\frac{\sqrt{\Norm{\bL}_\infty \Delta  } \Norm{\bsigma}_2 }{\sqrt{MT}} +  \frac{ \Norm{\bL}_\infty \Delta }{T}$ \\ {\scriptsize Theorem~\ref{thm:sgd_nonconvex_upper}} } & \makecell{ -- }  \\ \addlinespace
    \hline \addlinespace
    \makecell{ AdaGrad-Norm } & \makecell{ $\left(\EE\left[ \Norm{\nabla f(\bw)}_{\color{red} \mathlarger{\mathbf{2}}} \right]\right)^2$ } & \makecell{ $\frac{\sqrt{\Norm{\bL}_\infty \Delta } \Norm{\bsigma}_{2} }{\sqrt{MT}} +  \frac{ \Norm{\bL}_\infty \Delta }{T}$ \\ {\scriptsize \citep{faw2023beyond}} } & \makecell{ $\frac{\sqrt{L_0 \Delta } \Norm{\bsigma}_2 }{\sqrt{MT}} + \frac{ L_0 \Delta }{T}$ \\ $ + \frac{ L_1 \Delta \Norm{\bsigma}_2 }{\sqrt{MT}} + \frac{L_1^2 \Delta^2 }{T}$ \\ {\scriptsize \citep{faw2023beyond}} }  \\ \addlinespace
    \hline \addlinespace
    \makecell{AdaGrad  } & \makecell{  $ \left(\EE\left[ \Norm{\nabla f(\bw)}_{{ \color{blue} \mathlarger{\mathbf{1}}} } \right]\right)^2 $   } 
    & \makecell{ $\frac{\sqrt{\Norm{\bL}_1 \Delta } \Norm{\bsigma}_1 }{\sqrt{MT}} + \frac{ \Norm{\bL}_1 \Delta }{T} $ \\ {\scriptsize Theorem~\ref{thm:adagrad_nonconvex}} } 
    & \makecell{$\frac{\sqrt{\Norm{\bL_0}_1 \Delta } \Norm{\bsigma}_1 }{\sqrt{MT}} + \frac{ \Norm{\bL_0}_1 \Delta }{T}$ \\ $ + \frac{ \Norm{\bL_1}_\infty \Delta \Norm{\bsigma}_1 }{\sqrt{MT}} + \frac{\Norm{\bL_1}_\infty^2 \Delta^2 }{T} $ \\ {\scriptsize Theorem~\ref{thm:adagrad_generalized_nonconvex}} }  
    \\ \addlinespace
    \hline 
  \end{tabular}
\end{table*}

\section{AdaGrad with Generalized Anisotropic Smoothness}\label{sec:AdaGrad_generalized_smooth}
In previous sections, we have discussed how AdaGrad can potentially outperform algorithms with uniform step sizes under the anisotropic smoothness settings. However, the loss function in practice can commonly dissatisfy the smoothness assumption, with local smoothness potentially unbounded. 
To better describe the complicated real cases, \citet{zhang2019gradient} introduce the $(L_0,L_1)$-smoothness as a generalization of the standard $L$-smoothness, which can be written as
\begin{align}\label{eq:standard_generalized_smooth}
    \Norm{\nabla f(\bw) - \nabla f(\bw')}_2 \le \left( L_0 + L_1 \Norm{\nabla f(\bw)}_2 \right) \Norm{\bw - \bw'}_2
\end{align}
for all $\Norm{\bw - \bw'} \le 1/L_1$~\citep{zhang2020improved}. By involving the gradient term to describe the local smoothness, this assumption can even be applicable to describe neural networks~\citep{zhang2019gradient,crawshaw2022robustness}.
In this section, we discuss an extension of both the anisotropic $\bL$-smoothness and the $(L_0,L_1)$-smoothness and the corresponding convergence results of AdaGrad.
\begin{assumption}[Anisotropic $(\bL_0,\bL_1)$-smoothness]\label{asm:generalized_smooth}
    There exists positive vectors $\bL_0 = [\bL_{0,1},\dots, \bL_{0,d}] \in \RR_+^d$ and $\bL_1 = [\bL_{1,1},\dots, \bL_{1,d}] \in \RR_+^d$ such that $f(\cdot)$ is ($\bL_0$,$\bL_1$)-smooth, namely, for $\bw,\bw' \in \cW$ such that $\Norm{\bw - \bw'}_{\bL_1} \le \sqrt{d}$, it holds that
    \begin{align}\label{eq:asm_generalized_smooth}
        \Norm{\nabla f(\bw) - \nabla f(\bw')}_{(\bL(\bw))^{-1}} \le \Norm{\bw - \bw'}_{\bL(\bw)} ,
    \end{align}
    where $\bL(\bw) = [[\bL(\bw)]_{1}, \dots, [\bL(\bw)]_{d}] \in \RR^d$ and $[\bL(\bw)]_{j} = \bL_{0,j} + \bL_{1,j} \Abs{\partial_j f(\bw)}$ for all $j \in [d]$.
\end{assumption}
% Let us give some intuition for this novel assumption. 
For one thing, Assumption~\ref{asm:generalized_smooth} is a natural generalization of Assumption~\ref{asm:smooth} and can directly imply it by setting $\bL_1 = 0$, hence also enjoys the nice properties we have discussed for Assumption~\ref{asm:smooth}.
Also, in the spirits of \eqref{eq:standard_generalized_smooth}, we include the gradient magnitudes to describe the local smoothness. Instead of the gradient 2-norm, we use the absolute value of each coordinate of the gradient, which is intuitively more relevant to the anisotropic nature of curvature. We also conducted numerical experiments to verify the validity of the assumption, as shown in Figure~\ref{figs:exp_assumption_verification}.

\begin{remark}
    
It is also worth noticing that a similar but different coordinate-wise generalized smoothness assumption has been considered in \citet{crawshaw2022robustness}:
\begin{align}\label{eq:crawshaw_generalized_smooth}
    \Abs{\partial_j f(\bw) - \partial_j f(\bw')} \le \left( \bL_{0,j} + \bL_{1,j} \Abs{\partial_j f(\bw)} \right) \Norm{\bw - \bw'}_2
\end{align}
for all $\Norm{\bw - \bw'}_2 \le 1/\Norm{\bL_1}_\infty$ and all coordinates $j \in [d]$.
Compared to this assumption, Assumption~\ref{asm:generalized_smooth} offers several evident advantages: (1) Assumption~\ref{asm:generalized_smooth} does not require conditions for each coordinate; (2) Experiments show that Assumption~\ref{asm:generalized_smooth} can well describe the real anisotropic curvature; (3) Technically, it is difficult to avoid explicit existence of $d$ in the final convergence bound of AdaGrad if Eqn.~\eqref{eq:crawshaw_generalized_smooth} is employed, for instance, even the convergence of the Generalized SignSGD obtained in \citet{crawshaw2022robustness} has explicit dependence on $d$.

\end{remark}

\begin{theorem}[Convergence of Adagrad with generalized smoothness]\label{thm:adagrad_generalized_nonconvex}
    Under Assumptions~\ref{asm:f_lower_bound}, \ref{asm:unbiased_gradient}, \ref{asm:anisotropic_noise}, \ref{asm:generalized_smooth}, for the sequence $\{\bw_t\}_{t=1}^{T-1}$ generated by Adagrad (Algorithm~\ref{alg:adagrad} with option II) with constant step size $\eta_t \equiv \eta = \min\left\{ \frac{1}{4\Norm{\bL_1}_\infty}, \sqrt{ \frac{\Norm{\bL_0}_1}{\Delta} } \right\}$, it holds that
    \begin{align*}
        \frac{1}{T} \left( \EE\left[ \sqrt{\sum_{t=0}^{t-1} \Norm{\nabla f(\bw_t)}_1^2 } \right] \right)^2  =& \Tilde{\cO}\left( \frac{\sqrt{\Norm{\bL_0}_1 \Delta } \Norm{\bsigma}_1 }{\sqrt{MT}} + \frac{\Norm{\bsigma}_1^2}{M\sqrt{T}} + \frac{ \Norm{\bL_0}_1 \Delta }{T} \right) 
        \\
        &+ \Tilde{\cO}\left( \frac{ \Norm{\bL_1}_\infty \Delta \Norm{\bsigma}_1 }{\sqrt{MT}} + \frac{\Norm{\bL_1}_\infty^2 \Delta^2 }{T} \right)
        \\
        &+ \Tilde{\cO}\left( \frac{d\epsilon \left( \sqrt{\Norm{\bL_0}_1 \Delta}  + \Norm{\bL_1}_\infty \Delta \right) }{T} + \frac{d\epsilon \Norm{\bsigma}_1 }{\sqrt{M}T}\right) ,
    \end{align*}
    where $\Delta = f(\bw_0) - f^*$ and we use $\Tilde{\cO}(\cdot)$ to hide logarithmic factors.
\end{theorem}

Note that Theorem~\ref{thm:adagrad_generalized_nonconvex} is a generalization of Theorem~\ref{thm:adagrad_nonconvex} based on the relation between Assumptions \ref{asm:generalized_smooth} and \ref{asm:smooth}, which can directly imply Theorem~\ref{thm:adagrad_nonconvex} by simply setting $\bL_1 = 0$. % To further interpret the theorem, let us involve the discussions below.
Intuitively, the interpretation of the theorem leads to the following implications.

\textbf{Comparison with SGD. } \citet{zhang2019gradient} show that SGD generally requires an additional assumption on universally bounded gradients to obtain convergence under the generalized $(L_0,L_1)$-smoothness, and thus clipping should be introduced. The theoretical properties of clipped SGD has been extensively studied after that. However, to the best of our knowledge, existing results on clipped SGD either require restrictive noise assumptions~\citep{zhang2019gradient,zhang2020improved,chen2020understanding,qian2021understanding} or can only shows worse rate on $T$ and $M$~\citep{koloskova2023revisiting,gorbunov2020stochastic}. These results are generally incomparable with Theorem~\ref{thm:adagrad_nonconvex}, showing the superiority of adaptive gradient methods, as also found in existing results~\citep{wang2023convergence,li2024convergence}. 

\textbf{Comparison with AdaGrad-Norm. }
One can also look at the convergence results of AdaGrad-Norm~\citep{streeter2010less} assuming \eqref{eq:standard_generalized_smooth}, a scalar step size variant of AdaGrad:
\begin{align}\label{eq:adanorm_nonconvex}
    \text{AdaGrad-Norm:}& \quad \left( \EE\left[\Norm{\nabla f(\bw)}_2\right] \right)^2 \le \Tilde{\cO}\left( \frac{\sqrt{L_0 \Delta } \Norm{\bsigma}_2 }{\sqrt{MT}} + \frac{ L_0 \Delta }{T} + \frac{ L_1 \Delta \Norm{\bsigma}_2 }{\sqrt{MT}} + \frac{L_1^2 \Delta^2 }{T} \right) ,
\end{align}
which is obtained by taking $\eta$ to optimize the upper bound in Theorem 3 in \citet{faw2023beyond}. It is worth noticing that in the case $\bL_1 = 0$, the comparison between AdaGrad and AdaGrad-Norm is generally consistent with that between AdaGrad and SGD discussed in Section~\ref{sec:AdaGrad}, showing AdaGrad's effectiveness for anisotropic problems.
On the other hand, when it comes to the general $(\bL_0,\bL_1)$-smoothness assumption, things become more complicated given the indeterminate relationship between the anisotropic $\bL_1$ and $L_1$ in \eqref{eq:standard_generalized_smooth}. 
This relation is mainly about whether Assumption~\ref{asm:generalized_smooth} or the initial $(L_0,L_1)$-smoothness assumption can better fit the common training settings, such as large-scale language model pre-training, which is an intriguing topic worth exploring in the future.

\begin{remark}
    \textit{(Technical Contribution)} From the technical perspective, our proof generally considers a similar main line to \citet{defossez2020simple} but involves multiple novel proof techniques to remove the restrictive assumption on globally bounded $\Norm{\nabla f(\bw)}$ and enable the presence of generalized smoothness. Also, Theorem~\ref{thm:adagrad_nonconvex} can directly imply a comparable convergence result of AdaGrad-Norm with \citet{faw2023beyond,wang2023convergence} while it avoids the complicated proof in \citet{faw2023beyond} and the heavy dependence on $1/\epsilon$ in \citet{wang2023convergence}. On the other hand, however, \citet{faw2023beyond,wang2023convergence} allows relaxed noise assumption, which may be of interest and we leave it for possible future work.
\end{remark}

\section{Experimental Results}\label{sec:experiments}

\subsection{Convex Case}

\begin{table}[b]
\caption{Statistics on logistic regression. $D_2 = \max_t \|\bw_t - \bw_*\|_2$ and $D_{\infty} = \max_t \|\bw_t - \bw_*\|_{\infty}$ are estimated by the maximum value under all searched settings without loss explosion. Smaller values in $C_{\textbf{var}} \triangleq D_\infty / D_2$ represent better theoretical bounds for Adagrad when compared with SGD. It is evident that in large sparse datasets, $D_\infty$ is much smaller than $D_2$, verifying the empirical gains of Adagrad implied by our theory.}
\label{tab:quantity}
\vskip 0.15in
\begin{center}
\begin{small}
\begin{sc}
\begin{tabular}{lccccc}
\toprule
Dataset & $D_{\infty}$ & $D_2$ & $\Norm{\bL}_1$ & $\Norm{\bL}_\infty$ & $C_{\textbf{var}}$ % & $C_{\textbf{bias}}$
\\
\midrule
 \texttt{a4a} & 9.73  & 32.24 & 14.87 & 7.26 & \textbf{0.30} % &\textbf{0.19}
\\
\midrule
\texttt{a6a} & 8.88  & 28.57 & 14.87 & 7.26 & \textbf{0.31} % & \textbf{0.20} 
\\
\midrule
\texttt{a9a} & 9.79 & 29.43 & 14.87 & 7.27 & \textbf{0.32} % & \textbf{0.23}
\\
\midrule
\texttt{real-sim} & 34.70 & 729.50 & 2.00 & 1.01 & \textbf{0.05}
\\
\midrule
\texttt{rcv1.bin} & 15.32 & 635.47 & 2.00 & 1.02 & \textbf{0.02}
\\
\bottomrule
\end{tabular}
\end{sc}
\end{small}
\end{center}
\vskip -0.1in
\end{table}

To verify the aforementioned theoretical results, we conduct experiments on logistics regressions, whose loss functions are generally convex smooth. Specifically, we utilize real-world datasets \texttt{a4a}, \texttt{a6a}, \texttt{a9a}, \texttt{real-sim} and \texttt{rcv1.binary} from \texttt{libsvm}~\citep{chang2011libsvm}, which comprises of $N=4781$, $11220$, $32561$, $20242$ and $72309$ samples respectively. Within the former three datasets, each sample has a feature of $d=123$ dimensions, which are generally sparse given only 14 non-zero-valued dimensions in average for each sample. For the latter two large datasets, each sample possesses $d=47,236$ and $d=20,958$ feature dimensions individually in each dataset, where only $51.29$ and $74.05$ dimensions in average are non-zero. More details of the experimental setup are available in Appendix~\ref{appendix:exp_details}.

As shown by the results presented in Table~\ref{tab:logistic_regression}, for cases when the variance term dominates, e.g. $\texttt{a9a}$, Adagrad demonstrates similar convergence behaviors for varied batch sizes. This behavior cannot be explained by previous nonsmooth theories, since $T = N/M$ should provide worse convergence guarantees $\cO(1/\sqrt{T})$ when batch size $M$ increases. Furthermore, Adagrad constantly provides faster convergence than SGD, which verifies the superiority suggested in our theorems. In addition, when the batch size increases and the bias term becomes dominant, Adagrad is affected less than SGD, showing its robustness against different batch sizes.

\begin{table*}[t]
\caption{Training losses of SGD and Adagrad on logistic regression with dataset \texttt{a4a}, \texttt{a6a} and \texttt{a9a}. We report statistics over 3 seeds with different batch sizes $M$. Note that Adagrad's convergence behavior is generally unaffected by the batch size when $M \le \sqrt{N}$.}
\label{tab:logistic_regression}
\vskip 0.15in
\begin{center}
\begin{small}
\begin{sc}
\resizebox{\linewidth}{!}{
\begin{tabular}{lccccccc}
\toprule
\multirow{2}{*}{Dataset (Small)} & \multirow{2}{*}{Method} & \multicolumn{6}{c}{$(f(w) - f(w_*)) \times 10^{-3}$}
\\
\cmidrule{3-8}
& & $M = 1$ & $M=4$ & $M=16$ & $M=64$ & $M=256$ & $M=1024$ \\
\midrule
\multirow{2}{*}{\texttt{a4a}} & SGD & 
2.66$\pm$0.08 & 2.32$\pm$0.34 & 3.47$\pm$0.07 & 2.24$\pm$0.14 & 4.61$\pm$0.03 & 8.47$\pm$0.05 \\
& Adagrad & 
 \cellcolor{lightblue}0.22$\pm$0.03 & \cellcolor{lightblue}0.22$\pm$0.03 & \cellcolor{lightblue}0.25$\pm$0.03 & \cellcolor{lightblue}0.24$\pm$0.03 & \cellcolor{lightblue}0.29$\pm$0.03 & \cellcolor{lightblue}0.51$\pm$0.08 \\
\midrule
\multirow{2}{*}{\texttt{a6a}} & SGD & 
0.87$\pm$0.09 & 1.40$\pm$0.01 & 0.82$\pm$0.09 & 1.56$\pm$0.03 & 1.03$\pm$0.01 & 2.53$\pm$0.03 \\
& Adagrad & 
\cellcolor{lightblue}0.16$\pm$0.00 & \cellcolor{lightblue}0.16$\pm$0.00 & \cellcolor{lightblue}0.16$\pm$0.01 & \cellcolor{lightblue}0.21$\pm$0.05 & \cellcolor{lightblue}0.17$\pm$0.01 & \cellcolor{lightblue}0.20$\pm$0.02 \\
\midrule
\multirow{2}{*}{\texttt{a9a}} & SGD & 
0.77$\pm$0.08 & 0.47$\pm$0.01 & 0.48$\pm$0.05 & 0.58$\pm$0.01 & 0.52$\pm$0.06 & 0.76$\pm$0.01 \\
& Adagrad & 
\cellcolor{lightblue}0.14$\pm$0.01 & \cellcolor{lightblue}0.14$\pm$0.00 & \cellcolor{lightblue}0.15$\pm$0.00 & \cellcolor{lightblue}0.16$\pm$0.01 & \cellcolor{lightblue}0.20$\pm$0.01 & \cellcolor{lightblue}0.12$\pm$0.02 \\
\\[-2ex]\hline 
\hline \\[-2ex]
\multirow{2}{*}{Dataset (Large)} & \multirow{2}{*}{Method} & \multicolumn{6}{c}{$(f(w) - f(w_*)) \times 10^{-1}$}
\\
\cmidrule{3-8}
& & $M = 1$ & $M=4$ & $M=16$ & $M=64$ & $M=256$ & $M=1024$ \\
\midrule
\multirow{2}{*}{\texttt{real-sim}} & SGD & 0.42$\pm$0.08 & 0.27$\pm$0.08 & 0.52$\pm$0.06 & 0.92$\pm$0.03 & 1.57$\pm$0.02 & 2.68$\pm$0.01 \\
& Adagrad &
\cellcolor{lightblue}0.14$\pm$0.00 & \cellcolor{lightblue}0.14$\pm$0.00 & \cellcolor{lightblue}0.14$\pm$0.00 & \cellcolor{lightblue}0.14$\pm$0.00 & \cellcolor{lightblue}0.15$\pm$0.00 & \cellcolor{lightblue}0.19$\pm$0.00 \\
\midrule
\multirow{2}{*}{\texttt{rcv1.binary}} & SGD &
0.48$\pm$0.02 & 0.28$\pm$0.07 & 0.68$\pm$0.04 & 1.33$\pm$0.04 & 2.55$\pm$0.21 & 5.02$\pm$0.16 \\
& Adagrad &
\cellcolor{lightblue}0.10$\pm$0.00 & \cellcolor{lightblue}0.10$\pm$0.00 & \cellcolor{lightblue}0.10$\pm$0.00 & \cellcolor{lightblue}0.11$\pm$0.00 & \cellcolor{lightblue}0.14$\pm$0.01 & \cellcolor{lightblue}0.20$\pm$0.01 \\
\bottomrule
\end{tabular}
}
\end{sc}
\end{small}
\end{center}
\vskip -0.1in
\end{table*}

\iffalse
\begin{figure*}[ht]
\vskip 0.2in
\begin{center}
\centerline{
\includegraphics[width=0.66\columnwidth]{figs/a9a_bs-1.png}
\includegraphics[width=0.66\columnwidth]{figs/a9a_bs-16.png}
\includegraphics[width=0.66\columnwidth]{figs/a9a_bs-256.png}
}
\caption{Loss gap curve comparison in $\texttt{a9a}$.}
\label{figs:a9a_comp}
\end{center}
\vskip -0.2in
\end{figure*}
\fi

\subsection{Nonconvex Case}

\begin{wraptable}{r}{0.4\textwidth}
\vspace{-0.27in}
\caption{Training losses of SGD and Adagrad on instruction following tasks with dataset Alpaca and GPT2 model.}
\label{tab:nonconvex_gpt2}
\vskip 0.15in
\centering
\scriptsize
\begin{sc}
% \resizebox{\linewidth}{!}{
\begin{tabular}{cccc}
\toprule
\multirow{2}{*}{Method} & \multicolumn{3}{c}{Training Loss $f(w)$}
\\
\cmidrule{2-4}
& $M = 32$ & $64$ & $128$ \\
\midrule
SGD & 2.20 & 2.24 & 2.29\\
Adagrad &
 \cellcolor{lightblue}2.14 & \cellcolor{lightblue}2.12 & \cellcolor{lightblue}2.11
 \\
\\[-2ex]\hline 
\hline \\[-2ex]
& $M=256$ & $512$ & $1024$ \\
\midrule
SGD & 2.36 & 2.45 & 2.57 \\
Adagrad &
  \cellcolor{lightblue}2.12 & \cellcolor{lightblue}2.14 & \cellcolor{lightblue}2.20
 \\
\bottomrule
\end{tabular}
% }
\end{sc}
\vskip -0.1in
\end{wraptable}

For nonconvex cases, we check the instruction-following fine-tuning task on Alpaca~\citep{alpaca} dataset with GPT-2~\citep{radford2019gpt2} model. GPT-2 utilizes GELU~\citep{hendrycks2016gaussian} as its activation function. As shown in Table~\ref{tab:nonconvex_gpt2}, it can be observed that Adagrad still outperforms SGD with different batch sizes. The loss gap is especially salient under large batch sizes. This also confirms that Adagrad's convergence speed is not significantly affected by the large batch size, which matches the results in our theory. Full experimental details are available in Appendix~\ref{appendix:exp_details}.

We also examine Assumption~\ref{asm:generalized_smooth} under this non-convex setting by following the common practice of~\citep{zhang2019gradient,crawshaw2022robustness}. Due to the inherent dependency on dimensions in Assumption~\ref{asm:generalized_smooth}, we randomly sample a small portion of the parameters in GPT-2 to check. As shown in Figure~\ref{figs:exp_assumption_verification}, 
a linear bound of local smoothness with respect to the gradient value holds in all randomly chosen dimensions. Notice that the observed condition in Figure~\ref{figs:exp_assumption_verification} is even stronger, as it directly implies Assumption~\ref{asm:generalized_smooth} by summing up all the dimensions.

\begin{figure}[ht]
\vskip 0.2in
\centering
\includegraphics[width=0.4\columnwidth]{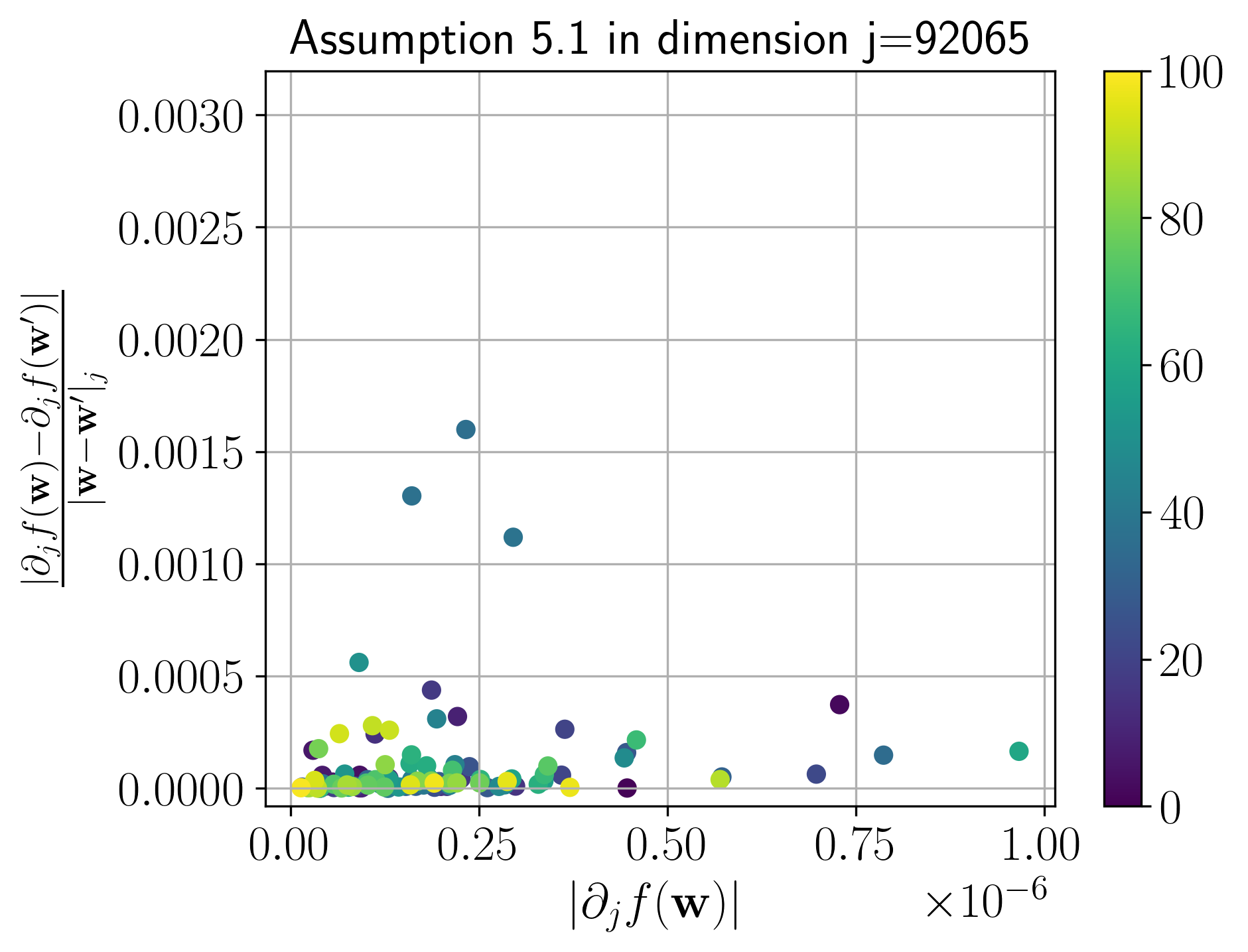}
\includegraphics[width=0.4\columnwidth]{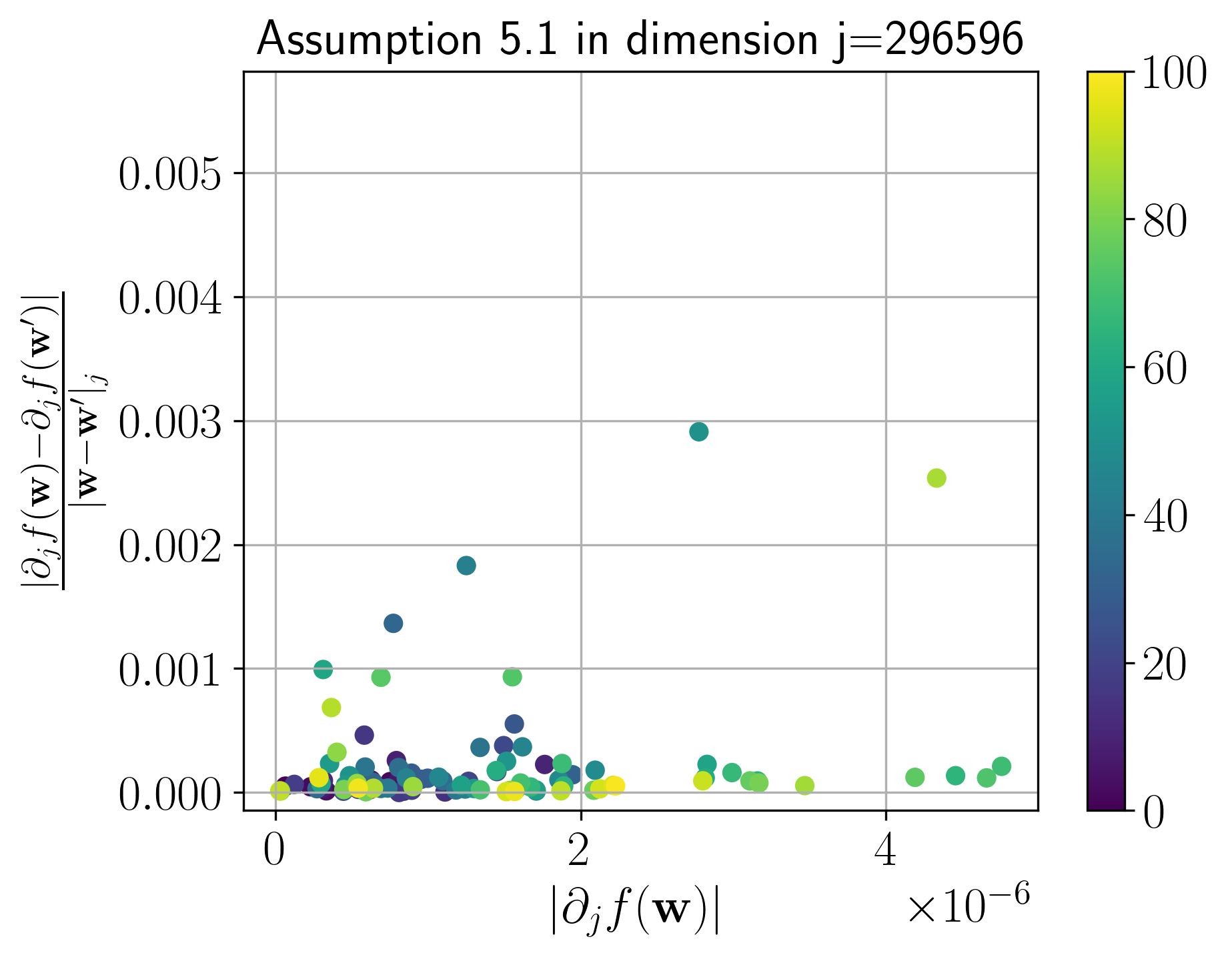}
\\
\includegraphics[width=0.4\columnwidth]{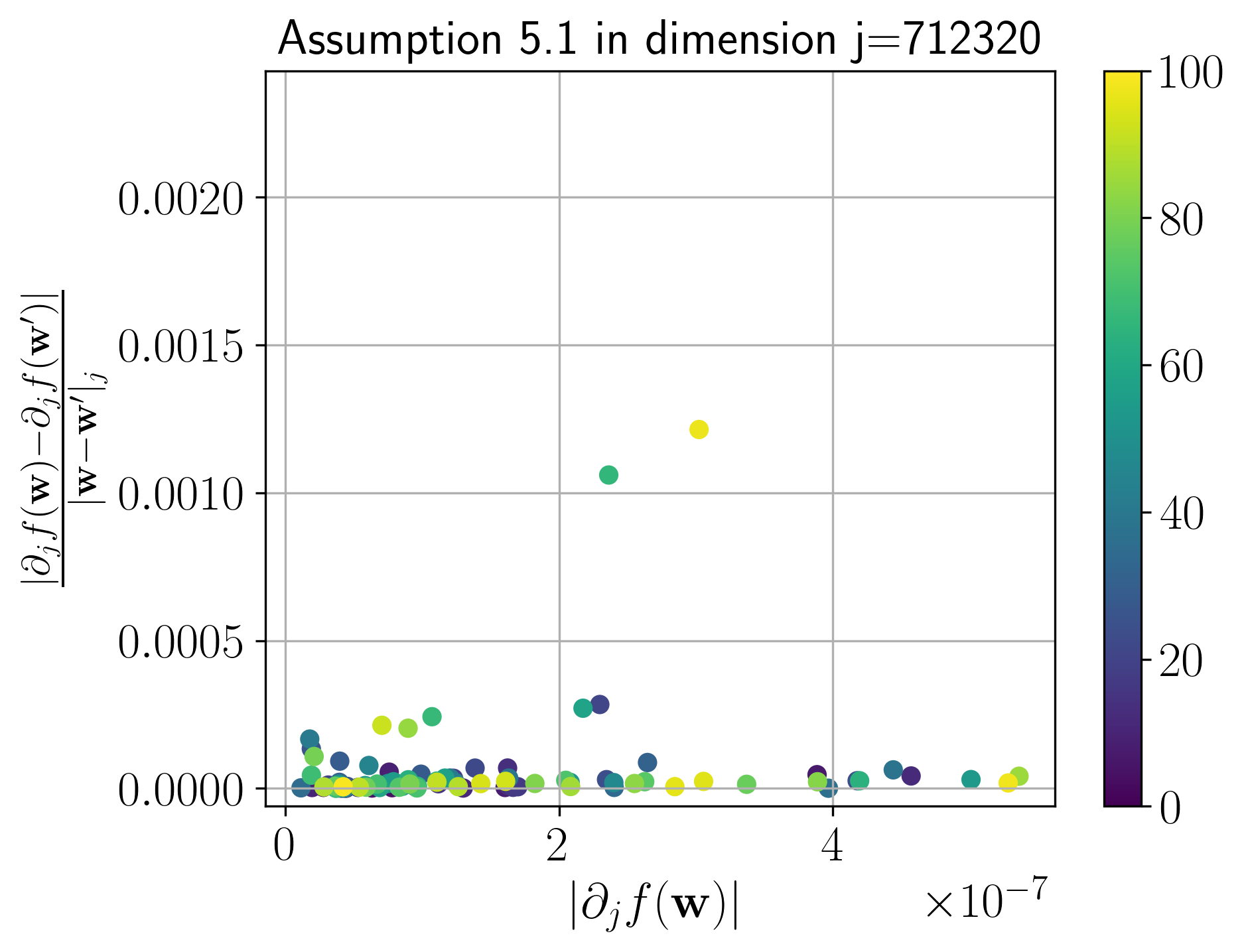}
\includegraphics[width=0.4\columnwidth]{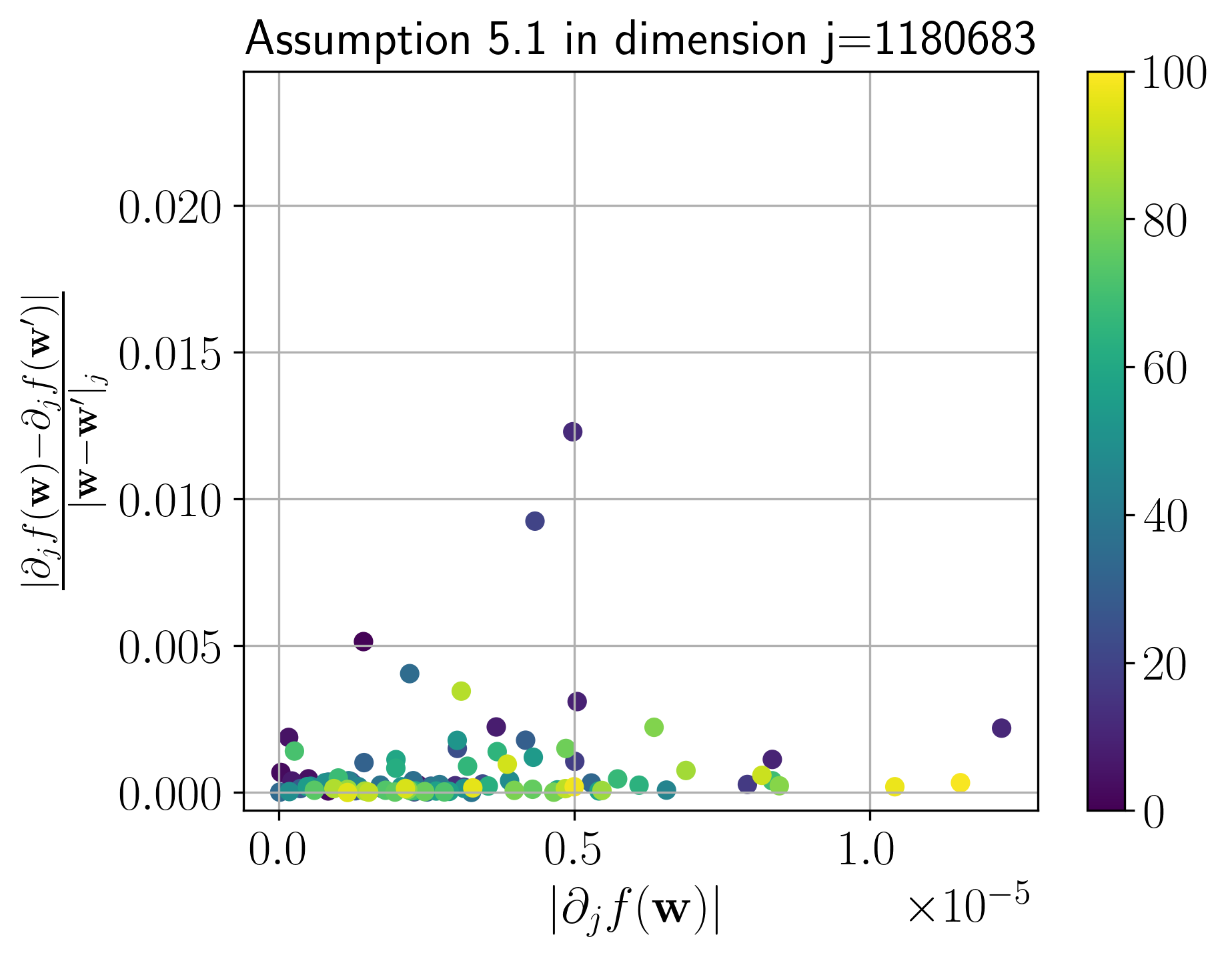}
\caption{Verification of Assumption~\ref{asm:generalized_smooth} in GPT-2 on Alpaca dataset. \textbf{x-axis}: $|\partial_j f(\bw)|$, \textbf{y-axis}: $|\partial_j f(\bw) - \partial_j f(\bw')| / |\bw - \bw'|_j$, where the color represents the iteration index of $\bw$. 
We run Adam with full gradients for 100 steps and randomly selected nearby points $\bw$ and $\bw$' along the trajectory to plot the scatter points.}
\label{figs:exp_assumption_verification}
\end{figure}

\section{Conclusion}
In this paper, we present the theoretical convergence results of Adagrad under anisotropic smoothness and noise assumptions. We further introduce a novel anisotropic generalized smoothness assumption to extend the aforementioned results in a more fine-grained analysis.
Based on the theorems, we conduct comparisons between the convergence rates of AdaGrad and SGD and AdaGrad-Norm in the large batch settings, which provides a deeper theoretical understanding of when and why adaptive gradient methods can outperform classical gradient algorithms with coordinate-wise uniform step sizes. 
Empirical studies offer strong evidence to support our proposed assumptions and theory.

\newpage
\bibliography{ref}

\begin{thebibliography}{92}
\providecommand{\natexlab}[1]{#1}
\providecommand{\url}[1]{\texttt{#1}}
\expandafter\ifx\csname urlstyle\endcsname\relax
  \providecommand{\doi}[1]{doi: #1}\else
  \providecommand{\doi}{doi: \begingroup \urlstyle{rm}\Url}\fi

\bibitem[Allen-Zhu(2018)]{allen2018katyusha}
Zeyuan Allen-Zhu.
\newblock Katyusha: The first direct acceleration of stochastic gradient methods.
\newblock \emph{Journal of Machine Learning Research}, 18\penalty0 (221):\penalty0 1--51, 2018.

\bibitem[Arjevani \& Field(2020)Arjevani and Field]{arjevani2020analytic}
Yossi Arjevani and Michael Field.
\newblock Analytic characterization of the hessian in shallow relu models: A tale of symmetry.
\newblock \emph{Advances in Neural Information Processing Systems}, 33:\penalty0 5441--5452, 2020.

\bibitem[Attia \& Koren(2023)Attia and Koren]{attia2023sgd}
Amit Attia and Tomer Koren.
\newblock Sgd with adagrad stepsizes: Full adaptivity with high probability to unknown parameters, unbounded gradients and affine variance.
\newblock \emph{arXiv preprint arXiv:2302.08783}, 2023.

\bibitem[Bernstein et~al.(2018{\natexlab{a}})Bernstein, Wang, Azizzadenesheli, and Anandkumar]{bernstein2018signsgd}
Jeremy Bernstein, Yu-Xiang Wang, Kamyar Azizzadenesheli, and Animashree Anandkumar.
\newblock signsgd: Compressed optimisation for non-convex problems.
\newblock In \emph{International Conference on Machine Learning}, pp.\  560--569. PMLR, 2018{\natexlab{a}}.

\bibitem[Bernstein et~al.(2018{\natexlab{b}})Bernstein, Zhao, Azizzadenesheli, and Anandkumar]{bernstein2019signsgd}
Jeremy Bernstein, Jiawei Zhao, Kamyar Azizzadenesheli, and Anima Anandkumar.
\newblock signsgd with majority vote is communication efficient and fault tolerant.
\newblock \emph{arXiv preprint arXiv:1810.05291}, 2018{\natexlab{b}}.

\bibitem[Cesa-Bianchi et~al.(2007)Cesa-Bianchi, Mansour, and Stoltz]{cesa2007improved}
Nicolo Cesa-Bianchi, Yishay Mansour, and Gilles Stoltz.
\newblock Improved second-order bounds for prediction with expert advice.
\newblock \emph{Machine Learning}, 66:\penalty0 321--352, 2007.

\bibitem[Chang \& Lin(2011)Chang and Lin]{chang2011libsvm}
Chih-Chung Chang and Chih-Jen Lin.
\newblock Libsvm: A library for support vector machines.
\newblock \emph{ACM transactions on intelligent systems and technology (TIST)}, 2\penalty0 (3):\penalty0 1--27, 2011.

\bibitem[Chen et~al.(2020)Chen, Wu, and Hong]{chen2020understanding}
Xiangyi Chen, Steven~Z Wu, and Mingyi Hong.
\newblock Understanding gradient clipping in private sgd: A geometric perspective.
\newblock \emph{Advances in Neural Information Processing Systems}, 33:\penalty0 13773--13782, 2020.

\bibitem[Choi et~al.(2019)Choi, Shallue, Nado, Lee, Maddison, and Dahl]{choi2019empirical}
Dami Choi, Christopher~J Shallue, Zachary Nado, Jaehoon Lee, Chris~J Maddison, and George~E Dahl.
\newblock On empirical comparisons of optimizers for deep learning.
\newblock \emph{arXiv preprint arXiv:1910.05446}, 2019.

\bibitem[Crawshaw et~al.(2022)Crawshaw, Liu, Orabona, Zhang, and Zhuang]{crawshaw2022robustness}
Michael Crawshaw, Mingrui Liu, Francesco Orabona, Wei Zhang, and Zhenxun Zhuang.
\newblock Robustness to unbounded smoothness of generalized signsgd.
\newblock \emph{Advances in neural information processing systems}, 35:\penalty0 9955--9968, 2022.

\bibitem[Das et~al.(2024)Das, Agarwal, Sanghavi, and Dhillon]{das2024towards}
Rudrajit Das, Naman Agarwal, Sujay Sanghavi, and Inderjit~S Dhillon.
\newblock Towards quantifying the preconditioning effect of adam.
\newblock \emph{arXiv preprint arXiv:2402.07114}, 2024.

\bibitem[De et~al.(2018)De, Mukherjee, and Ullah]{de2018convergence}
Soham De, Anirbit Mukherjee, and Enayat Ullah.
\newblock Convergence guarantees for rmsprop and adam in non-convex optimization and an empirical comparison to nesterov acceleration.
\newblock \emph{arXiv preprint arXiv:1807.06766}, 2018.

\bibitem[D{\'e}fossez et~al.(2020)D{\'e}fossez, Bottou, Bach, and Usunier]{defossez2020simple}
Alexandre D{\'e}fossez, L{\'e}on Bottou, Francis Bach, and Nicolas Usunier.
\newblock A simple convergence proof of adam and adagrad.
\newblock \emph{arXiv preprint arXiv:2003.02395}, 2020.

\bibitem[Dieuleveut et~al.(2017)Dieuleveut, Flammarion, and Bach]{dieuleveut2017harder}
Aymeric Dieuleveut, Nicolas Flammarion, and Francis Bach.
\newblock Harder, better, faster, stronger convergence rates for least-squares regression.
\newblock \emph{The Journal of Machine Learning Research}, 18\penalty0 (1):\penalty0 3520--3570, 2017.

\bibitem[Duchi et~al.(2011)Duchi, Hazan, and Singer]{duchi2011adaptive}
John Duchi, Elad Hazan, and Yoram Singer.
\newblock Adaptive subgradient methods for online learning and stochastic optimization.
\newblock \emph{Journal of machine learning research}, 12\penalty0 (7), 2011.

\bibitem[Faw et~al.(2022)Faw, Tziotis, Caramanis, Mokhtari, Shakkottai, and Ward]{faw2022power}
Matthew Faw, Isidoros Tziotis, Constantine Caramanis, Aryan Mokhtari, Sanjay Shakkottai, and Rachel Ward.
\newblock The power of adaptivity in sgd: Self-tuning step sizes with unbounded gradients and affine variance.
\newblock In \emph{Conference on Learning Theory}, pp.\  313--355. PMLR, 2022.

\bibitem[Faw et~al.(2023)Faw, Rout, Caramanis, and Shakkottai]{faw2023beyond}
Matthew Faw, Litu Rout, Constantine Caramanis, and Sanjay Shakkottai.
\newblock Beyond uniform smoothness: A stopped analysis of adaptive sgd.
\newblock In \emph{The Thirty Sixth Annual Conference on Learning Theory}, pp.\  89--160. PMLR, 2023.

\bibitem[Garrigos \& Gower(2023)Garrigos and Gower]{garrigos2023handbook}
Guillaume Garrigos and Robert~M Gower.
\newblock Handbook of convergence theorems for (stochastic) gradient methods.
\newblock \emph{arXiv preprint arXiv:2301.11235}, 2023.

\bibitem[Ge et~al.(2019)Ge, Kakade, Kidambi, and Netrapalli]{ge2019step}
Rong Ge, Sham~M Kakade, Rahul Kidambi, and Praneeth Netrapalli.
\newblock The step decay schedule: A near optimal, geometrically decaying learning rate procedure for least squares.
\newblock In \emph{Advances in Neural Information Processing Systems}, pp.\  14977--14988, 2019.

\bibitem[Ghadimi \& Lan(2013)Ghadimi and Lan]{ghadimi2013stochastic}
Saeed Ghadimi and Guanghui Lan.
\newblock Stochastic first-and zeroth-order methods for nonconvex stochastic programming.
\newblock \emph{SIAM Journal on Optimization}, 23\penalty0 (4):\penalty0 2341--2368, 2013.

\bibitem[Gorbunov et~al.(2020)Gorbunov, Danilova, and Gasnikov]{gorbunov2020stochastic}
Eduard Gorbunov, Marina Danilova, and Alexander Gasnikov.
\newblock Stochastic optimization with heavy-tailed noise via accelerated gradient clipping.
\newblock \emph{Advances in Neural Information Processing Systems}, 33:\penalty0 15042--15053, 2020.

\bibitem[Guo et~al.(2021)Guo, Xu, Yin, Jin, and Yang]{guo2021novel}
Zhishuai Guo, Yi~Xu, Wotao Yin, Rong Jin, and Tianbao Yang.
\newblock A novel convergence analysis for algorithms of the adam family.
\newblock \emph{arXiv preprint arXiv:2112.03459}, 2021.

\bibitem[Hazan et~al.(2007)Hazan, Agarwal, and Kale]{hazan2007logarithmic}
Elad Hazan, Amit Agarwal, and Satyen Kale.
\newblock Logarithmic regret algorithms for online convex optimization.
\newblock \emph{Machine Learning}, 69:\penalty0 169--192, 2007.

\bibitem[Hendrycks \& Gimpel(2016)Hendrycks and Gimpel]{hendrycks2016gaussian}
Dan Hendrycks and Kevin Gimpel.
\newblock Gaussian error linear units (gelus).
\newblock \emph{arXiv preprint arXiv:1606.08415}, 2016.

\bibitem[Hong \& Lin(2024{\natexlab{a}})Hong and Lin]{hong2024convergence}
Yusu Hong and Junhong Lin.
\newblock On convergence of adam for stochastic optimization under relaxed assumptions.
\newblock \emph{arXiv preprint arXiv:2402.03982}, 2024{\natexlab{a}}.

\bibitem[Hong \& Lin(2024{\natexlab{b}})Hong and Lin]{hong2024revisiting}
Yusu Hong and Junhong Lin.
\newblock Revisiting convergence of adagrad with relaxed assumptions.
\newblock \emph{arXiv preprint arXiv:2402.13794}, 2024{\natexlab{b}}.

\bibitem[Kavis et~al.(2022)Kavis, Levy, and Cevher]{kavis2022high}
Ali Kavis, Kfir~Yehuda Levy, and Volkan Cevher.
\newblock High probability bounds for a class of nonconvex algorithms with adagrad stepsize.
\newblock \emph{arXiv preprint arXiv:2204.02833}, 2022.

\bibitem[Kingma \& Ba(2014)Kingma and Ba]{kingma2014adam}
Diederik~P Kingma and Jimmy Ba.
\newblock Adam: A method for stochastic optimization.
\newblock \emph{arXiv preprint arXiv:1412.6980}, 2014.

\bibitem[Koloskova et~al.(2023)Koloskova, Hendrikx, and Stich]{koloskova2023revisiting}
Anastasia Koloskova, Hadrien Hendrikx, and Sebastian~U Stich.
\newblock Revisiting gradient clipping: Stochastic bias and tight convergence guarantees.
\newblock In \emph{International Conference on Machine Learning}, pp.\  17343--17363. PMLR, 2023.

\bibitem[Kunstner et~al.(2023)Kunstner, Chen, Lavington, and Schmidt]{kunstner2023noise}
Frederik Kunstner, Jacques Chen, Jonathan~Wilder Lavington, and Mark Schmidt.
\newblock Noise is not the main factor behind the gap between sgd and adam on transformers, but sign descent might be.
\newblock \emph{arXiv preprint arXiv:2304.13960}, 2023.

\bibitem[Kushner \& Clark(2012)Kushner and Clark]{kushner2012stochastic}
Harold~Joseph Kushner and Dean~S Clark.
\newblock \emph{Stochastic approximation methods for constrained and unconstrained systems}, volume~26.
\newblock Springer Science \& Business Media, 2012.

\bibitem[Levy et~al.(2018)Levy, Yurtsever, and Cevher]{levy2018online}
Kfir~Y Levy, Alp Yurtsever, and Volkan Cevher.
\newblock Online adaptive methods, universality and acceleration.
\newblock \emph{Advances in neural information processing systems}, 31, 2018.

\bibitem[Li et~al.(2024)Li, Rakhlin, and Jadbabaie]{li2024convergence}
Haochuan Li, Alexander Rakhlin, and Ali Jadbabaie.
\newblock Convergence of adam under relaxed assumptions.
\newblock \emph{Advances in Neural Information Processing Systems}, 36, 2024.

\bibitem[Li \& Orabona(2019)Li and Orabona]{li2019convergence}
Xiaoyu Li and Francesco Orabona.
\newblock On the convergence of stochastic gradient descent with adaptive stepsizes.
\newblock In \emph{The 22nd international conference on artificial intelligence and statistics}, pp.\  983--992. PMLR, 2019.

\bibitem[Li et~al.(2021)Li, Choudhary, Wei, Yuan, Bhushanam, Zhao, and Lan]{li2021frequency}
Yan Li, Dhruv Choudhary, Xiaohan Wei, Baichuan Yuan, Bhargav Bhushanam, Tuo Zhao, and Guanghui Lan.
\newblock Frequency-aware sgd for efficient embedding learning with provable benefits.
\newblock \emph{arXiv preprint arXiv:2110.04844}, 2021.

\bibitem[Liu et~al.(2023{\natexlab{a}})Liu, Li, Hall, Liang, and Ma]{liu2023sophia}
Hong Liu, Zhiyuan Li, David Hall, Percy Liang, and Tengyu Ma.
\newblock Sophia: A scalable stochastic second-order optimizer for language model pre-training.
\newblock \emph{arXiv preprint arXiv:2305.14342}, 2023{\natexlab{a}}.

\bibitem[Liu et~al.(2023{\natexlab{b}})Liu, Nguyen, Nguyen, Ene, and Nguyen]{liu2023high}
Zijian Liu, Ta~Duy Nguyen, Thien~Hang Nguyen, Alina Ene, and Huy Nguyen.
\newblock High probability convergence of stochastic gradient methods.
\newblock In \emph{International Conference on Machine Learning}, pp.\  21884--21914. PMLR, 2023{\natexlab{b}}.

\bibitem[Loshchilov \& Hutter(2016)Loshchilov and Hutter]{loshchilov2016sgdr}
Ilya Loshchilov and Frank Hutter.
\newblock Sgdr: Stochastic gradient descent with warm restarts.
\newblock \emph{arXiv preprint arXiv:1608.03983}, 2016.

\bibitem[Loshchilov \& Hutter(2017)Loshchilov and Hutter]{loshchilov2017decoupled}
Ilya Loshchilov and Frank Hutter.
\newblock Decoupled weight decay regularization.
\newblock \emph{arXiv preprint arXiv:1711.05101}, 2017.

\bibitem[Maladkar et~al.(2024)Maladkar, Jiang, and Mokhtari]{maladkar2024convergence}
Devyani Maladkar, Ruichen Jiang, and Aryan Mokhtari.
\newblock Convergence analysis of adaptive gradient methods under refined smoothness and noise assumptions.
\newblock \emph{arXiv preprint arXiv:2406.04592}, 2024.

\bibitem[Moulines \& Bach(2011)Moulines and Bach]{moulines2011non}
Eric Moulines and Francis Bach.
\newblock Non-asymptotic analysis of stochastic approximation algorithms for machine learning.
\newblock \emph{Advances in neural information processing systems}, 24, 2011.

\bibitem[Muchnik et~al.(2013)Muchnik, Pei, Parra, Reis, Andrade~Jr, Havlin, and Makse]{muchnik2013origins}
Lev Muchnik, Sen Pei, Lucas~C Parra, Saulo~DS Reis, Jos{\'e}~S Andrade~Jr, Shlomo Havlin, and Hern{\'a}n~A Makse.
\newblock Origins of power-law degree distribution in the heterogeneity of human activity in social networks.
\newblock \emph{Scientific reports}, 3\penalty0 (1):\penalty0 1783, 2013.

\bibitem[Mukkamala \& Hein(2017)Mukkamala and Hein]{mukkamala2017variants}
Mahesh~Chandra Mukkamala and Matthias Hein.
\newblock Variants of rmsprop and adagrad with logarithmic regret bounds.
\newblock In \emph{International conference on machine learning}, pp.\  2545--2553. PMLR, 2017.

\bibitem[Nemirovski \& Yudin(1978)Nemirovski and Yudin]{nemirovski1978cezari}
Arkadi Nemirovski and D~Yudin.
\newblock On cezari's convergence of the steepest descent method for approximating saddle point of convex-concave functions.
\newblock In \emph{Soviet Mathematics. Doklady}, volume~19, pp.\  258--269, 1978.

\bibitem[Nemirovski et~al.(2009)Nemirovski, Juditsky, Lan, and Shapiro]{nemirovski2009robust}
Arkadi Nemirovski, Anatoli Juditsky, Guanghui Lan, and Alexander Shapiro.
\newblock Robust stochastic approximation approach to stochastic programming.
\newblock \emph{SIAM Journal on optimization}, 19\penalty0 (4):\penalty0 1574--1609, 2009.

\bibitem[Nemirovskij \& Yudin(1983)Nemirovskij and Yudin]{nemirovskij1983problem}
Arkadij~Semenovi{\v{c}} Nemirovskij and David~Borisovich Yudin.
\newblock Problem complexity and method efficiency in optimization.
\newblock 1983.

\bibitem[Nesterov et~al.(2018)]{nesterov2018lectures}
Yurii Nesterov et~al.
\newblock \emph{Lectures on convex optimization}, volume 137.
\newblock Springer, 2018.

\bibitem[Nguyen et~al.(2019)Nguyen, Nguyen, and van Dijk]{nguyen2019tight}
Phuong\_Ha Nguyen, Lam Nguyen, and Marten van Dijk.
\newblock Tight dimension independent lower bound on the expected convergence rate for diminishing step sizes in sgd.
\newblock \emph{Advances in Neural Information Processing Systems}, 32, 2019.

\bibitem[Orabona(2019)]{orabona2019modern}
Francesco Orabona.
\newblock A modern introduction to online learning.
\newblock \emph{arXiv preprint arXiv:1912.13213}, 2019.

\bibitem[Orabona \& P{\'a}l(2015)Orabona and P{\'a}l]{orabona2015scale}
Francesco Orabona and D{\'a}vid P{\'a}l.
\newblock Scale-free algorithms for online linear optimization.
\newblock In \emph{International Conference on Algorithmic Learning Theory}, pp.\  287--301. Springer, 2015.

\bibitem[Orabona \& P{\'a}l(2018)Orabona and P{\'a}l]{orabona2018scale}
Francesco Orabona and D{\'a}vid P{\'a}l.
\newblock Scale-free online learning.
\newblock \emph{Theoretical Computer Science}, 716:\penalty0 50--69, 2018.

\bibitem[Pan et~al.(2021)Pan, Ye, and Zhang]{pan2021eigencurve}
Rui Pan, Haishan Ye, and Tong Zhang.
\newblock Eigencurve: Optimal learning rate schedule for sgd on quadratic objectives with skewed hessian spectrums.
\newblock \emph{arXiv preprint arXiv:2110.14109}, 2021.

\bibitem[Pan et~al.(2022)Pan, Diao, Chen, and Zhang]{pan2022extremebert}
Rui Pan, Shizhe Diao, Jianlin Chen, and Tong Zhang.
\newblock Extremebert: A toolkit for accelerating pretraining of customized bert.
\newblock \emph{arXiv preprint arXiv:2211.17201}, 2022.

\bibitem[Pan et~al.(2023)Pan, Liu, Wang, and Zhang]{pan2023accelerated}
Rui Pan, Yuxing Liu, Xiaoyu Wang, and Tong Zhang.
\newblock Accelerated convergence of stochastic heavy ball method under anisotropic gradient noise.
\newblock \emph{arXiv preprint arXiv:2312.14567}, 2023.

\bibitem[Qian et~al.(2021)Qian, Wu, Zhuang, Wang, and Xiao]{qian2021understanding}
Jiang Qian, Yuren Wu, Bojin Zhuang, Shaojun Wang, and Jing Xiao.
\newblock Understanding gradient clipping in incremental gradient methods.
\newblock In \emph{International Conference on Artificial Intelligence and Statistics}, pp.\  1504--1512. PMLR, 2021.

\bibitem[Qu et~al.(2022)Qu, Gao, Hinder, Ye, and Zhou]{qu2022optimal}
Zhaonan Qu, Wenzhi Gao, Oliver Hinder, Yinyu Ye, and Zhengyuan Zhou.
\newblock Optimal diagonal preconditioning: Theory and practice.
\newblock \emph{arXiv preprint arXiv:2209.00809}, 2022.

\bibitem[Radford et~al.(2019)Radford, Wu, Child, Luan, Amodei, and Sutskever]{radford2019gpt2}
Alec Radford, Jeff Wu, Rewon Child, David Luan, Dario Amodei, and Ilya Sutskever.
\newblock Language models are unsupervised multitask learners.
\newblock 2019.

\bibitem[Rakhlin et~al.(2011)Rakhlin, Shamir, and Sridharan]{rakhlin2011making}
Alexander Rakhlin, Ohad Shamir, and Karthik Sridharan.
\newblock Making gradient descent optimal for strongly convex stochastic optimization.
\newblock \emph{arXiv preprint arXiv:1109.5647}, 2011.

\bibitem[Reddi et~al.(2019)Reddi, Kale, and Kumar]{reddi2019convergence}
Sashank~J Reddi, Satyen Kale, and Sanjiv Kumar.
\newblock On the convergence of adam and beyond.
\newblock \emph{arXiv preprint arXiv:1904.09237}, 2019.

\bibitem[Robbins \& Monro(1951)Robbins and Monro]{robbins1951stochastic}
Herbert Robbins and Sutton Monro.
\newblock A stochastic approximation method.
\newblock \emph{The annals of mathematical statistics}, pp.\  400--407, 1951.

\bibitem[Sagun et~al.(2016)Sagun, Bottou, and LeCun]{sagun2016eigenvalues}
Levent Sagun, Leon Bottou, and Yann LeCun.
\newblock Eigenvalues of the hessian in deep learning: Singularity and beyond.
\newblock \emph{arXiv preprint arXiv:1611.07476}, 2016.

\bibitem[Shamir \& Zhang(2013)Shamir and Zhang]{shamir2013stochastic}
Ohad Shamir and Tong Zhang.
\newblock Stochastic gradient descent for non-smooth optimization: Convergence results and optimal averaging schemes.
\newblock In \emph{International conference on machine learning}, pp.\  71--79. PMLR, 2013.

\bibitem[Streeter \& McMahan(2010)Streeter and McMahan]{streeter2010less}
Matthew Streeter and H~Brendan McMahan.
\newblock Less regret via online conditioning.
\newblock \emph{arXiv preprint arXiv:1002.4862}, 2010.

\bibitem[Taori et~al.(2023)Taori, Gulrajani, Zhang, Dubois, Li, Guestrin, Liang, and Hashimoto]{alpaca}
Rohan Taori, Ishaan Gulrajani, Tianyi Zhang, Yann Dubois, Xuechen Li, Carlos Guestrin, Percy Liang, and Tatsunori~B. Hashimoto.
\newblock Stanford alpaca: An instruction-following llama model.
\newblock \url{https://github.com/tatsu-lab/stanford_alpaca}, 2023.

\bibitem[Tieleman et~al.(2012)Tieleman, Hinton, et~al.]{tieleman2012lecture}
Tijmen Tieleman, Geoffrey Hinton, et~al.
\newblock Lecture 6.5-rmsprop: Divide the gradient by a running average of its recent magnitude.
\newblock \emph{COURSERA: Neural networks for machine learning}, 4\penalty0 (2):\penalty0 26--31, 2012.

\bibitem[Touvron et~al.()Touvron, Martin, and Stone]{touvron2023llama2}
Hugo Touvron, Louis Martin, and Kevin Stone.
\newblock Llama 2: {Open} {Foundation} and {Fine}-{Tuned} {Chat} {Models}.

\bibitem[Touvron et~al.(2023)Touvron, Lavril, Izacard, Martinet, Lachaux, Lacroix, Rozière, Goyal, Hambro, Azhar, Rodriguez, Joulin, Grave, and Lample]{touvron2023llama}
Hugo Touvron, Thibaut Lavril, Gautier Izacard, Xavier Martinet, Marie-Anne Lachaux, Timothée Lacroix, Baptiste Rozière, Naman Goyal, Eric Hambro, Faisal Azhar, Aurelien Rodriguez, Armand Joulin, Edouard Grave, and Guillaume Lample.
\newblock {LLaMA}: {Open} and {Efficient} {Foundation} {Language} {Models}.
\newblock 2023.
\newblock \doi{10.48550/ARXIV.2302.13971}.
\newblock URL \url{https://arxiv.org/abs/2302.13971}.
\newblock Publisher: arXiv Version Number: 1.

\bibitem[Vani \& Rao(2019)Vani and Rao]{vani2019experimental}
S~Vani and TV~Madhusudhana Rao.
\newblock An experimental approach towards the performance assessment of various optimizers on convolutional neural network.
\newblock In \emph{2019 3rd international conference on trends in electronics and informatics (ICOEI)}, pp.\  331--336. IEEE, 2019.

\bibitem[Vaswani et~al.(2020)Vaswani, Laradji, Kunstner, Meng, Schmidt, and Lacoste-Julien]{vaswani2020adaptive}
Sharan Vaswani, Issam~H Laradji, Frederik Kunstner, Si~Yi Meng, Mark Schmidt, and Simon Lacoste-Julien.
\newblock Adaptive gradient methods converge faster with over-parameterization (and you can do a line-search).
\newblock \emph{arXiv preprint arXiv:2006.06835}, 2020.

\bibitem[Verbraeken et~al.(2020)Verbraeken, Wolting, Katzy, Kloppenburg, Verbelen, and Rellermeyer]{verbraeken2020survey}
Joost Verbraeken, Matthijs Wolting, Jonathan Katzy, Jeroen Kloppenburg, Tim Verbelen, and Jan~S Rellermeyer.
\newblock A survey on distributed machine learning.
\newblock \emph{Acm computing surveys (csur)}, 53\penalty0 (2):\penalty0 1--33, 2020.

\bibitem[Wang et~al.(2022)Wang, Zhang, Zhang, Meng, Ma, Liu, and Chen]{wang2022provable}
Bohan Wang, Yushun Zhang, Huishuai Zhang, Qi~Meng, Zhi-Ming Ma, Tie-Yan Liu, and Wei Chen.
\newblock Provable adaptivity in adam.
\newblock \emph{arXiv preprint arXiv:2208.09900}, 2022.

\bibitem[Wang et~al.(2023)Wang, Zhang, Ma, and Chen]{wang2023convergence}
Bohan Wang, Huishuai Zhang, Zhiming Ma, and Wei Chen.
\newblock Convergence of adagrad for non-convex objectives: Simple proofs and relaxed assumptions.
\newblock In \emph{The Thirty Sixth Annual Conference on Learning Theory}, pp.\  161--190. PMLR, 2023.

\bibitem[Wang et~al.(2024)Wang, Fu, Zhang, Zheng, and Chen]{wang2024closing}
Bohan Wang, Jingwen Fu, Huishuai Zhang, Nanning Zheng, and Wei Chen.
\newblock Closing the gap between the upper bound and lower bound of adam's iteration complexity.
\newblock \emph{Advances in Neural Information Processing Systems}, 36, 2024.

\bibitem[Wang et~al.(2019)Wang, Lu, Tu, and Zhang]{wang2019sadam}
Guanghui Wang, Shiyin Lu, Weiwei Tu, and Lijun Zhang.
\newblock Sadam: A variant of adam for strongly convex functions.
\newblock \emph{arXiv preprint arXiv:1905.02957}, 2019.

\bibitem[Ward et~al.(2020)Ward, Wu, and Bottou]{ward2020adagrad}
Rachel Ward, Xiaoxia Wu, and Leon Bottou.
\newblock Adagrad stepsizes: Sharp convergence over nonconvex landscapes.
\newblock \emph{The Journal of Machine Learning Research}, 21\penalty0 (1):\penalty0 9047--9076, 2020.

\bibitem[Wu et~al.(2018)Wu, Ward, and Bottou]{wu2018wngrad}
Xiaoxia Wu, Rachel Ward, and L{\'e}on Bottou.
\newblock Wngrad: Learn the learning rate in gradient descent.
\newblock \emph{arXiv preprint arXiv:1803.02865}, 2018.

\bibitem[Wu et~al.(2020)Wu, Zhu, Wu, Wang, and Ge]{wu2020dissecting}
Yikai Wu, Xingyu Zhu, Chenwei Wu, Annie Wang, and Rong Ge.
\newblock Dissecting hessian: Understanding common structure of hessian in neural networks.
\newblock \emph{arXiv preprint arXiv:2010.04261}, 2020.

\bibitem[Yin et~al.(2012)Yin, Cui, Li, Yao, and Chen]{yin2012challenging}
Hongzhi Yin, Bin Cui, Jing Li, Junjie Yao, and Chen Chen.
\newblock Challenging the long tail recommendation.
\newblock \emph{arXiv preprint arXiv:1205.6700}, 2012.

\bibitem[You et~al.(2017{\natexlab{a}})You, Gitman, and Ginsburg]{you2017large}
Yang You, Igor Gitman, and Boris Ginsburg.
\newblock Large batch training of convolutional networks.
\newblock \emph{arXiv preprint arXiv:1708.03888}, 2017{\natexlab{a}}.

\bibitem[You et~al.(2017{\natexlab{b}})You, Zhang, Hsieh, Demmel, and Keutzer]{you2017100}
Yang You, Zhao Zhang, C~Hsieh, James Demmel, and Kurt Keutzer.
\newblock 100-epoch imagenet training with alexnet in 24 minutes.
\newblock \emph{arXiv preprint arXiv:1709.05011}, 8, 2017{\natexlab{b}}.

\bibitem[You et~al.(2018)You, Zhang, Hsieh, Demmel, and Keutzer]{you2018imagenet}
Yang You, Zhao Zhang, Cho-Jui Hsieh, James Demmel, and Kurt Keutzer.
\newblock Imagenet training in minutes.
\newblock In \emph{Proceedings of the 47th International Conference on Parallel Processing}, pp.\  1--10, 2018.

\bibitem[You et~al.(2019)You, Li, Reddi, Hseu, Kumar, Bhojanapalli, Song, Demmel, Keutzer, and Hsieh]{you2019large}
Yang You, Jing Li, Sashank Reddi, Jonathan Hseu, Sanjiv Kumar, Srinadh Bhojanapalli, Xiaodan Song, James Demmel, Kurt Keutzer, and Cho-Jui Hsieh.
\newblock Large batch optimization for deep learning: Training bert in 76 minutes.
\newblock \emph{arXiv preprint arXiv:1904.00962}, 2019.

\bibitem[Zeiler(2012)]{zeiler2012adadelta}
Matthew~D Zeiler.
\newblock Adadelta: an adaptive learning rate method.
\newblock \emph{arXiv preprint arXiv:1212.5701}, 2012.

\bibitem[Zhang et~al.(2020{\natexlab{a}})Zhang, Jin, Fang, and Wang]{zhang2020improved}
Bohang Zhang, Jikai Jin, Cong Fang, and Liwei Wang.
\newblock Improved analysis of clipping algorithms for non-convex optimization.
\newblock \emph{Advances in Neural Information Processing Systems}, 33:\penalty0 15511--15521, 2020{\natexlab{a}}.

\bibitem[Zhang et~al.(2019)Zhang, He, Sra, and Jadbabaie]{zhang2019gradient}
Jingzhao Zhang, Tianxing He, Suvrit Sra, and Ali Jadbabaie.
\newblock Why gradient clipping accelerates training: A theoretical justification for adaptivity.
\newblock \emph{arXiv preprint arXiv:1905.11881}, 2019.

\bibitem[Zhang et~al.(2020{\natexlab{b}})Zhang, Karimireddy, Veit, Kim, Reddi, Kumar, and Sra]{zhang2020adaptive}
Jingzhao Zhang, Sai~Praneeth Karimireddy, Andreas Veit, Seungyeon Kim, Sashank Reddi, Sanjiv Kumar, and Suvrit Sra.
\newblock Why are adaptive methods good for attention models?
\newblock \emph{Advances in Neural Information Processing Systems}, 33:\penalty0 15383--15393, 2020{\natexlab{b}}.

\bibitem[Zhang(2023)]{zhang_2023_ltbook}
Tong Zhang.
\newblock \emph{Mathematical Analysis of Machine Learning Algorithms}.
\newblock Cambridge University Press, 2023.
\newblock \doi{10.1017/9781009093057}.

\bibitem[Zhang et~al.(2022)Zhang, Chen, Shi, Sun, and Luo]{zhang2022adam}
Yushun Zhang, Congliang Chen, Naichen Shi, Ruoyu Sun, and Zhi-Quan Luo.
\newblock Adam can converge without any modification on update rules.
\newblock \emph{Advances in neural information processing systems}, 35:\penalty0 28386--28399, 2022.

\bibitem[Zhang et~al.(2024)Zhang, Chen, Ding, Li, Sun, and Luo]{zhang2024transformers}
Yushun Zhang, Congliang Chen, Tian Ding, Ziniu Li, Ruoyu Sun, and Zhi-Quan Luo.
\newblock Why transformers need adam: A hessian perspective.
\newblock \emph{arXiv preprint arXiv:2402.16788}, 2024.

\bibitem[Zhou et~al.(2023)Zhou, Li, Li, Yu, Liu, Wang, Zhang, Ji, Yan, He, et~al.]{zhou2023comprehensive}
Ce~Zhou, Qian Li, Chen Li, Jun Yu, Yixin Liu, Guangjing Wang, Kai Zhang, Cheng Ji, Qiben Yan, Lifang He, et~al.
\newblock A comprehensive survey on pretrained foundation models: A history from bert to chatgpt.
\newblock \emph{arXiv preprint arXiv:2302.09419}, 2023.

\bibitem[Zhuang et~al.(2022)Zhuang, Liu, Cutkosky, and Orabona]{zhuang2022understanding}
Zhenxun Zhuang, Mingrui Liu, Ashok Cutkosky, and Francesco Orabona.
\newblock Understanding adamw through proximal methods and scale-freeness.
\newblock \emph{arXiv preprint arXiv:2202.00089}, 2022.

\bibitem[Zipf(2016)]{zipf2016human}
George~Kingsley Zipf.
\newblock \emph{Human behavior and the principle of least effort: An introduction to human ecology}.
\newblock Ravenio Books, 2016.

\end{thebibliography}
\bibliographystyle{iclr2025_conference}

%%%%%%%%%%%%%%%%%%%%%%%%%%%%%%%%%%%%%%%%%%%%%%%%%%%%%%%%%%%%

\appendix

\allowdisplaybreaks[3]

\section{More Experimental Details}
\label{appendix:exp_details}

\subsection{Convex Case}
For \texttt{a4a}, \texttt{a6a} and \texttt{a9a}, we run 100 epochs of optimization with SGD and Adagrad individually, starting from a uniform distribution initialization $\bw_0 \in \mathcal{U}(-0.05, 0.05)^d$. To find the best hyperparameter, grid searches are conducted for both algorithms, with the search space being initial learning rate $\eta \in \{ 10.0, 1.0, 0.1, 0.01\}$ and learning rate schedules being either constant $\eta_t \equiv \eta$ or inverse square root decay $\eta_t = \eta / \sqrt{t+1}$, the same choices as in our theorems. For large datasets \texttt{real-sim} and \texttt{rcv1.binary}, all settings stay the same, except for the number of epochs 3 and the initialization $\bw_0 \in \mathcal{U}(-1/d, 1/d)^d \times 10^{-2}$.

Since it is generally hard to obtain analytical closed-form solutions for logistic regressions, we run gradient descent for $10^6$ epochs to obtain an approximated optimum $w_*$ for small datasets \texttt{a4a}, \texttt{a6a} and \texttt{a9a}. For larger ones \texttt{real-sim} and \texttt{rcv1.binary}, we run $10^2$ epochs of Adagrad instead, since GD converges much slower in comparison. In addition, for computing $\|\bH\|_2$ in large datasets, we run 10 iterations of power iteration to approximate the largest eigenvalue, which quickly converge to desired precisions $\le 10^{-5}$.

\subsection{Nonconvex Case}

For all experiments, we run 3 epochs of optimization with SGD and Adagrad, which is one of the recommended settings of the Alpaca dataset~\citep{alpaca} (\url{https://github.com/tatsu-lab/stanford_alpaca?tab=readme-ov-file#fine-tuning}). We search the learning rate $\eta \in \{1.0, 10^{-1}, 10^{-2}, 10^{-3}, 10^{-4}, 10^{-5}, 10^{-6} \}$ and report the best result in training loss for both SGD and Adagrad. The maximum sequence length is set to 512, along with the learning rate schedule being set to cosine decay~\citep{loshchilov2016sgdr}. Other settings remain the same as default ones in huggingface \texttt{transformers} (\url{https://github.com/huggingface/transformers}). In all our implementations, we use the version \texttt{transformers==4.38.2}. All experiments are conducted on a single A40 GPU, where gradient accumulation is adopted for batch sizes larger than 128 to reduce memory cost.

In addition to the presented results, we include in Figure~\ref{fig:gpt2-alpaca_loss_curve} the loss curves of SGD and Adagrad under batch sizes 256 and 512. They represent the typical loss decrease tendency of all experiments, where Adagrad converges faster than SGD since the beginning, and converges to a point with smaller loss.

Regarding licenses, the Alpaca dataset is released under Creative Commons Attribution-NonCommercial 4.0 International Public
License (\url{https://github.com/tatsu-lab/stanford_alpaca/blob/main/DATA_LICENSE}), while GPT-2 is released under MIT License (\url{https://huggingface.co/openai-community/gpt2}). The code repository of huggingface `transformers` is released under Apache License 2.0 (\url{https://github.com/huggingface/transformers/blob/main/LICENSE}).

\begin{figure}
\vskip 0.2in
\centering
\includegraphics[width=0.45\linewidth]{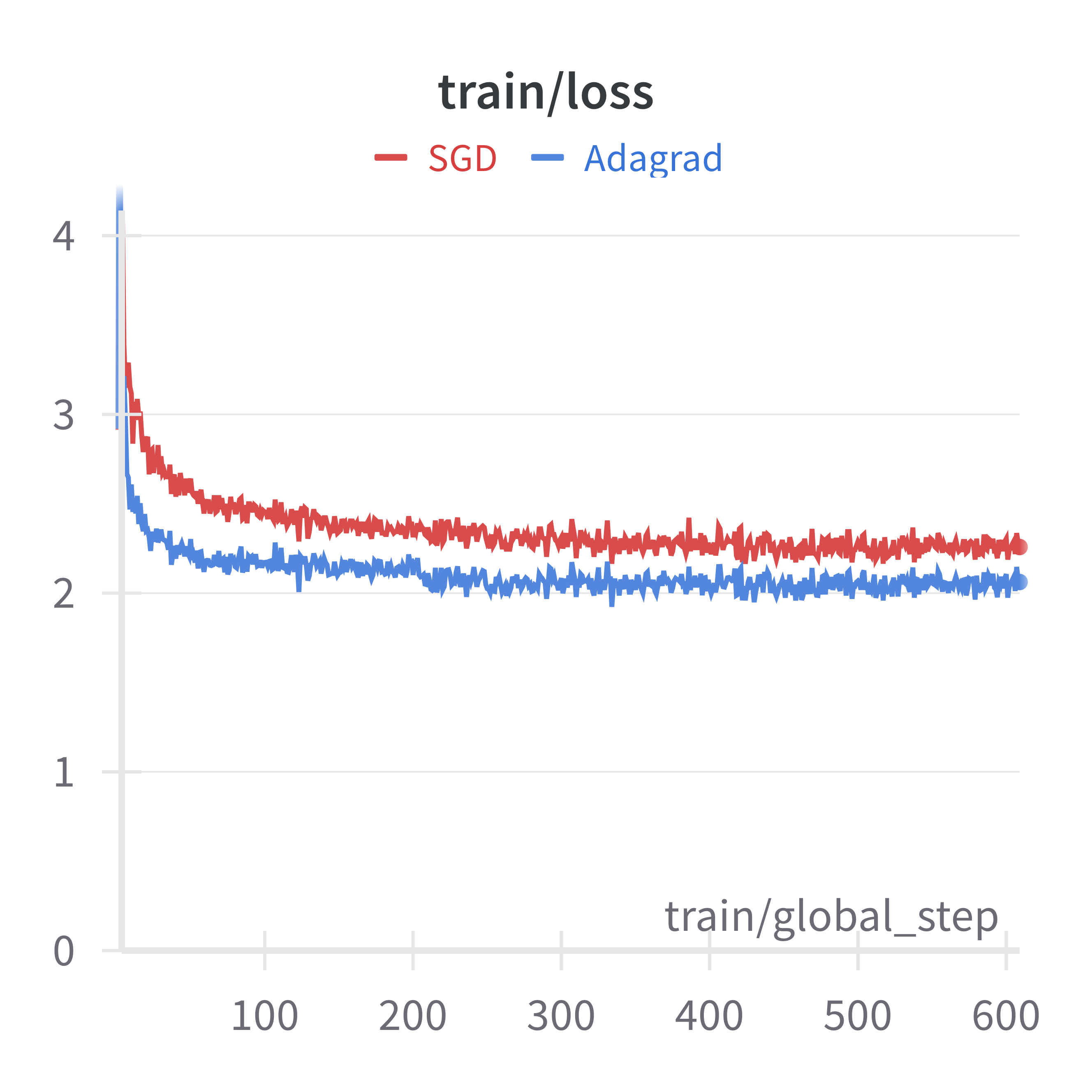}
\includegraphics[width=0.45\linewidth]{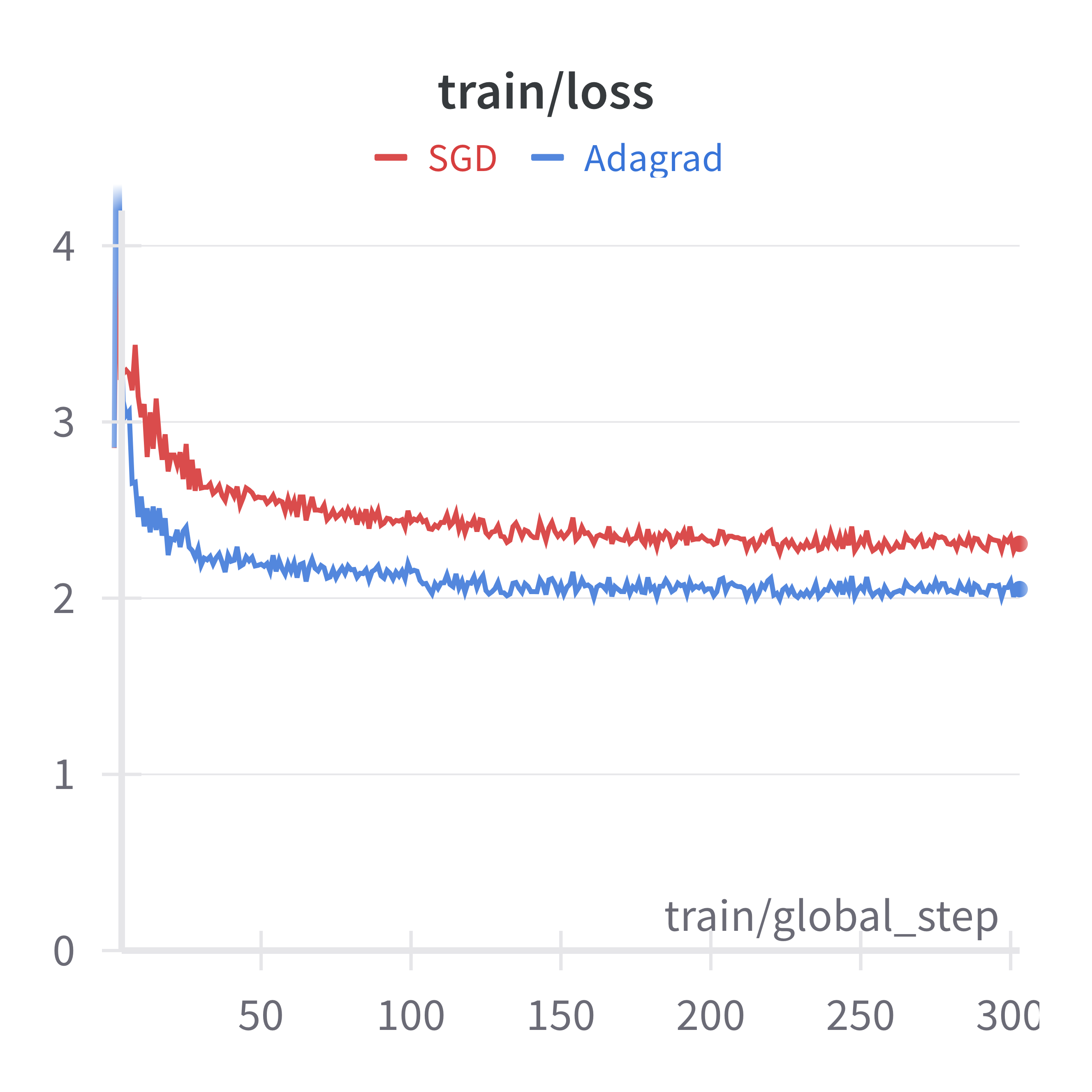}
\caption{Training loss curves of SGD and Adagrad for instruction following tasks on Alpaca with GPT2. \textbf{Left:} batch size 256, \textbf{Right:} batch size 512.}
\label{fig:gpt2-alpaca_loss_curve}
\end{figure}

\section{An Example in the Convex Case}\label{appendix:example_convex}
Following the SVM example in~\citet{duchi2011adaptive}, we provide Example~\ref{example:illustrative_example} to create a concrete case to help our illustration.

\begin{example}\label{example:illustrative_example}
    Consider a finite-sum optimization problem in a hypercube $\cW$:
    \begin{align*}
        f(\bw) = \frac{1}{n} \sum_{i=1}^n \phi_i(\bx_i^\top \bw),
    \end{align*}
    where $\phi_i:\RR\rightarrow\RR$ are convex and twice-differentiable functions and $\bx_i\in \RR^d, \bw\in\cW$. We assume that the first and second order derivatives are bounded such that $\Abs{\phi_i'(\bw)} \le G_1$ and $\phi_i''(\bw) \le G_2$ for all $i$ and $\bw \in \cW$.
    We also assume that the data points $\{\bx_i\}_{i=1}^n$ yield that
    \begin{align*}
        \left[\sum_{j=1}^d \left[{\frac{1}{n} \sum_{i=1}^n \bx_{i,j}^2}\right]^\frac{p}{2}\right]^{\frac{1}{p}} \le M_p,
    \end{align*}
    for $p\ge 1$ and denote $\bH = 1/n\sum_{i=1}^n \bx_i\bx_i^\top$.
    In this case, stochastic gradient yields that
    \begin{align*}
        \nabla_\bw f(\bw;\xi) = \bx_\xi \cdot \phi_\xi'(\bx_\xi^\top \bw) ,
    \end{align*}
    where $\xi$ is uniformly sampled from $\{1,\cdots, n\}$.
\end{example}

Note that Example~\ref{example:illustrative_example} includes a wide range of problems such as linear and logistic regression. 
% It is worth noticing that our results can help explain the 
We can check that the problem in Example~\ref{example:illustrative_example} yields Assumptions of Theorem~\ref{thm:sgd_convex_upper}~and~\ref{thm:smooth_convex_adagrad}. As the stochastic gradients are sampled uniformly with replacement, the variance of gradient noise is reduced by factor $M$, where $M$ is the batch size, as proven in Lemma~\ref{lem:variance_batch}.
We also assume an exponential tail to quantify the imbalance of Hessian in our following example.
\begin{assumption}\label{def:exponential_tail}
    We say the problem has an exponential tail when there exists some constants $\tau > 1/d$ such that
    \begin{align*}
        \bH_{j,j} = \frac{1}{n}\sum_{i=1}^n \bx_{i,j}^2 \propto \exp\left( -\tau j \right)
    \end{align*}
    for all $j=1,\cdots,d$,
    where without loss of generality, we assume that $\bH_{j,j} \ge \bH_{k,k}$ for all $j \le k$.
\end{assumption}
Note that the exponential tail is a typical example of highly imbalanced distributions of data, which is common in natural language modeling~\cite{zipf2016human}, social networks~\cite{muchnik2013origins}, and recommender systems~\cite{yin2012challenging}. 
Then recall the comparison between AdaGrad and SGD in the convex case:
\begin{align*}
    \text{AdaGrad:}& \quad \EE\left[ f(\bar{\bw}_T) - f(\bw_*) \right] = \cO\left( \frac{D_\infty \Norm{\bsigma}_1}{\sqrt{MT}} + \frac{\Norm{\bL}_1 D_\infty^2}{T} \right)  \\
    \text{SGD \& AdaGrad-Norm: }& \quad \EE\left[ f(\bar{\bw}_T) - f(\bw_*) \right] = \cO\left( \frac{D_2 \Norm{\bsigma}_2}{\sqrt{MT}} + \frac{\Norm{\bL}_\infty D_2^2}{T}  \right) 
\end{align*}
and we take the following ratios to denote the difference in bias and variance terms separately.
\begin{align*}
    R_1 \triangleq \frac{D_\infty \Norm{\bsigma}_1}{D_2 \Norm{\bsigma}_2}, \quad R_2 \triangleq \frac{\Norm{\bL}_1 D_\infty^2}{\Norm{\bL}_\infty D_2^2}.
\end{align*}
As we take $\cW$ to be a hypercube, it holds that $D_2^2 = d D_\infty^2$.
Concerning $R_1$, we have
\begin{align*}
    \Norm{\bsigma}_1 \le \max_{\bw\in\cW} \sum_{j=1}^d \sqrt{\frac{1}{n}\sum_{i=1}^n \left[\phi_i'(\bx_i^\top\bw)\right]^2\bx_{i,j}^2} \le G_1M_1
\end{align*}
and similarly,
\begin{align*}
    \Norm{\bsigma}_2 \le \max_{\bw\in\cW} \sqrt{ \frac{1}{n}\sum_{i=1}^n \sum_{j=1}^d \left[\phi_i'(\bx_i^\top\bw)\right]^2\bx_{i,j}^2} \le G_1M_2.
\end{align*}
Therefore, based on the heavy-tailed assumption~(Assumption~\ref{def:exponential_tail}), we can obtain that 
\begin{align*}
    R_1 = \sqrt{ \frac{\phi(\bsigma)}{\phi(D)} } = \frac{M_1 D_\infty}{M_2 D_2} = \frac{M_1}{\sqrt{d} M_2} \propto \frac{\sum_{j=1}^d \exp\left( - \frac{1}{2} \tau j \right)}{\sqrt{d \sum_{j=1}^d \exp\left( -\tau j \right)}} \le \frac{\sqrt{1 - \exp\left( - \tau  \right)}}{1 - \exp\left( - \frac{1}{2}\tau  \right)} \cdot \frac{1}{\sqrt{d}}.
\end{align*}
If we consider $\tau$ to be a mild constant, it should be clear that $R_1$ can be close to $1/\sqrt{d}$ and $R_1 \ll 1$.
On the other side, we have for all $\bw\in\cW$,
\begin{align*}
    \mathbf{0} \preceq \nabla^2 f(\bw) = \frac{1}{n} \sum_{i=1}^n \phi_i''(\bw) \cdot \bx_i\bx_i^\top \preceq G_2 \bH .
\end{align*}
Moreover, that the diagonal matrix $\bL_m=\text{diagonal} (\bL)$ satisfies $\bL_m \succeq G_2\bH$ is a sufficient condition for the smoothness assumption (Assumption~\ref{asm:smooth}). This is equivalent to that $\Norm{\bL_m^{-\frac{1}{2}} \bH \bL_m^{-\frac{1}{2}} }_2 \le 1/G_2$. Thus actually $\bL_m$ can be taken as the optimal solution to an optimization problem:
\begin{align*}
    \min_{\bL \in \RR_+^d} \Norm{\bL}_1 \quad  \text{s.t.} 
    \quad
    \Norm{\bL_m^{-\frac{1}{2}} \bH \bL_m^{-\frac{1}{2}} }_2 \le \frac{1}{G_2} ,
\end{align*}
which shares a similar form with the optimal diagonal preconditioner problem in solving linear systems, where an optimal preconditioner for a fixed matrix $\bH$ can result in much faster convergence as pointed out in~\cite{qu2022optimal}. As an example, when $\bH$ is diagonally dominant, we can choose $\bL_m = \text{diagonal}({2G_2 \bH})$ which takes the diagonal entries. Then we have $\Norm{\bL}_\infty=2 G_2 \bH_{1,1}$ and $\Norm{\bL}_1=2 G_2 \tr{\bH}$. In this case, based on the heavy-tailed assumption~(Assumption~\ref{def:exponential_tail}), we have
\begin{align*}
    R_2 = \frac{\phi(\bL)}{\phi(D)} = \frac{\tr{\bH} D_\infty^2}{\bH_{1,1} D_2^2} = \frac{\tr{\bH} }{d \bH_{1,1} } \propto \frac{\sum_{j=1}^d \exp(-\tau j)}{d \exp(-\frac{1}{2}\tau)}
    \le \frac{1}{1 - \exp\left( - \tau  \right)} \cdot \frac{1}{d} .
\end{align*}
When $\tau$ is a mild constant, it should be clear that $R_2$ can be close to $1/d$ and $R_2 \ll 1$.

To conclude, in this concrete example, we address that $R_1$ can be close to $1/\sqrt{d}$, $R_2$ can be close to $1/d$ and both $R_1,R_2 \ll 1$, showing that AdaGrad has better dimensional dependence than SGD and AdaGrad and the larger gap between bias terms than that of variance terms.

\section{Proof Preliminaries}\label{appendix:proof_preliminaries}
\paragraph{Notations.} In the appendix, we define
\begin{align}\label{eq:define_n_t}
    \bn_t \triangleq \frac{1}{M} \sum_{\xi\in \cB_t} \nabla_\bw f(\bw_t;\xi) - \nabla f(\bw_t) = \frac{1}{M} \sum_{\xi\in \cB_t} \bn(\bw_t;\xi)
\end{align}
to note the gradient noise at iteration $t$.
% Also, we use $\Abs{\bg}_j$ to note the absolute value of the $j$-th coordinate of a vector $\bg$.
We further use $\nabla f_{t,j}$, $\bg_{t.j}$ and $\bn_{t,j}$ to denote the $j$-th coordinate of $\nabla f(\bw_t)$, $\bg_t$ and $\bn_t$, separately.
\begin{lemma}[Projection]\label{lem:proj}
    Suppose $\cW\subseteq \RR^d$ is a closed convex set and $\bLa\in\mathbb{R}^{d\times d}$ is symmetric and positive definite. Then for $\bw\in \mathbb{R}^d$, $\Bar{\bw}\in \Pi_{C}^{\bLa}[\bw]$ if and only if for all $\bz\in \cW$,
    \begin{align}\label{eq:lem_projection_1}
        \frac{1}{2} \Norm{\bw - \bz}_{\bLa}^2 \geq \frac{1}{2} \Norm{\bw - \Bar{\bw}}_{\bLa}^2
    \end{align}
    or equivalently, for all $\bz\in \cW$,
    \begin{align}\label{eq:lem_projection_2}
        \dotprod{\bz - \Bar{\bw}}{\bLa(\bw - \Bar{\bw})} \le 0.
    \end{align}
\end{lemma}
\begin{proof}
    Equation~\eqref{eq:lem_projection_1} simply follows the definition. To prove~\eqref{eq:lem_projection_2}, take any $\bz\in \cW$ and $\alpha\in(0,1)$ we know that $\alpha\bz+(1-\alpha)\Bar{\bw}\in \cW$ as $\cW$ is a convex set. Hence by \eqref{eq:lem_projection_1}, we have
    \begin{align*}
        \frac{1}{2} \Norm{\bw - \Bar{\bw}}_{\bLa}^2 \le& 
        \frac{1}{2} \Norm{\bw - [\alpha\bz+(1-\alpha)\Bar{\bw}]}_{\bLa}^2 \\
        \Longleftrightarrow \quad \frac{1}{2} \Norm{\bw - \Bar{\bw}}_{\bLa}^2 \le& 
        \frac{1}{2} \Norm{(\bw - \Bar{\bw}) - \alpha(\bz-\Bar{\bw})}_{\bLa}^2 \\
        \Longleftrightarrow \quad 0 \le& 
        \frac{\alpha^2}{2} \Norm{\bz-\Bar{\bw}}_{\bLa}^2 - \alpha \dotprod{\bz - \Bar{\bw}}{\bLa(\bw - \Bar{\bw})} \\
        \Longleftrightarrow \quad \dotprod{\bz - \Bar{\bw}}{\bLa(\bw - \Bar{\bw})} \le& 
        \frac{\alpha}{2} \Norm{\bz-\Bar{\bw}}_{\bLa}^2.
    \end{align*}
    As $\alpha$ can be arbitrarily close to $0$, we have that~\eqref{eq:lem_projection_2} holds.
\end{proof}

With the definition of \eqref{eq:define_n_t}, we can obtain the following lemma describing the variance reduced by batch size $M$.
\begin{lemma}[Variance reduced by batch size]\label{lem:variance_batch}
    Under Assumption~\ref{asm:unbiased_gradient}, \ref{asm:anisotropic_noise}, it holds that
    \begin{align}\label{eq:lem_variance_batch}
        \EE_t\left[ \bn_{t,j}^2 \right] \le \frac{\bsigma_j^2}{M} .
    \end{align}
\end{lemma}
\begin{proof}
    It holds that
    \begin{align*}
        \EE_t\left[ \bn_{t,j}^2 \right] =& \EE_t\left[ \left( \frac{1}{M} \sum_{\xi \in \mathcal{B}_t} \bn_j(\bw_t;\xi) \right)^2 \right] 
        \overset{\eqref{eq:asm_unbiased_gradient}}{=} \frac{1}{M^2} \sum_{\xi \in \mathcal{B}_t} \EE_t\left[ (\bn_j(\bw_t;\xi))^2 \right] \overset{\eqref{eq:asm_anisotropic_noise}}{\le} \frac{\bsigma_j^2}{M} ,
    \end{align*}
    where the second equality is based on the independence of $\bn(\bw_t;\xi)$.
\end{proof}

\subsection{Properties of Assumption~\ref{asm:generalized_smooth}}\label{appendix:properties_of_smoothness}
The following is the descent lemma based on Assumption~\ref{asm:generalized_smooth}. We basically follow similar proof techniques to \citet{zhang2019gradient,zhang2020improved,crawshaw2022robustness} but obtain results because of the difference in assumption.
\begin{lemma}[Descent lemma]\label{lem:descent_lemma}
    Under Assumption~\ref{asm:generalized_smooth}, for all $\bw,\bw'\in\cW$ such that $\Norm{\bw - \bw'}_{\bL_1} \le \sqrt{d}$, it holds that
    \begin{align}\label{eq:lem_descent_lemma}
        f(\bw') \le f(\bw) + \dotprod{\nabla f(\bw)}{\bw'-\bw} + \frac{1}{2} \sum_{j=1}^d \left( \bL_{0,j} + \bL_{1,j} \partial_j f(\bw) \right) \Abs{\bw' - \bw}_j^2 ,
    \end{align}
    where $\Abs{\bw'-\bw}_j$ denotes the absolute value of the $j$-th coordinate of $\bw' - \bw$.
\end{lemma}
\begin{proof}
    Denote $\bw_u = u\bw' + (1-u)\bw$ with $u\in[0,1]$. Based on the Taylor expansion, we have
    \begin{align*}
        f(\bw') =& f(\bw) + \int_0^1 \dotprod{\nabla f(\bw_u)}{\bw' - \bw} {\rm d} u \\
        =&
        f(\bw) + \dotprod{\nabla f(\bw)}{\bw' - \bw} + \int_0^1 \dotprod{\nabla f(\bw_u) - \nabla f(\bw)}{\bw' - \bw} {\rm d} u \\
        =&
        f(\bw) + \dotprod{\nabla f(\bw)}{\bw' - \bw} + \int_0^1 \sum_{j=1}^d {(\partial_j f(\bw_u) - \partial_j f(\bw)) (\bw' - \bw)_j } {\rm d} u \\
        \le&
        f(\bw) + \dotprod{\nabla f(\bw)}{\bw' - \bw} \\
        &+ \int_0^1 \sqrt{\sum_{j=1}^d \frac{\Abs{\partial_j f(\bw_u) - \partial_j f(\bw)}^2}{\left( \bL_{0,j} + \bL_{1,j} \Abs{\partial_j f(\bw)} \right)}} \cdot \sqrt{\sum_{j=1}^d \left( \bL_{0,j} + \bL_{1,j} \Abs{\partial_j f(\bw)} \right) \Abs{\bw_u - \bw}_j^2} {\rm d} u \\
        \overset{\eqref{eq:asm_smooth}}{\le}&
        f(\bw) + \dotprod{\nabla f(\bw)}{\bw' - \bw} + \int_0^1 \sum_{j=1}^d \left( \bL_{0,j} + \bL_{1,j} \Abs{\partial_j f(\bw)} \right) \Abs{\bw_u - \bw}_j^2 {\rm d} u \\
        =&
        f(\bw) + \dotprod{\nabla f(\bw)}{\bw' - \bw} + \sum_{j=1}^d \left( \bL_{0,j} + \bL_{1,j} \Abs{\partial_j f(\bw)} \right) \Abs{\bw' - \bw}_j^2 \cdot \int_0^1  u {\rm d} u \\
        =&
        f(\bw) + \dotprod{\nabla f(\bw)}{\bw' - \bw} + \frac{1}{2} \sum_{j=1}^d \left( \bL_{0,j} + \bL_{1,j} \Abs{\partial_j f(\bw)} \right) \Abs{\bw' - \bw}_j^2 ,
    \end{align*}
    where we use Cauchy-Schwarz inequality in the first inequality.
\end{proof}
% One can notice that this lemma, we know Assumption~\ref{asm:generalized_smooth} can imply the anisotropic $\bL$-smoothness used in e.g. \citet{bernstein2018signsgd} by directly setting $\bL_1 = 0$.

\section{Standard Convergence of SGD }\label{sec:sgd}
% \subsection{An Upper Bound}
\begin{algorithm}[tb]
   \caption{SGD}
   \label{alg:sgd}
\begin{algorithmic}[1]
   \STATE {\bfseries Input:} $\bw_0\in\RR^d$, $\{\eta_t\}_{t=0}^{T-1}\in\RR$, and batch size $M\in \mathbb{N}$
   \FOR{$t=0$ {\bfseries to} $T-1$}
   \STATE Sample mini-batch $\cB_t$ with $\Abs{\cB_t} \equiv M$ uniformly
   \STATE $\bg_t = \frac{1}{M} \sum_{\xi\in\cB_t} \nabla_\bw f(\bw_t;\xi)$
   \STATE 
   \textbf{Option I:}
   $\bw_{t+1} = \Pi^{\bI_d}_\cW(\bw_t - \eta_t\bg_t)$
   \STATE 
   \textbf{Option II:}
   $\bw_{t+1} = \bw_t - \eta_t\bg_t$
   \ENDFOR
   \STATE {\bfseries Output:} $1/T\sum_{t=0}^{T-1}\bw_t$
\end{algorithmic}
\end{algorithm}

\begin{algorithm}[tb]
   \caption{Adagrad-Norm}
   \label{alg:adagrad-norm}
\begin{algorithmic}[1]
   \STATE {\bfseries Input:} $\bw_0\in \RR^d$, $\{\eta_t\}_{t=0}^{T-1}\in\RR$, $\epsilon\in\RR$, and batch size $M\in \mathbb{N}$
   \STATE Initialize $v_{-1}= \epsilon^2$
   \FOR{$t=0$ {\bfseries to} $T-1$}
   \STATE Sample mini-batch $\cB_t$ with $\Abs{\cB_t} \equiv M$  uniformly
   \STATE $\bg_t = \frac{1}{M} \sum_{\xi\in\cB_t} \nabla_\bw f(\bw_t;\xi)$
   \STATE $v_t = v_{t-1} + \Norm{\bg_t}_2^2$
   \STATE $b_t = \sqrt{v_t} $
   % \STATE $\bLa_t = \text{diag}(\sqrt{\bv_t}) + \epsilon\bI_d$

   \STATE 
   \textbf{Option I: } 
   $\bw_{t+1} = \Pi^{\bLa_t}_\cW(\bw_t - \frac{\eta_t}{b_t} \bg_t)$
   \STATE 
   \textbf{Option II: } 
   $\bw_{t+1} = \bw_t - \frac{\eta_t}{b_t} \bg_t$
   % \IF{$x_i > x_{i+1}$}
   % \STATE Swap $x_i$ and $x_{i+1}$
   % \STATE $noChange = false$
   % \ENDIF
   \ENDFOR
   \STATE {\bfseries Output:} $1/T\sum_{t=0}^{T-1}\bw_t$
\end{algorithmic}
\end{algorithm}

\begin{theorem}[Convex convergence of SGD]\label{thm:sgd_convex_upper}
    Under Assumption~\ref{asm:convex_set}, \ref{asm:smooth}, \ref{asm:unbiased_gradient}, \ref{asm:anisotropic_noise}, for the sequence $\{\bw_t\}_{t=1}^{T}$ generated by SGD (Algorithm~\ref{alg:sgd} with option I), if we appropriately take step size $\eta_t = \min\left\{ \frac{1}{\Norm{\bL}_\infty}, \sqrt{ \frac{D_2^2 M}{2\Norm{\bsigma}_2^2 (t+1)} } \right\}$, it holds that for $\Bar{\bw}_T \triangleq 1/T\sum_{t=0}^{T-1} \bw_t$,
    \begin{align*}
        \EE\left[ f(\Bar{\bw}_T) - f(\bw_*) \right] 
        = \cO\left( \frac{D_2 \Norm{\bsigma}_2}{\sqrt{MT}} + \frac{\Norm{\bL}_\infty D_2^2}{T} \right) .
    \end{align*}
\end{theorem}

\begin{theorem}[Nonconvex convergence of SGD]\label{thm:sgd_nonconvex_upper}
    Under Assumption~\ref{asm:f_lower_bound}, \ref{asm:smooth}, \ref{asm:unbiased_gradient}, \ref{asm:anisotropic_noise}, for the sequence $\{\bw_t\}_{t=1}^{T}$ generated by SGD (Algorithm~\ref{alg:sgd} with option II), if we appropriately take step size $\eta_t \equiv \eta = \left\{ \frac{1}{2\Norm{\bL}_\infty}, \sqrt{\frac{2(f(\bw_0) - f^*)M}{\Norm{\bL}_\infty \Norm{\bsigma}_2^2 T}} \right\}$, it holds that
    \begin{align*}
        \frac{1}{T} \sum_{t=0}^{T-1} \EE\left[ \Norm{\nabla f(\bw_t)}_2^2 \right] = \cO\left( \frac{\sqrt{\Norm{\bL}_\infty (f(\bw_0) - f^*)} \Norm{\bsigma}_2 }{\sqrt{MT}} + \frac{\Norm{\bL}_\infty(f(\bw_0) - f^*)}{T} \right) .
    \end{align*}
\end{theorem}
Theorem~\ref{thm:sgd_convex_upper} and \ref{thm:sgd_nonconvex_upper} are standard results of the convergence of SGD under $\Norm{\bL}_\infty$-smooth settings. Similar results have been extensively studied (e.g. \citep{orabona2019modern,garrigos2023handbook,ghadimi2013stochastic,bernstein2018signsgd}). We also include the proof below for completeness.

\begin{proof}[Proof of Theorem~\ref{thm:sgd_convex_upper}]
    For $\bw_*\in\cW$, it holds that
    \begin{align*}
        \EE_t\left[\Norm{\bw_{t+1} - \bw_*}_2^2\right] \overset{\eqref{eq:lem_projection_1}}{\le}& \EE_t\left[\Norm{\bw_t - \eta_t\bg_t - \bw_*}_2^2\right] \\
        \le&
        \Norm{\bw_{t} - \bw_*}_2^2 - 2\eta_t\EE_t\left[\dotprod{\bg_t}{\bw_t - \bw_*}\right] + \eta_t^2 \EE_t\left[\Norm{\bg_t}_2^2\right] \\
        =&
        \Norm{\bw_{t} - \bw_*}_2^2 - 2\eta_t\dotprod{\nabla f(\bw_t)}{\bw_t - \bw_*} + \eta_t^2 \Norm{\nabla f(\bw_t)}_2^2 + \eta_t^2\EE_t\left[\Norm{\bn_t}_2^2\right],
    \end{align*}
    where $\bn_t = \bg_t - \nabla f(\bw_t)$.
    With convexity and standard smoothness of $f(\cdot)$, we have
    \begin{align*}
        \dotprod{\nabla f(\bw_t)}{\bw_t - \bw_*} \overset{}{\ge} f(\bw_t) - f(\bw_*) \quad \text{and} \quad \dotprod{\nabla f(\bw_t)}{\bw_t - \bw_*} \overset{}{\ge} \frac{1}{\Norm{\bL}_\infty}\Norm{\nabla f(\bw_t)}_2^2.
    \end{align*}
    Then after rearranging, it holds that
    \begin{align*}
        (2 - \eta_t \Norm{\bL}_\infty)(f(\bw_t) - f(\bw_*)) \le \frac{1}{\eta_t}\Norm{\bw_{t} - \bw_*}_2^2 - \frac{1}{\eta_t}\EE_t\left[\Norm{\bw_{t+1} - \bw_*}_2^2\right] + \eta_t \EE_t\left[\Norm{\bn_t}_2^2\right].
    \end{align*}
    If we take summation and full expectation with $\eta_t \le 1/L$ and $\eta_t\le \eta_{t+1}$ for $t=1,\cdots,T$, it holds that
    \begin{align*}
        \sum_{t=0}^{T-1} \EE\left[ f(\bw_t) - f(\bw_*) \right] \le& \sum_{t=0}^{T-1} \EE\left[\frac{1}{\eta_t} \Norm{\bw_{t} - \bw_*}_2^2 - \frac{1}{\eta_t} \Norm{\bw_{t+1} - \bw_*}_2^2\right] + \sum_{t=0}^{T-1} \eta_t \EE\left[\Norm{\bn_t}_2^2\right] \\
        \overset{\eqref{eq:lem_variance_batch}}{\le}&
        \sum_{t=0}^{T-1} \EE\left[\frac{1}{\eta_t} \Norm{\bw_{t} - \bw_*}_2^2 - \frac{1}{\eta_t} \Norm{\bw_{t+1} - \bw_*}_2^2\right] + \sum_{t=0}^{T-1} \frac{\eta_t \Norm{\bsigma}_2^2}{M} \\
        =&
        \frac{1}{\eta_0} \Norm{\bw_{0} - \bw_*}_2^2 - \frac{1}{\eta_{T-1}} \EE\left[ \Norm{\bw_{T} - \bw_*}_2^2 \right] + \sum_{t=1}^{T-1} \left( \frac{1}{\eta_{t}} - \frac{1}{\eta_{t-1}} \right) \EE\left[ \Norm{\bw_t - \bw_*}_2^2 \right] \\
        &+ \sum_{t=0}^{T-1} \frac{\eta_t \Norm{\bsigma}_2^2}{M} \\
        \overset{\eqref{eq:asm_convex_set}}{\le}&
        \frac{D_2^2}{\eta_0} + \sum_{t=1}^{T-1} \left( \frac{1}{\eta_{t}} - \frac{1}{\eta_{t-1}} \right) D_2^2 + \sum_{t=0}^{T-1} \frac{\eta_t \Norm{\bsigma}_2^2}{M} \\
        =&
        \frac{D_2^2}{\eta_{T-1}} + \sum_{t=0}^{T-1} \frac{\eta_t \Norm{\bsigma}_2^2}{M}.
    \end{align*}
    Therefore, by having both sides of the inequality divided by $T$ and making use of the convexity, we conclude the proof that
    \begin{align*}
        \EE\left[ f(\Bar{\bw}_T) - f(\bw_*) \right] \overset{\text{Asm.}~\ref{asm:convex_set}}{\le} \frac{1}{T} \sum_{t=0}^{T-1} \EE\left[ f(\bw_t) - f(\bw_*) \right]
        \le \frac{D_2^2}{\eta_{T-1}T} + \frac{\Norm{\bsigma}_2^2}{MT}\sum_{t=0}^{T-1} \eta_t ,
    \end{align*}
    where $\Bar{\bw}_T = 1/T\sum_{t=0}^{T-1} \bw_t$.
    Then by taking 
    \begin{align*}
        \eta_t = \min\left\{ \frac{1}{\Norm{\bL}_\infty}, \sqrt{p \frac{D^2_2M}{\sigma^2_2(t+1)}} \right\},
    \end{align*}
    where $p$ is a constant,
    we can obtain that
    \begin{align*}
        \EE\left[ f(\bar{\bw}_T) - f(\bw_*) \right] \le& \frac{D_2^2}{\eta_{T-1}T} + \frac{\Norm{\bsigma}_2^2}{MT} \sum_{t=0}^{t-1} \eta_t \\
        \le&
        \frac{\Norm{\bL}_\infty D_2^2}{T} + \sqrt{\frac{D_2^2\Norm{\bsigma}_2^2}{MT}} \cdot \sqrt{\frac{1}{p}} + \sqrt{\frac{D_2^2\Norm{\bsigma}_2^2}{MT}} \cdot 2\sqrt{p},
    \end{align*}
    where the second inequality holds as
    \begin{align*}
        \sum_{t=0}^{T-1} \frac{1}{\sqrt{t+1}} \le  2\sum_{t=0}^{T-1} \frac{1}{\sqrt{t+1} + \sqrt{t}} = 2\sum_{t=0}^{T-1} \frac{\sqrt{t+1} - \sqrt{t}}{t+1-t} = 2\sum_{t=0}^{T-1} \sqrt{t+1} - \sqrt{t} = 2\sqrt{T}.
    \end{align*}
    Thus by taking $p = \frac{1}{2}$, we finish the proof.
\end{proof}

\begin{proof}[Proof of Theorem~\ref{thm:sgd_nonconvex_upper}]
    Based on the smoothness condition, it holds that
    \begin{align*}
        f(\bw_{t+1}) \le& f(\bw_t) - \dotprod{\nabla f(\bw_t)}{\bw_{t+1} - \bw_t} + \frac{\Norm{\bL}_\infty}{2} \Norm{\bw_{t+1} - \bw_t}_2^2 \\
        =&
        f(\bw_t) - \eta_t \dotprod{\nabla f(\bw_t)}{\bg_t} + \frac{\Norm{\bL}_\infty \eta_t^2}{2} \Norm{\bg_t}_2^2 \\
        \le&
        f(\bw_t) - \eta_t \dotprod{\nabla f(\bw_t)}{\bg_t} + \Norm{\bL}_\infty \eta_t^2 \Norm{\nabla f(\bw_t)}_2^2 + \Norm{\bL}_\infty \eta_t^2 \Norm{\bn_t}_2^2.
    \end{align*}
    By taking expectation, we have
    \begin{align*}
        \EE\left[ f(\bw_{t+1}) \right] \le& 
        \EE\left[ f(\bw_t) \right] - \eta_t \EE\left[ \dotprod{\nabla f(\bw_t)}{\bg_t} \right] + \frac{\Norm{\bL}_\infty \eta_t^2}{2} \EE\left[ \Norm{\bg_t}_2^2 \right] \\
        =&
        \EE\left[ f(\bw_t) \right] - \eta_t \EE\left[ \Norm{ \nabla f(\bw_t)}_2^2 \right] + \frac{\Norm{\bL}_\infty \eta_t^2}{2} \EE\left[ \Norm{\nabla f(\bw_t)}_2^2 \right] + \frac{\Norm{\bL}_\infty \eta_t^2}{2} \EE\left[ \Norm{\bn_t}_2^2 \right] \\
        \overset{\eqref{eq:lem_variance_batch}}{\le}&
        \EE\left[ f(\bw_t) \right] - \eta_t \EE\left[ \Norm{ \nabla f(\bw_t)}_2^2 \right] + \frac{\Norm{\bL}_\infty \eta_t^2}{2} \EE\left[ \Norm{\nabla f(\bw_t)}_2^2 \right] + \frac{\Norm{\bL}_\infty \eta_t^2 \Norm{\bsigma}_2^2}{2M} \\
        \le&
        \EE\left[ f(\bw_t) \right] - \frac{\eta_t}{2} \EE\left[ \Norm{ \nabla f(\bw_t)}_2^2 \right] + \frac{\Norm{\bL}_\infty \eta_t^2 \Norm{\bsigma}_2^2}{2M} ,
    \end{align*}
    where $\bn_t = \nabla f(\bw_t) - \bg_t$ and the last inequality holds as we take $\eta_t \le 1/\Norm{\bL}_\infty$. As $\eta_t \equiv \eta$, we have after rearranging that
    \begin{align*}
        \frac{1}{T}\sum_{t=0}^{T-1} \EE\left[ \Norm{ \nabla f(\bw_t)}_2^2 \right] \le& 
        \frac{2 \EE[f(\bw_0) - f(\bw_T)]}{\eta T} + \frac{\Norm{\bL}_\infty \eta \Norm{\bsigma}_2^2}{M} \\
        \le&
        \frac{2 (f(\bw_0) - f^*)}{\eta T} + \frac{\Norm{\bL}_\infty \eta \Norm{\bsigma}_2^2}{M} .
    \end{align*}
    Then by taking $\eta = \min\left\{ \frac{1}{\Norm{\bL}_\infty}, \sqrt{\frac{(f(\bw_0) - f^*) M}{\Norm{\bL}_\infty \Norm{\bsigma}_2^2 T}} \right\}$, we finish the proof.
    % TODO: check the improved bound
\end{proof}

\section{Proof of Theorem~\ref{thm:smooth_convex_adagrad}}\label{appendix:convex_adagrad}
To prove Theorem~\ref{thm:smooth_convex_adagrad}, we first look at the standard AdaGrad convergence under the non-smooth convex stochastic optimization setting. 
\begin{theorem}[Convergence for convex nonsmooth AdaGrad]\label{thm:convex_adagrad}
    Under Assumption~\ref{asm:convex_set}, \ref{asm:unbiased_gradient}, for the sequence $\{\bw_t\}_{t=1}^T$ generated by Algorithm~\ref{alg:adagrad} with constant step size $\eta_t\equiv\eta$,
    it holds that
    \begin{equation}\label{eq:thm_convex_adagrad}
    \begin{aligned}
        \sum_{t=0}^{T-1} \EE[f(\bw_t)] - f(\bw_*) \le& \frac{1}{\sqrt{2}}\left( \eta + \frac{D_\infty^2}{\eta} \right) 
        \sum_{j=1}^d \EE\left[  \sqrt{\sum_{t=0}^{T-1} \nabla f_{t,j}^2  } + \sqrt{\sum_{t=0}^{T-1}\bn_{t,j}^2} \right] + \frac{\epsilon D_2^2}{2\eta}
    \end{aligned}
    \end{equation}
\end{theorem}
\begin{proof}
    This proof is a stochastic optimization version of the proof of AdaGrad in the online convex learning scheme~\cite{duchi2011adaptive,streeter2010less,zhang_2023_ltbook}. We also include the proof here for completeness.

    First, we give an important result that the gradient norm can be expressed as~$\bLa_t-\bLa_{t-1}$.
    \begin{align}\label{eq:proof_thm_convex_adagrad_norm_trace}
        \Norm{\bg_t}_{\bLa_t^{-1}}^2 &= \sum_{j=1}^d \frac{\bg_{t,j}^2}{\bla_{t,j}} \le \sum_{j=1}^d \frac{\bg_{t,j}^2}{\bla_{t,j}+\bla_{t-1,j}} 
        = \sum_{j=1}^d \frac{\bla_{t,j}^2-\bla_{t-1,j}^2}{\bla_{t,j}+\bla_{t-1,j}} = \tr{\bLa_{t}-\bLa_{t-1}},
    \end{align}
    where we note $\bla_{t,j}$ is the $j$-th coordinate of $\bLa_t$. Then we start the proof.
    \begin{align*}
        \EE_{t}\left[ \Norm{\bw_{t+1} - \bw_*}^2_{\bLa_t} \right] \overset{\eqref{eq:lem_projection_1}}{\le}& \EE\left[ \Norm{\bw_t - \bw_* - \eta_t\bLa_t^{-1}\bg_t}_{\bLa_t}^2 \right] \\
        =&
        \EE_t\left[ \Norm{\bw_t - \bw_*}_{\bLa_t}^2 - 2\eta_t \dotprod{\bg_t}{\bw_t - \bw_*} + \eta_t^2\Norm{\bg_t}_{\bLa_t^{-1}}^2 \right] \\
        =&
        \EE_t\left[ \Norm{\bw_t - \bw_*}_{\bLa_{t}}^2 \right] - 2\eta_t \dotprod{\nabla f(\bw_t)}{\bw_t - \bw_*} + \EE_t\left[ \eta_t^2\Norm{\bg_t}_{\bLa_t^{-1}}^2 \right] \\
        \overset{\eqref{eq:proof_thm_convex_adagrad_norm_trace}}{\le}&
        \EE_t\left[ \Norm{\bw_t - \bw_*}_{\bLa_{t}}^2 \right] - 2\eta_t \dotprod{\nabla f(\bw_t)}{\bw_t - \bw_*} + \EE_t\left[ \eta_t^2\tr{\bLa_{t}-\bLa_{t-1}} \right] \\
        \overset{}{\le}&
        \Norm{\bw_t - \bw_*}_{\bLa_{t-1}}^2 - 2\eta_t \left( f(\bw_t) - f(\bw_*) \right) + \EE_t\left[ \eta_t^2\tr{\bLa_{t}-\bLa_{t-1}} \right] 
        \\ &+ \EE_t\left[ \Norm{\bw_t - \bw_*}_{\bLa_{t}-\bLa_{t-1}}^2 \right] \\
        \overset{\eqref{eq:asm_convex_set}}{\le}&
        \Norm{\bw_t - \bw_*}_{\bLa_{t-1}}^2 - 2\eta_t \left( f(\bw_t) - f(\bw_*) \right) + \EE_t\left[ \eta_t^2\tr{\bLa_{t}-\bLa_{t-1}} \right] 
        \\ &+ \EE_t\left[ D_{\infty}^2\tr{\bLa_{t}-\bLa_{t-1}} \right],
    \end{align*}
    where the third inequality holds as $f(\cdot)$ is convex.
    After taking $\eta_t=\eta$, summing up, taking full expectation and rearrangement, we can obtain that
    \begin{align*}
        \sum_{t=0}^{T-1} \EE\left[f(\bw_t) - f(\bw_*) \right] \le& \frac{1}{2\eta} \sum_{t=0}^{T-1} \EE\left[ \Norm{\bw_t - \bw_*}_{\bLa_{t-1}}^2 - \Norm{\bw_{t+1} - \bw_*}^2_{\bLa_t} \right] 
        \\ &+ \left( \frac{\eta}{2} + \frac{D_\infty^2}{2\eta} \right) \sum_{t=0}^{T-1} \EE\left[ \tr{\bLa_{t}-\bLa_{t-1}} \right] \\
        =&
        \frac{\epsilon \Norm{\bw_0 - \bw_*}_2^2}{2\eta} - \frac{ \Norm{\bw_{T} - \bw_*}^2_{\bLa_{T-1}}}{2\eta} + \left( \frac{\eta}{2} + \frac{D_\infty^2}{2\eta} \right) \EE\left[ \tr{\bLa_{T-1} - \epsilon\bI_d} \right] \\
        \overset{\eqref{eq:asm_convex_set}}{\le}&
        \frac{\epsilon D_2^2}{2\eta} + \left( \frac{\eta}{2} + \frac{D_\infty^2}{2\eta} \right) \sum_{j=1}^d \EE\left[  \sqrt{\sum_{t=0}^{T-1} \bg_{t,j}^2 } \right] \\
        \le&
        \frac{\epsilon D_2^2}{2\eta} + \sqrt{2}\left( \frac{\eta}{2} + \frac{D_\infty^2}{2\eta} \right) \sum_{j=1}^d \EE\left[  \sqrt{\sum_{t=0}^{T-1}\nabla f_{t,j}^2 } + \sqrt{\sum_{t=0}^{T-1}\bn_{t,j}^2} \right],
    \end{align*}
    where we take $\bLa_{-1}= \epsilon \bI_d$ and the last inequality holds as $\sqrt{x+y} \le \sqrt{x} + \sqrt{y}$ for all $x,y\ge0$.
\end{proof}

Then we consider giving a bound on the noise summation term.
\begin{lemma}[variance bound]\label{lem:variance_bound}
    Under Assumption~\ref{asm:unbiased_gradient} and \ref{asm:anisotropic_noise}, it holds that
    \begin{align*}
        \frac{1}{T}\sum_{j=1}^d \EE\left[ \sqrt{\sum_{t=0}^{T-1}\bn_{t,j}^2} \right] \le \frac{\Norm{\bsigma}_1}{\sqrt{MT}}.
    \end{align*}
\end{lemma}
\begin{proof}
    It holds that
    \begin{align*}
        & \sum_{j=1}^d \EE\left[ \sqrt{\sum_{t=0}^{T-1}\bn_{t,j}^2} \right] \le \sum_{j=1}^d \EE\left[ \sqrt{\sum_{t=0}^{T-1} \EE_t\left[\bn_{t,j}^2\right]} \right]
        \overset{\eqref{eq:lem_variance_batch}}{\le} \sum_{j=1}^d \sqrt{\sum_{t=0}^{T-1} \frac{\bsigma_j^2}{M} }
        = \frac{\sqrt{T}\Norm{\bsigma}_1}{\sqrt{M}},
    \end{align*}
    where the first inequality holds as Jensen's inequality. 
\end{proof}
Then we are ready to prove Theorem~\ref{thm:smooth_convex_adagrad}.

\begin{proof}[Proof of Theorem~\ref{thm:smooth_convex_adagrad}]
    From Theorem~\ref{thm:convex_adagrad}, we can obtain that
    \begin{align*}
        \sum_{t=0}^{T-1} \EE[f(\bw_t)] - f(\bw_*) \overset{\eqref{eq:thm_convex_adagrad}}{\le}& \frac{1}{\sqrt{2}}\left( \eta + \frac{D_\infty^2}{\eta} \right) 
        \sum_{j=1}^d \EE\left[  \sqrt{\sum_{t=0}^{T-1}\nabla f_{t,j}^2 } + \sqrt{\sum_{t=0}^{T-1}\bn_{t,j}^2} \right] + \frac{\epsilon D_2^2}{2\eta},
    \end{align*}
    and in Lemma~\ref{lem:variance_bound} we bound the noise term. Then we consider the bias term.
    It holds that
    \begin{align*}
        & \sum_{j=1}^d \EE\left[ \sqrt{\sum_{t=0}^{T-1} \nabla f_{t,j}^2} \right] \overset{\eqref{eq:lem_useful_inequality}}{\le} \EE\left[ \sqrt{\Norm{\bL}_1 \sum_{j=1}^d \sum_{t=0}^{T-1} \frac{\nabla f_{t,j}^2}{L_j}} \right] \\
        & \le
        \sqrt{ \Norm{\bL}_1 \sum_{t=0}^{T-1} \EE\left[ \sum_{j=1}^d \frac{\nabla f_{t,j}^2}{L_j} \right]}
        =
        \sqrt{ \Norm{\bL}_1 \sum_{t=0}^{T-1} \EE\left[ \Norm{\nabla f(\bw_t)}_{\bL^{-1}}^2 \right]} ,
    \end{align*}
    where the first inequality uses Lemma~\ref{lem:useful_inequality} and the second inequality holds as the Jensen's inequality. 
    Moreover, if we make use of Assumption~\ref{asm:smooth}, we can obtain that
    \begin{align*}
        \frac{1}{2T}{ \sum_{t=0}^{T-1} \EE\left[ \Norm{\nabla f(\bw_t)}_{\bL^{-1}}^2 \right]} \le \frac{1}{T} \sum_{t=0}^{T-1} \EE\left[ f(\bw_t) - f(\bw_*) \right].
    \end{align*}
    Thus if we denote $A=1/T{ \sum_{t=0}^{T-1} \EE\left[ f(\bw_t) - f(\bw_*) \right]}$, then combining~\eqref{eq:thm_convex_adagrad}, there is a simplified expression that
    \begin{align*}
        A - CC_1\sqrt{A} - CC_0 \le 0,
    \end{align*}
    where from Theorem~\ref{thm:convex_adagrad} and Lemma~\ref{lem:variance_bound},
    \begin{align*}
        C = \frac{1}{\sqrt{2}} \left( \eta + \frac{D_{\infty}^2}{\eta} \right) , \quad C_1 = \sqrt{\frac{2\Norm{\bL}_1}{T}}, \quad C_0 = \frac{\Norm{\bsigma}_1}{\sqrt{MT}} + \frac{\epsilon D_2^2}{2\eta T C} .
    \end{align*}
    Then we can solve this inequality by conducting simple derivation that
    \begin{align*}
        \sqrt{A} & \le \frac{1}{2}\left[ CC_1 + \sqrt{C^2C_1^2 + 4CC_0} \right] \\
        \Longrightarrow 
        A & \le \frac{1}{2}\left[ C^2C_1^2 + \left( C^2C_1^2 + 4CC_0 \right) \right] \\
        & =
        C^2C_1^2 + 2CC_0.
    \end{align*}
    If we replace $C,C_0$ and $C_1$ by their corresponding value, we can obtain that
    \begin{align*}
        \frac{1}{T} \sum_{t=0}^{T-1} \EE\left[ f(\bw_t) - f(\bw_*) \right] \le
        \frac{1}{2}\left( \eta + \frac{D_{\infty}^2}{\eta} \right)^2\frac{2\Norm{\bL}_1}{T} + \frac{1}{\sqrt{2}}\left( \eta + \frac{D_{\infty}^2}{\eta} \right) \frac{\Norm{\bsigma}_1}{\sqrt{MT}} + \frac{\epsilon D_2^2}{\eta T} 
    \end{align*}
    If we further take the optimal step size $\eta = D_{\infty}$ based on this bound, we can obtain that
    \begin{align*}
        \frac{1}{T} \sum_{t=0}^{T-1} \EE\left[ f(\bw_t) - f(\bw_*) \right] \le
        \frac{4\Norm{\bL}_1 D_{\infty}^2}{T} + \frac{\sqrt{2}D_{\infty}\Norm{\bsigma}_1}{\sqrt{MT}}
        + \frac{\epsilon D_2^2}{D_\infty T},
    \end{align*}
    which concludes the proof.
\end{proof}

\section{Proof of Theorem~\ref{thm:adagrad_generalized_nonconvex}}\label{appendix:proof_nonconvex}
In this section, we prove the convergence of AdaGrad in the nonconvex generalized smooth setting. Note that by setting $\bL_1 = 0$, the proof below can be directly applied to prove Theorem~\ref{thm:adagrad_nonconvex}.
Let us first give a brief overview. Based on Assumption~\ref{asm:smooth} and Lemma~\ref{lem:descent_lemma}, when we set $\eta_t \equiv \eta \le 1/\Norm{\bL_1}_\infty$, it holds that $\Norm{\bw_{t+1} - \bw_t}_{\bL_1} \le \sqrt{d}$ and
\begin{align*}
    f(\bw_{t+1}) \overset{\eqref{eq:lem_descent_lemma}}{\le}& f(\bw_t) + \dotprod{\nabla f(\bw_t)}{\bw_{t+1} - \bw_t} + \frac{1}{2} \sum_{j=1}^d \left( \bL_{0,j} + \bL_{1,j} \Abs{\nabla f_{t,j}} \right) \Abs{\bw_{t+1} - \bw_t}_j^2 \\
    =&
    f(\bw_t) - \eta\dotprod{\nabla f(\bw_t)}{\bLa_{t}^{-1}\bg_t} + \frac{\eta^2}{2}\sum_{j=1}^d \left( \bL_{0,j} + \bL_{1,j} \Abs{\nabla f_{t,j}} \right) \frac{\bg_{t,j}^2}{\bla_{t,j}}.
\end{align*}
Here a critical problem is it is nontrivial to straightforwardly transfer $\EE_t[\dotprod{\nabla f(\bw_t)}{\bLa_{t}^{-1}\bg_t}]$ to $\EE_t[\Norm{\nabla f(\bw)}]$ as both $\bLa_t$ and $\bg_t$ is relevant with $\bg_t$. To deal with this issue, we consider an auxiliary diagonal matrix $\bLat_t$ with each diagonal entry being
\begin{align}\label{eq:def_blat}
    \blat_{t,j}^2 = \bla_{t-1,j}^2 + \EE_t\left[ \bg_{t,j}^2 \right]
\end{align}
for all $j\in[d]$.
Note that this auxiliary sequence is the same as \citet{defossez2020simple}, but our proof technique is much different and obtains better results. Then we consider 
\begin{align*}
    - \dotprod{\nabla f(\bw_t)}{\bLa_{t}^{-1}\bg_t} =& - \dotprod{\nabla f(\bw_t)}{\bLat_{t}^{-1}\bg_t} + \dotprod{\nabla f(\bw_t)}{\left( \bLat_{t}^{-1} - \bLa_t^{-1} \right)\bg_t} 
\end{align*}
and attempt to bound the second term, which is described in Lemma~\ref{lem:bound_dotprod_error}.

Another problem is how to bound the additional terms introduced by the generalized smoothness, namely, $\sum_{j=1}^d \bL_{1,j} \Abs{\nabla f_{t,j}} \frac{\bg_{t,j}^2}{\bla_{t,j}}$. We use the divide-and-conquer strategy to have this additional term resolved by the existing terms, as described in Lemma~\ref{lem:bound_l1_smooth_term}.

With these issues solved, the final convergence property would be determined by
\begin{align*}
    \sum_{t=0}^{T-1} \EE\left[ \dotprod{\nabla f(\bw_t)}{\bLat_{t}^{-1}\bg_t} \right] ,
\end{align*}
which largely relies on $\EE\left[ \tr{\bLa_t} \right]$ that we bound in Lemma~\ref{lem:bound_bla}. Note that Lemma~\ref{lem:bound_bla} considers a different main line and is important for the proof. Then we are ready to complete the proof of Theorem~\ref{thm:adagrad_generalized_nonconvex}.

\begin{lemma}[Bound of $\dotprod{\nabla f(\bw_t)}{\left( \bLat_t^{-1} - \bLa_t^{-1}\right)\bg_t}$]\label{lem:bound_dotprod_error}
    Under the same setting as Theorem~\ref{thm:adagrad_generalized_nonconvex}, if we take diagonal matrix $\bLat_t$ as defined in \eqref{eq:def_blat},
    it holds that
    \begin{align}\label{eq:lem_bound_dotprod_error}
        \EE_t\left[\dotprod{\nabla f(\bw_t)}{\left( \bLat_t^{-1} - \bLa_t^{-1}\right)\bg_t}\right] \le \sum_{j=1}^d 2\bsigma_j\EE_t\left[ \frac{\bg_{t,j}^2}{\bla_{t,j}^2} \right] + \frac{1}{2} \dotprod{\nabla f(\bw_t)}{\bLat_{t}^{-1}\nabla f(\bw_t)}.
    \end{align}
\end{lemma}
\begin{proof}
    It holds that
    \begin{equation}\label{eq:proof_lemma_bound_error_1}
    \begin{aligned}
        \dotprod{\nabla f(\bw_t)}{\left( \bLat_t^{-1} - \bLa_t^{-1}\right)\bg_t} =& \sum_{j=1}^d \nabla f_{t,j} \bg_{t,j} \left( \frac{1}{\blat_{t,j}} - \frac{1}{\bla_{t,j}} \right) \\
        =&
        \sum_{j=1}^d \frac{\nabla f_{t,j} \bg_{t,j}\left( \bg_{t,j}^2 - \EE_t\left[ \bg_{t,j}^2 \right] \right)}{\blat_{t,j}\bla_{t,j}(\bla_{t,j}+\blat_{t,j})} \\
        \le&
        \sum_{j=1}^d \frac{\Abs{\nabla f_{t,j}} \Abs{\bg_{t,j}}\Abs{ \bg_{t,j}^2 - \EE_t\left[ \bg_{t,j}^2 \right] }}{\blat_{t,j}\bla_{t,j}(\bla_{t,j}+\blat_{t,j})} \\
        \le&
        \sum_{j=1}^d \frac{\Abs{\nabla f_{t,j}} \Abs{\bg_{t,j}}\Abs{ \Abs{\bg_{t,j}} - \sqrt{\EE_t\left[ \bg_{t,j}^2 \right]} }}{\blat_{t,j}\bla_{t,j}} ,
    \end{aligned}
    \end{equation}
    where in the first inequality we take the properties of absolute values. In the last inequality, we use the fact that
    \begin{align*}
        \blat_{t,j}^2 = \bla_{t-1,j}^2 + \EE_t\left[ \bg_{t,j}^2 \right] \ge \EE_t\left[ \bg_{t,j}^2 \right] \quad \text{and} \quad \bla_{t,j}^2 = \bla_{t-1,j}^2 + \bg_{t,j}^2 \ge \bg_{t,j}^2.
    \end{align*}
    Then we consider an arbitrary coordinate $j\in[d]$. 
    By applying inequality~\eqref{eq:lem_basic_inequality} with
    \begin{align*}
        c = \frac{\EE_t\left[ \left( \Abs{\bg_{t,j}} - \sqrt{\EE_t\left[ \bg_{t,j}^2 \right]} \right)^2 \right] }{\blat_{t,j}}, \quad x = \frac{\Abs{\bg_{t,j}}}{\bla_{t,j}}, \quad y = \frac{\Abs{\Abs{\bg_{t,j}} - \sqrt{\EE_t\left[ \bg_{t,j}^2 \right]}} \Abs{\nabla f_{t,j}}}{\blat_{t,j}},
    \end{align*}
    it holds that
    \begin{align*}
        \EE_t\left[ \frac{\Abs{ \nabla f_{t,j}} \Abs{\bg_{t,j} }\Abs{\Abs{\bg_{t,j}}-\sqrt{\EE_t\left[ \bg_{t,j}^2 \right]}}}{\blat_{t,j}\bla_{t,j}} \right] \overset{\eqref{eq:lem_basic_inequality}}{\le}& 
        \frac{c}{2} \EE_t\left[\frac{\bg_{t,j}^2}{\bla_{t,j}^2}\right] + \frac{1}{2c} \frac{\nabla f_{t,j}^2}{\blat_{t,j}^2} \EE_t\left[ \Abs{\Abs{\bg_{t,j}} - \sqrt{\EE_t\left[ \bg_{t,j}^2 \right]}}^2 \right] \\
        =&
        \frac{ \EE_t\left[ \left( \Abs{\bg_{t,j}} - \sqrt{\EE_t\left[ \bg_{t,j}^2 \right]} \right)^2 \right] }{2\blat_{t,j}} \EE_t\left[\frac{\bg_{t,j}^2}{\bla_{t,j}^2}\right] + \frac{\blat_{t,j}}{2} \frac{\nabla f_{t,j}^2}{\blat_{t,j}^2}  \\
        \le&
        \sqrt{ \EE_t\left[ \left( \Abs{\bg_{t,j}} - \sqrt{\EE_t\left[ \bg_{t,j}^2 \right]} \right)^2 \right] } \EE_t\left[\frac{\bg_{t,j}^2}{\bla_{t,j}^2}\right] + \frac{\nabla f_{t,j}^2}{2\blat_{t,j}}
        \\
        \overset{\eqref{eq:asm_anisotropic_noise}}{\le}& 
        2\bsigma_j \EE_t\left[\frac{\bg_{t,j}^2}{\bla_{t,j}^2}\right] + \frac{\nabla f_{t,j}^2}{2\blat_{t,j}} ,
    \end{align*}
    where in the second inequality we use the fact that
    \begin{align*}
        \EE_t\left[ \left( \Abs{\bg_{t,j}} - \sqrt{\EE_t\left[ \bg_{t,j}^2 \right]} \right)^2 \right] =& 2\EE_t\left[ \bg_{t,j}^2 \right] - 2 \EE_t\left[ \Abs{\bg_{t,j}} \right] \sqrt{\EE_t\left[ \bg_{t,j}^2 \right]} \le 2 \EE_t\left[ \bg_{t,j}^2 \right] \\ \le& 2 \blat_{t,j} \sqrt{ \EE_t\left[ \bg_{t,j}^2 \right] } 
    \end{align*}
    and the last inequality is based on the fact that
    \begin{align*}
        \EE_t\left[ \left( \Abs{\bg_{t,j}} - \sqrt{\EE_t\left[ \bg_{t,j}^2 \right]} \right)^2 \right] =& 2\EE_t\left[ \bg_{t,j}^2 \right] - 2 \EE_t\left[ \Abs{\bg_{t,j}} \right] \sqrt{\EE_t\left[ \bg_{t,j}^2 \right]} \\
        =&
        2\sqrt{\EE_t\left[ \bg_{t,j}^2 \right]} \left( \sqrt{\EE_t\left[ \bg_{t,j}^2 \right]} - \EE_t\left[ \Abs{\bg_{t,j}} \right] \right) \\
        =&
        2\sqrt{\EE_t\left[ \bg_{t,j}^2 \right]} \frac{\EE_t\left[ \bg_{t,j}^2 \right] - \EE_t\left[ \Abs{\bg_{t,j}} \right]^2}{\sqrt{\EE_t\left[ \bg_{t,j}^2 \right]} + \EE_t\left[ \Abs{\bg_{t,j}} \right]} \\
        \le&
        2\sqrt{\EE_t\left[ \bg_{t,j}^2 \right]} \frac{\EE_t\left[ \bg_{t,j}^2 \right] - \EE_t\left[ {\bg_{t,j}} \right]^2}{\sqrt{\EE_t\left[ \bg_{t,j}^2 \right]} + \EE_t\left[ \Abs{\bg_{t,j}} \right]} \\
        \overset{\eqref{eq:asm_anisotropic_noise}}{\le}&
        \frac{2\sqrt{\EE_t\left[ \bg_{t,j}^2 \right]} \bsigma_j^2}{\sqrt{\EE_t\left[ \bg_{t,j}^2 \right]} + \EE_t\left[ \Abs{\bg_{t,j}} \right]} \le 2\bsigma_j^2 .
    \end{align*}
    Thus by substituting into~\eqref{eq:proof_lemma_bound_error_1}, we can obtain that
    \begin{align*}
        \EE_t\left[ \dotprod{\nabla f(\bw_t)}{\left( \bLat_t^{-1} - \bLa_t^{-1}\right)\bg_t} \right] \le & 
        \EE_t\left[ \sum_{j=1}^d \frac{\Abs{\nabla f_{t,j}} \Abs{\bg_{t,j}}\Abs{ \Abs{\bg_{t,j}} - \sqrt{\EE_t\left[ \bg_{t,j}^2 \right]}}}{\blat_{t,j}\bla_{t,j}} \right] \\
        \le&
        \sum_{j=1}^d 2\bsigma_j\EE_t\left[ \frac{\bg_{t,j}^2}{\bla_{t,j}^2} \right] + \sum_{j=1}^d \frac{\nabla f_{t,j}^2}{2\blat_{t,j}} \\
        =&
        \sum_{j=1}^d 2\bsigma_j\EE_t\left[ \frac{\bg_{t,j}^2}{\bla_{t,j}^2} \right] + \frac{1}{2} \dotprod{\nabla f(\bw_t)}{\bLat_{t}^{-1}\nabla f(\bw_t)},
    \end{align*}
    which concludes the proof.
\end{proof}

\begin{lemma}[Bound of $\bL_{1,j} \Abs{\nabla f_{t,j}} \frac{\bg_{t,j}^2}{\bla_{t,j}^2}$]\label{lem:bound_l1_smooth_term}
    Under the same settings as Theorem~\ref{thm:adagrad_generalized_nonconvex}, it holds that
    \begin{align}\label{eq:lem_bound_l1_smooth_term}
        \EE_t\left[ \sum_{j=1}^d \bL_{1,j} \Abs{\nabla f_{t,j}} \frac{\bg_{t,j}^2}{\bla_{t,j}^2}\right] \le 
        2 \Norm{\bL_1}_\infty \dotprod{\nabla f(\bw_t)}{\bLat_t^{-1} \nabla f(\bw_t)} + \sum_{j=1}^d \bL_{1,j} \bsigma_j \EE_t\left[ \frac{\bg_{t,j}^2}{\bla_{t,j}^2} \right].
    \end{align}
\end{lemma}
\begin{proof}
    Let us first consider an arbitrary coordinate $j\in[d]$ and two cases regarding $\Abs{\nabla f_{t,j}}$. 
    
    \textbf{(1)} If $\Abs{\nabla f_{t,j}} \le \bsigma_j$, we have
    \begin{align*}
        \EE_t\left[ \bL_{1,j} \Abs{\nabla f_{t,j}} \frac{\bg_{t,j}^2}{\bla_{t,j}^2} \right]
        = \bL_{1,j} \Abs{\nabla f_{t,j}} \EE_t\left[ \frac{\bg_{t,j}^2}{\bla_{t,j}^2} \right]
        \le \bL_{1,j} \bsigma_j  \EE_t\left[ \frac{\bg_{t,j}^2}{\bla_{t,j}^2} \right] .
    \end{align*}

    \textbf{(2)} Else, we have $\Abs{\nabla f_{t,j}} \ge \bsigma_j$. In this case, we have
    \begin{align*}
        \EE_t\left[ \frac{\bg_{t,j}^2}{\bla_{t,j}^2}  \right] 
        =& 
        \EE_t\left[ \frac{\bg_{t,j}^2}{\bla_{t-1,j}^2 + \bg_{t,j}^2}  \right] \le \frac{\EE_t\left[ \bg_{t,j}^2 \right]}{\bla_{t-1,j}^2 + \EE_t\left[ \bg_{t,j}^2 \right]}
        \overset{\eqref{eq:asm_anisotropic_noise}}{\le}
        \frac{\nabla f_{t,j}^2 + \bsigma_j^2}{\bla_{t-1,j}^2 + \EE_t\left[ \bg_{t,j}^2 \right]} 
        \le
        \frac{2 \nabla f_{t,j}^2 }{\blat_{t,j}^2} ,
    \end{align*}
    where we use the Jensen inequality of convex function $h(x) = \frac{-x}{a^2 + x}$ in the first inequality. Therefore, 
    \begin{align*}
        \bL_{1,j} \Abs{\nabla f_{t,j}} \EE_t\left[ \frac{\bg_{t,j}^2}{\bla_{t,j}^2} \right] \le 2\bL_{1,j} \Abs{\nabla f_{t,j}} \frac{ \nabla f_{t,j}^2 }{\blat_{t,j}^2} \le
        2\bL_{1,j} \frac{ \nabla f_{t,j}^2 }{\blat_{t,j}},
    \end{align*}
    where in the last inequality we use the fact that $\Abs{\nabla f_{t,j}} = \Abs{\EE[\bg_{t,j}]} \le \sqrt{\EE\left[ \bg_{t,j}^2 \right]} \le \blat_{t,j}$.
    By combining the two cases, we have
    \begin{align*}
        \bL_{1,j} \Abs{\nabla f_{t,j}} \EE_t\left[ \frac{\bg_{t,j}^2}{\bla_{t,j}^2} \right] \le 2\bL_{1,j} \frac{ \nabla f_{t,j}^2 }{\blat_{t,j}} + \bL_{1,j} \bsigma_j  \EE_t\left[ \frac{\bg_{t,j}^2}{\bla_{t,j}^2} \right]
    \end{align*}
    and by summing up, we have
    \begin{align*}
        \EE_t\left[ \sum_{j=1}^d \bL_{1,j} \Abs{\nabla f_{t,j}} \frac{\bg_{t,j}^2}{\bla_{t,j}^2} \right] \le 2 \Norm{\bL_1}_\infty \dotprod{\nabla f(\bw_t)}{\bLat_t^{-1} \nabla f(\bw_t)} + \sum_{j=1}^d \bL_{1,j} \bsigma_j \EE_t\left[ \frac{\bg_{t,j}^2}{\bla_{t,j}^2} \right] .
    \end{align*}
\end{proof}

\begin{lemma}[Bound of $\sum_{t=0}^{T-1} \frac{\bg_{t,j}^2}{\bla_{t,j}^2}$]\label{lem:bound_sum_ratio_bla}
    Under the same settings of Theorem~\ref{thm:adagrad_generalized_nonconvex}, for all $j\in[d]$, it holds
    \begin{align}\label{eq:lem_bound_sum_ratio_bla}
        \EE\left[ \sum_{t=0}^{T-1} \frac{\bg_{t,j}^2}{\bla_{t,j}^2} \right] \le 2\ln\left( \EE\left[ \frac{\tr{\bLa_{T-1}}}{\epsilon} \right] \right) .
    \end{align}
\end{lemma}
\begin{proof}
    It holds that
    \begin{align*}
        \sum_{t=0}^{T-1} \EE\left[ \frac{\bg_{t,j}^2}{\bla_{t,j}^2} \right] \overset{\eqref{eq:lem_sum_ratio}}{\le} \EE\left[ \ln\left( \frac{\bla_{T-1,j}^2}{\epsilon^2} \right) \right] \le \ln\left( 2\EE\left[ \frac{\bla_{T-1,j}}{\epsilon} \right] \right) \le 2\ln\left( \EE\left[ \frac{\tr{\bLa_{T-1}}}{\epsilon} \right] \right) ,
    \end{align*}
    where the first inequality is based on Lemma~\ref{lem:sum_ratio} and the fact that $\bla_{t,j}^2 = \bla_{t-1,j}^2 + \bg_{t,j}^2$, and the second inequality is based on Jensen's inequality.
\end{proof}

\begin{lemma}[Bound of $\tr{\bLa_t}$]\label{lem:bound_bla}
    Under the same settings of Theorem~\ref{thm:adagrad_generalized_nonconvex}, it holds that
    \begin{align}\label{eq:lem_bound_bla}
    \begin{split}
        \EE\left[ \tr{\bLa_{T-1}} \right] \le& 2 d \epsilon + \frac{4}{\eta}(f(\bw_0) - f^*) \\
        &+ 5 \left( \eta \Norm{\bL_0}_1 + 3\sqrt{T} \Norm{\bsigma}_1 \right)
        \ln\left( \frac{2\eta \Norm{\bL_0}_1 + 5\sqrt{T} \Norm{\bsigma}_1}{\epsilon} + {\rm e} \right) .
    \end{split}
    \end{align}
\end{lemma}
\begin{proof}
    As we take $\eta \le 1/\Norm{\bL_1}_\infty$, we have $\Norm{\bw_{t+1} - \bw_t}_{\bL_1} \le \sqrt{d}$. Then from Assumption~\ref{asm:smooth} and Lemma~\ref{lem:descent_lemma} it holds that
    \begin{align*}
        f(\bw_{t+1}) \overset{\eqref{eq:lem_descent_lemma}}{\le}& f(\bw_t) + \dotprod{\nabla f(\bw_t)}{\bw_{t+1} - \bw_t} + \frac{1}{2} \sum_{j=1}^d \left( \bL_{0,j} + \bL_{1,j} \Abs{\nabla f_{t,j}} \right) \Abs{\bw_{t+1} - \bw_t}_j^2 \\
        =&
        f(\bw_t) - \eta\dotprod{\nabla f(\bw_t)}{\bLa_{t}^{-1}\bg_t} + \frac{\eta^2}{2} \sum_{j=1}^d \left( \bL_{0,j} + \bL_{1,j} \Abs{\nabla f_{t,j}} \right) \frac{\bg_{t,j}^2}{\bla_{t,j}^2} \\
        =&
        f(\bw_t) - \eta\dotprod{\bg_t}{\bLa_{t}^{-1}\bg_t} + \frac{\eta^2}{2} \sum_{j=1}^d \left( \bL_{0,j} + \bL_{1,j} \Abs{\nabla f_{t,j}} \right) \frac{\bg_{t,j}^2}{\bla_{t,j}^2} + \eta\dotprod{\bn_t}{\bLa_{t}^{-1}\bg_t} \\
        \le&
        f(\bw_t) - \eta\dotprod{\bg_t}{\bLa_{t}^{-1}\bg_t} + \frac{\eta^2}{2} \sum_{j=1}^d \left( \bL_{0,j} + \bL_{1,j} (\Abs{\bg_{t,j}} + \Abs{\bn_{t,j}}) \right) \frac{\bg_{t,j}^2}{\bla_{t,j}^2} + \eta\dotprod{\bn_t}{\bLa_{t}^{-1}\bg_t} \\
        =&
        f(\bw_t) - \eta\dotprod{\bg_t}{\bLa_{t}^{-1}\bg_t} + \frac{\eta^2}{2} \sum_{j=1}^d \bL_{0,j} \frac{\bg_{t,j}^2}{\bla_{t,j}^2} + \frac{\eta^2}{2} \sum_{j=1}^d \bL_{1,j} \frac{\Abs{\bg_{t,j}}^3}{\bla_{t,j}^2} \\
        &+ \frac{\eta^2}{2} \sum_{j=1}^d \bL_{1,j} \Abs{\bn_{t,j}} \frac{{\bg_{t,j}}^2}{\bla_{t,j}^2} 
        + \eta\dotprod{\bn_t}{\bLa_{t}^{-1}\bg_t} \\
        \le&
        f(\bw_t) - \eta\dotprod{\bg_t}{\bLa_{t}^{-1}\bg_t} + \frac{\eta^2}{2} \sum_{j=1}^d \bL_{0,j} \frac{\bg_{t,j}^2}{\bla_{t,j}^2} + \frac{\eta^2}{2} \sum_{j=1}^d \bL_{1,j} \frac{\Abs{\bg_{t,j}}^3}{\bla_{t,j}^2} \\
        &+ \frac{\eta^2}{2} \sum_{j=1}^d \bL_{1,j} \Abs{\bn_{t,j}} \frac{\bg_{t,j}^2}{\bla_{t,j}^2} 
        + \eta \sum_{j=1}^d \Abs{\bn_{t,j}} \frac{\Abs{\bg_{t,j}}}{\bla_{t,j}},
    \end{align*}
    where we use absolute value inequality.
    Then we deal with the terms separately by
    \begin{align*}
        \frac{\eta^2}{2} \sum_{j=1}^d \bL_{1,j} \frac{\Abs{\bg_{t,j}}^3}{\bla_{t,j}^2} \le& \frac{\eta^2}{2} \Norm{\bL_1}_\infty \sum_{j=1}^d \frac{\Abs{\bg_{t,j}}^3}{\bla_{t,j}^2} \le \frac{\eta^2}{2} \Norm{\bL_1}_\infty \sum_{j=1}^d \frac{\bg_{t,j}^2}{\bla_{t,j}} 
        = \frac{\eta^2}{2} \Norm{\bL_1}_\infty \dotprod{\bg_t}{\bLa_{t}^{-1}\bg_t} \\
        \le&
        \frac{\eta}{2} \dotprod{\bg_t}{\bLa_{t}^{-1}\bg_t} ,
    \end{align*}
    where we use the fact that $\Abs{\bg_{t,j}} \le \bla_{t,j}$ and $\eta \le 1/\Norm{\bL_1}_\infty$.
    Similarly, we can also obtain that
    \begin{align*}
        \frac{\eta^2}{2} \sum_{j=1}^d \bL_{1,j} \Abs{\bn_{t,j}} \frac{\bg_{t,j}^2}{\bla_{t,j}^2} 
        \le \frac{\eta^2}{2} \sum_{j=1}^d \bL_{1,j} \Abs{\bn_{t,j}} \frac{\Abs{\bg_{t,j}}}{\bla_{t,j}} \le 
        \frac{\eta}{2} \sum_{j=1}^d \Abs{\bn_{t,j}} \frac{\Abs{\bg_{t,j}}}{\bla_{t,j}},
    \end{align*}
    where we use the fact that $\Abs{\bg_{t,j}} \le \bla_{t,j}$ and $\eta \le 1/\Norm{\bL_1}_\infty$.
    By combining the bounds of the two terms, we can obtain that
    \begin{align*}
        f(\bw_{t+1}) \le f(\bw_t) - \frac{\eta}{2} \dotprod{\bg_t}{\bLa_t^{-1} \bg_t} + \frac{\eta^2}{2} \sum_{j=1}^d \bL_{0,j} \frac{\bg_{t,j}^2}{\bla_{t,j}^2} + \frac{3 \eta}{2} \sum_{j=1}^d \Abs{\bn_{t,j}} \frac{\Abs{\bg_{t,j}}}{\bla_{t,j}}.
    \end{align*}
    
    Then after summation in $t$ and rearrangement, it holds that
    \begin{align*}
        \sum_{t=0}^{T-1} \frac{\eta}{2} \dotprod{\bg_t}{\bLa_{t}^{-1}\bg_t} \le& \sum_{t=0}^{T-1} [f(\bw_t) - f(\bw_{t+1})] + \frac{\eta^2}{2} \sum_{j=1}^d \bL_{0,j} \sum_{t=0}^{T-1} \frac{\bg_{t,j}^2}{\bla_{t,j}^2} + \frac{3 \eta}{2} \sum_{j=1}^d \sum_{t=0}^{T-1} \Abs{\bn_{t,j}} \frac{\Abs{\bg_{t,j}}}{\bla_{t,j}} \\
        \le&
        f(\bw_0) - f(\bw_T) + \frac{\eta^2}{2} \sum_{j=1}^d \bL_{0,j} \sum_{t=0}^{T-1} \frac{\bg_{t,j}^2}{\bla_{t,j}^2} + \frac{3 \eta}{2} \sum_{j=1}^d \sqrt{\sum_{t=0}^{T-1} \bn_{t,j}^2}  \sqrt{\sum_{t=0}^{T-1} \frac{\bg_{t,j}^2}{\bla_{t,j}^2} },
    \end{align*}
    where in the last inequality we use Cauchy-Schwarz Inequality. 
    Moreover, it holds that
    \begin{align*}
        \sum_{t=0}^{T-1} \dotprod{\bg_t}{\bLa_{t}^{-1}\bg_t} =& \sum_{j=1}^d \sum_{t=0}^{T-1} \frac{\bg_{t,j}^2}{\bla_{t,j}} \ge \sum_{j=1}^d \sum_{t=0}^{T-1} \frac{\bg_{t,j}^2}{\bla_{t,j} + \bla_{t-1,j}} = \sum_{t=0}^{T-1} \frac{\bla_{t,j}^2 - \bla_{t-1,j}^2}{\bla_{t,j} + \bla_{t-1,j}} \\
        =& \sum_{j=1}^d \sum_{t=0}^{T-1} (\bla_{t,j} - \bla_{t-1,j}) 
        = \tr{\bLa_{T-1}-\bLa_{-1}} ,
    \end{align*}
    where $\bLa_{-1} = \epsilon \bI_d$.
    Then by combining the two inequalities together with Lemma~\ref{lem:bound_sum_ratio_bla} and taking expectation, we can obtain that
    \begin{align*}
        \EE\left[\tr{\bLa_{T-1}}\right]
        \le&
        \tr{\bLa_{-1}} + \frac{2}{\eta}\EE[f(\bw_0) - f(\bw_T)] + 2 \eta \sum_{j=1}^d \bL_{0,j} \EE\left[ \sum_{t=0}^{T-1} \frac{\bg_{t,j}^2}{\bla_{t,j}^2} \right] \\
        &+
        3 \sum_{j=1}^d \EE\left[ \sqrt{\sum_{t=0}^{T-1} \bn_{t,j}^2} \cdot \sqrt{\sum_{t=0}^{T-1} \frac{\bg_{t,j}^2}{\bla_{t,j}^2} } \right] \\
        \overset{\eqref{eq:lem_bound_sum_ratio_bla}}{\le}& 
        d \epsilon + \frac{2}{\eta}\EE[f(\bw_0) - f(\bw_T)] + 2\eta \sum_{j=1}^d \bL_{0,j} \ln \left( \EE\left[ \frac{\tr{\bLa_{T-1}}}{\epsilon} \right] \right)  \\
        &+
        3 \sum_{j=1}^d \EE\left[ \sqrt{\sum_{t=0}^{T-1} \bn_{t,j}^2} \cdot \sqrt{\sum_{t=0}^{T-1} \frac{\bg_{t,j}^2}{\bla_{t,j}^2} } \right] \\
        \le&
        d \epsilon + \frac{2}{\eta}\EE[f(\bw_0) - f(\bw_T)] + 2\eta \sum_{j=1}^d \bL_{0,j} \ln \left( \EE\left[ \frac{\tr{\bLa_{T-1}}}{\epsilon} \right] \right) \\
        &+
        3 \sum_{j=1}^d \sqrt{\EE\left[ \sum_{t=0}^{T-1} \bn_{t,j}^2 \right]} \cdot \sqrt{\EE\left[ \sum_{t=0}^{T-1} \frac{\bg_{t,j}^2}{\bla_{t,j}^2} \right]} \\
        \overset{\eqref{eq:asm_anisotropic_noise},\eqref{eq:lem_bound_sum_ratio_bla}}{\le}&
        d \epsilon + \frac{2}{\eta}\EE[f(\bw_0) - f(\bw_T)] + 2\eta \sum_{j=1}^d \bL_{0,j} \ln \left( \EE\left[ \frac{\tr{\bLa_{T-1}}}{\epsilon} \right] \right) \\
        &+
        3 \sqrt{T} \Norm{\bsigma}_1 \cdot \sqrt{ 2 \EE\left[ \ln\left( \frac{\tr{\bLa_{T-1}}}{\epsilon} \right) \right]} \\
        \le&
        d \epsilon + \frac{2}{\eta}\EE[f(\bw_0) - f(\bw_T)] 
        +
        \left( 2\eta \Norm{\bL_0}_1 + 5\sqrt{T} \Norm{\bsigma}_1 \right) \ln\left( \frac{ \EE\left[ \tr{\bLa_{T-1}}^2 \right] }{\epsilon} \right) ,
    \end{align*}
    where the second inequality is based on the Jensen inequality and the third inequality is based on the Cauchy-Schwarz inequality.
    Then by taking 
    \begin{align*}
        x = \frac{\EE[\tr{\bLa_{T-1}}]}{\epsilon}, \quad C_1 = \frac{2\eta \Norm{\bL_0}_1 + 5\sqrt{T} \Norm{\bsigma}_1}{\epsilon}, \quad \text{and} \quad 
        C_0 = \frac{d \epsilon + \frac{1}{\eta}[f(\bw_0) - f(\bw_T)]}{\epsilon} 
    \end{align*}
    in Lemma~\ref{lem:inequality_log}, and using Assumption~\ref{asm:f_lower_bound}, we can obtain that
    \begin{align*}
        \EE\left[ \tr{\bLa_{T-1}} \right] \le& 2 d \epsilon + \frac{4}{\eta}(f(\bw_0) - f^*) \\
        &+ 5 \left( 2\eta \Norm{\bL_0}_1 + 5\sqrt{T} \Norm{\bsigma}_1 \right)
        \ln\left( \frac{2\eta \Norm{\bL_0}_1 + 5\sqrt{T} \Norm{\bsigma}_1}{\epsilon} + {\rm e} \right) ,
    \end{align*}
    which concludes the proof.
\end{proof}

Then we are ready to prove Theorem~\ref{thm:adagrad_generalized_nonconvex}.

\begin{proof}[Proof of Theorem~\ref{thm:adagrad_generalized_nonconvex}]
    As we take $\eta \le 1/\Norm{\bL_1}_\infty$, we have $\Norm{\bw_{t+1} - \bw_t}_{\bL_1} \le \sqrt{d}$. Then from Assumption~\ref{asm:smooth} and Lemma~\ref{lem:descent_lemma} it holds that
    \begin{align*}
        f(\bw_{t+1}) \overset{\eqref{eq:lem_descent_lemma}}{\le}& f(\bw_t) + \dotprod{\nabla f(\bw_t)}{\bw_{t+1} - \bw_t} + \frac{1}{2} \sum_{j=1}^d \left( \bL_{0,j} + \bL_{1,j} \Abs{\nabla f_{t,j}} \right) \Abs{\bw_{t+1} - \bw_t}_j^2 \\
        =&
        f(\bw_t) - \eta\dotprod{\nabla f(\bw_t)}{\bLa_{t}^{-1}\bg_t} + \frac{\eta^2}{2} \sum_{j=1}^d \left( \bL_{0,j} + \bL_{1,j} \Abs{\nabla f_{t,j}} \right) \frac{\bg_{t,j}^2}{\bla_{t,j}^2} .
    \end{align*}
    Then by taking expectation and summation on $t$ we can obtain that
    \begin{align}\label{eq:proof_thm_nonconvex_one_step}
    \begin{split}
        \sum_{t=0}^{T-1} \EE_t[f(\bw_{t+1})] 
        \le&
        \sum_{t=0}^{T-1} f(\bw_t) - \eta \sum_{t=0}^{T-1} \EE_t\left[ \dotprod{\nabla f(\bw_t)}{\bLa_{t}^{-1}\bg_t} \right] + \frac{\eta^2}{2} \sum_{j=1}^d \bL_{0,j} \sum_{t=0}^{T-1} \EE_t\left[ \frac{\bg_{t,j}^2}{\bla_{t,j}^2} \right] \\
        &+ \frac{\eta^2}{2}\sum_{j=1}^d \bL_{1,j} \sum_{t=0}^{T-1} \Abs{\nabla f_{t,j}} \EE_t\left[ \frac{\bg_{t,j}^2}{\bla_{t,j}^2} \right] \\
        \overset{\eqref{eq:lem_bound_l1_smooth_term}}{\le}&
        \sum_{t=0}^{T-1} f(\bw_t) - \eta \sum_{t=0}^{T-1} \EE_t\left[ \dotprod{\nabla f(\bw_t)}{\bLa_{t}^{-1}\bg_t} \right] + \frac{\eta^2}{2} \sum_{j=1}^d \bL_{0,j} \sum_{t=0}^{T-1} \EE_t\left[ \frac{\bg_{t,j}^2}{\bla_{t,j}^2} \right] \\
        &+ \eta^2 \Norm{\bL_1}_\infty \left( \sum_{t=0}^{T-1}  \dotprod{\nabla f(\bw_t)}{\bLat_t^{-1} \nabla f(\bw_t)} + \frac{1}{2}\sum_{j=1}^d \bsigma_j \sum_{t=0}^{T-1}  \EE_t\left[ \frac{\bg_{t,j}^2}{\bla_{t,j}^2} \right] \right). 
    \end{split}
    \end{align}
    We first consider the second term on the right hand side of \eqref{eq:proof_thm_nonconvex_one_step}. 
    It holds that
    \begin{align*}
        -\sum_{t=0}^{T-1} \EE\left[ \dotprod{\nabla f(\bw_t)}{\bLa_{t}^{-1}\bg_t} \right] =&
        -\sum_{t=0}^{T-1} \EE\left[ \dotprod{\nabla f(\bw_t)}{\bLat_{t}^{-1}\bg_t} \right] \\
        &+ \sum_{t=0}^{T-1} \EE\left[ \dotprod{\nabla f(\bw_t)}{\left(\bLat_{t}^{-1} - \bLa_{t}^{-1}\right)\bg_t} \right] \\
        \overset{\eqref{eq:lem_bound_dotprod_error}}{\le}&
        -\frac{1}{2}\sum_{t=0}^{T-1} \EE\left[ \dotprod{\nabla f(\bw_t)}{\bLat_{t}^{-1}\bg_t} \right] +\sum_{t=0}^{T-1}\sum_{j=1}^d 2\bsigma_j\EE\left[ \frac{\bg_{t,j}^2}{\bla_{t,j}^2} \right] 
    \end{align*}
    % We can also bound the gradient ratio summation as the following:
    % \begin{align}\label{eq:proof_thm_nonconvex_sum_ratio}
    %     \sum_{t=0}^{T-1} \EE\left[ \frac{\bg_{t,j}^2}{\bla_{t,j}^2} \right] \overset{\eqref{eq:lem_sum_ratio}}{\le} \EE\left[ \ln\left( \frac{\bla_{T-1,j}}{\epsilon} \right) \right] \le \ln\left( \EE\left[ \frac{\bla_{T-1,j}}{\epsilon} \right] \right) \le \ln\left( \EE\left[ \frac{\tr{\bLa_{T-1}}}{\epsilon} \right] \right) ,
    % \end{align}
    % where the second last inequality is based on Lemma~\ref{lem:sum_ratio} and the fact that $\bla_{t,j}^2 = \bla_{t-1,j}^2 + \bg_{t,j}^2$, and the last inequality is based on Jensen's inequality. 
    Therefore, after substituting the inequality and \eqref{eq:lem_bound_sum_ratio_bla} into \eqref{eq:proof_thm_nonconvex_one_step} and setting $\eta \le \frac{1}{4\Norm{\bL_1}_\infty}$, we have
    % \begin{align*}
    %     \sum_{t=0}^{T-1} \EE\left[ \dotprod{\nabla f(\bw_t)}{\bLa_{t}^{-1}\bg_t} \right] \le& 
    %     \frac{1}{\eta} \sum_{t=0}^{T-1} \EE\left[ f(\bw_t) - f(\bw_{t+1}) \right] + \frac{\eta}{2} \sum_{j=1}^d \bL_{0,j} \sum_{t=0}^{T-1} \EE_t\left[ \frac{\bg_{t,j}^2}{\bla_{t,j}^2} \right] \\
    %     &+ \eta \Norm{\bL_1}_\infty \left( \sum_{t=0}^{T-1}  \dotprod{\nabla f(\bw_t)}{\bLat_t^{-1} \nabla f(\bw_t)} + \frac{1}{2}\sum_{j=1}^d \bsigma_j \sum_{t=0}^{T-1}  \EE_t\left[ \frac{\bg_{t,j}^2}{\bla_{t,j}^2} \right] \right) \\
    %     \le&
        
    % \end{align*}
    
    \begin{align}\label{eq:proof_thm_nonconvex_main}
    \begin{split}
        \sum_{t=0}^{T-1} \EE\left[ \dotprod{\nabla f(\bw_t)}{\bLat_{t}^{-1}\bg_t} \right] \le \frac{4}{\eta} \EE\left[ f(\bw_0) - f(\bw_{T}) \right] + \left( 12\Norm{\bsigma}_1 + 2\Norm{\bL_0}_1 \right) \ln\left( \EE\left[ \frac{(\tr{\bLa_{T-1})^2}}{\epsilon} \right] \right).
    \end{split}
    \end{align}
    Moreover, for the left hand side of \eqref{eq:proof_thm_nonconvex_main}, we have
    \begin{align*}
        \EE\left[ \dotprod{\nabla f(\bw_t)}{\bLat_{t}^{-1}\bg_t} \right] 
        =&
        \EE\left[ \dotprod{\nabla f(\bw_t)}{\bLat_{t}^{-1}\nabla f(\bw_t)} \right] =
        \sum_{j=1}^d \EE\left[ \frac{\Abs{\nabla f_{t,j}}^2}{\sqrt{\bla_{t-1,j}^2 + \EE_t\left[ \bg_{t,j} \right]^2}} \right] \\
        \overset{\eqref{eq:asm_anisotropic_noise}}{\ge}&
        \sum_{j=1}^d \EE\left[ \frac{\Abs{\nabla f_{t,j}}^2}{\sqrt{\bla_{t-1,j}^2 + \nabla f_{t,j}^2 + \bsigma_j^2}} \right] \\
        \ge&
        \sum_{j=1}^d \EE\left[ \frac{\Abs{\nabla f_{t,j}}^2}{\sqrt{\bla_{T-1,j}^2 + \sum_{s=0}^{T-1} \nabla f_{s,j}^2 + \bsigma_j^2}} \right] .
    \end{align*}
    Thus by taking summation we have
    \begin{align}\label{eq:proof_thm_nonconvex_frac_left}
    \begin{split}
        \sum_{t=0}^{T-1} \EE\left[ \dotprod{\nabla f(\bw_t)}{\bLat_{t}^{-1}\bg_t} \right] 
        \ge&
        \sum_{j=1}^d \sum_{t=0}^{T-1} \EE\left[ \frac{\Abs{\nabla f_{t,j}}^2}{\sqrt{\bla_{T-1,j}^2 + \sum_{s=0}^{T-1} \nabla f_{s,j}^2 + \bsigma_j^2}} \right] \\
        =& \sum_{j=1}^d \EE\left[ \frac{\sum_{t=0}^{T-1} \Abs{\nabla f_{t,j}}^2}{\sqrt{\bla_{T-1,j}^2 + \sum_{s=0}^{T-1} \nabla f_{s,j}^2 + \bsigma_j^2}} \right] \\
        \ge&
        \sum_{j=1}^d \frac{\EE\left[ \sqrt{\sum_{t=0}^{T-1} \Abs{\nabla f_{t,j}}^2} \right]^2}{\EE\left[ \sqrt{\bla_{T-1,j}^2 + \sum_{s=0}^{T-1} \nabla f_{s,j}^2 + \bsigma_j^2} \right]} \\
        \ge&
        \frac{\EE\left[ \sum_{j=1}^d \sqrt{\sum_{t=0}^{T-1} \Abs{\nabla f_{t,j}}^2} \right]^2}{\sum_{j=1}^d \EE\left[ \sqrt{\bla_{T-1,j}^2 + \sum_{s=0}^{T-1} \nabla f_{s,j}^2 + \bsigma_j^2} \right]},
    \end{split}
    \end{align}
    where in the second last and last inequality we apply Cauchy-Schwarz inequality
    \begin{align*}
        \EE[\Abs{XY}] \le \EE\left[ \Abs{X}^2 \right]^{\frac{1}{2}} \cdot \EE\left[ \Abs{Y}^2 \right]^{\frac{1}{2}} \quad \text{and} \quad 
        \sum_{j=1}^d \Abs{X_jY_j} \le \left( \sum_{j=1}^d X_j^2 \right)^{\frac{1}{2}} \left( \sum_{j=1}^d Y_j^2 \right)^{\frac{1}{2}}.
    \end{align*}
    We can further deal with the denominator such that for all $j \in [d]$,
    \begin{align}\label{eq:proof_thm_nonconvex_denominator}
    \begin{split}
        \EE\left[ \sqrt{\bla_{T-1,j}^2 + \sum_{s=0}^{T-1} \nabla f_{s,j}^2 + \bsigma_j^2} \right] \le& \EE\left[ \sqrt{\sum_{s=0}^{T-1} \bg_{s,j}^2 } \right] + \EE\left[ \sqrt{\sum_{s=0}^{T-1} \nabla f_{s,j}^2 + \bsigma_j^2} \right] \\
        =&
        \EE\left[ \sqrt{\sum_{s=0}^{T-1} \bg_{s,j}^2 } \right] + \EE\left[ \sqrt{\sum_{s=0}^{T-1} \left( \bg_{s,j} - \bn_{s,j} \right)^2 + \bsigma_j^2} \right] \\
        \le&
        \EE\left[ \sqrt{\sum_{s=0}^{T-1} \bg_{s,j}^2 } \right] + \EE\left[ \sqrt{\sum_{s=0}^{T-1} 2\bg_{s,j}^2 + 2\bn_{s,j}^2 + \bsigma_j^2} \right] \\
        \le&
        3 \EE\left[ \sqrt{\sum_{s=0}^{T-1} \bg_{s,j}^2 } \right] + \EE\left[ \sqrt{\sum_{s=0}^{T-1} 2\bn_{s,j}^2 + \bsigma_j^2} \right] \\
        \overset{\eqref{eq:asm_anisotropic_noise}}{\le}&
        3 \EE\left[ \sqrt{\sum_{s=0}^{T-1} \bg_{s,j}^2} \right] + \sqrt{T+1} \bsigma_j = 3\EE\left[ \bla_{t,j} \right] + \sqrt{T+1} \bsigma_j
    \end{split}
    \end{align}
    where the first and the third inequality holds as $\sqrt{a+b} \le \sqrt{a} + \sqrt{b}$ for $a,b>0$ and the second inequality holds as $(a+b)^2 \le 2a^2+2b^2$ for $a,b\in\RR$.
    Therefore, with Lemma~\ref{lem:bound_dotprod_error} bounding $\tr{\bLa_t}$, combining \eqref{eq:proof_thm_nonconvex_main} with \eqref{eq:proof_thm_nonconvex_frac_left} and \eqref{eq:proof_thm_nonconvex_denominator} we can obtain that
    \begin{align*}
        &\EE\left[ \sum_{j=1}^d \sqrt{\sum_{t=0}^{T-1} \Abs{\nabla f_{t,j}}^2} \right]^2 \\
        \overset{\eqref{eq:proof_thm_nonconvex_main},\eqref{eq:proof_thm_nonconvex_frac_left}}{\le}&  \sum_{j=1}^d \EE\left[ \sqrt{\bla_{T-1,j}^2 + \sum_{s=0}^{T-1} \nabla f_{s,j}^2} + \bsigma_j \right] \\
        &\cdot \left( \frac{4}{\eta} \EE [f(\bw_{0}) - f(\bw_T)] + \left(12 \Norm{\bsigma}_1 + 2\eta\Norm{\bL_0}_1\right) \ln\left( \EE\left[ \frac{(\tr{\bLa_{T-1})^2}}{\epsilon} \right] \right) \right) \\
        \overset{\eqref{eq:proof_thm_nonconvex_denominator}}{\le}&
        \left( 3\EE\left[ \tr{\bLa_{T-1}} \right] + \sqrt{T+1}  \Norm{\bsigma}_1 \right) \\
        & \cdot \left( \frac{4}{\eta} \EE [f(\bw_{0}) - f(\bw_T)] + \left(12 \Norm{\bsigma}_1 + 2\eta\Norm{\bL_0}_1\right) \ln\left( \EE\left[ \frac{(\tr{\bLa_{T-1})^2}}{\epsilon} \right] \right) \right) \\
        \overset{\eqref{eq:lem_bound_bla}}{\le}&
        \left( 30 B(\eta) + 6d \epsilon + 76 \sqrt{T+1} V \right) \cdot ( 4B(\eta) + 12V ) ,
    \end{align*}
    where we denote
    \begin{align*}
        &B(\eta) = \frac{1}{\eta} ( f(\bw_{0}) - f^* ) + \eta \Norm{\bL_0}_1 C_{\log}, \quad
        V = \Norm{\bsigma}_1 C_{\log} 
    \end{align*}
    and
    \begin{align*}
        C_{\log} =& \ln\Bigg( 2 d \epsilon + \frac{4}{\eta}(f(\bw_0) - f^*) \\
        &+ 5 \left( 2\eta \Norm{\bL_0}_1 + 5\sqrt{T} \Norm{\bsigma}_1 \right)
        \ln\left( \frac{2\eta \Norm{\bL_0}_1 + 5\sqrt{T} \Norm{\bsigma}_1}{\epsilon} + {\rm e} \right) \Bigg) \\
        =& \cO\left( \log T \right)
    \end{align*}
    based on Lemma~\ref{lem:bound_bla}.
    By taking $\eta = \min\left\{ \frac{1}{4\Norm{\bL_1}_\infty}, \sqrt{\frac{\Norm{\bL_0}_1}{\Delta}} \right\}$, where $\Delta \triangleq f(\bw_0) - f^*$, we can obtain that
    \begin{align*}
        B(\eta) = \Tilde{\cO}\left( \sqrt{\Norm{\bL_0}_1 \Delta} + \Norm{\bL_1}_\infty \Delta \right) .
    \end{align*}
    Then combining the fact proven in Lemma~\ref{lem:comparison_measure} that
    \begin{align*}
        \sum_{j=1}^d \sqrt{\sum_{t=0}^{T-1} \Abs{\nabla f_{t,j}}^2} \ge \sqrt{ \sum_{t=0}^{T-1} \Norm{\nabla f(\bw_t)}_1^2 }
    \end{align*}
    and dividing $T$ in both side, we can obtain that
    \begin{align*}
        \frac{1}{T} \left( \EE\left[ \sqrt{\sum_{t=0}^{t-1} \Norm{\nabla f(\bw_t)}_1^2 } \right] \right)^2  =& \Tilde{\cO}\left( \frac{\sqrt{\Norm{\bL_0}_1 \Delta } \Norm{\bsigma}_1 }{\sqrt{MT}} + \frac{\Norm{\bsigma}_1^2}{M\sqrt{T}} + \frac{ \Norm{\bL_0}_1 \Delta }{T} \right) 
        \\
        &+ \Tilde{\cO}\left( \frac{ \Norm{\bL_1}_\infty \Delta \Norm{\bsigma}_1 }{\sqrt{MT}} + \frac{\Norm{\bL_1}_\infty^2 \Delta^2 }{T} \right)
        \\
        &+ \Tilde{\cO}\left( \frac{d\epsilon \left( \sqrt{\Norm{\bL_0}_1 \Delta}  + \Norm{\bL_1}_\infty \Delta \right) }{T} + \frac{d\epsilon \Norm{\bsigma}_1 }{\sqrt{M}T}\right) ,
    \end{align*}
    which concludes the proof.
\end{proof}

\section{Useful Lemmas}

\begin{lemma}[A useful inequality]\label{lem:useful_inequality}
    Assume a non-negative sequence $\{x_j\}_{j=1}^n$ and a positive sequence $\{s_j\}_{j=1}^n$ with $S=\sum_{i=1}^ns_j$, it holds that
    \begin{align}\label{eq:lem_useful_inequality}
        \frac{1}{S}\sum_{j=1}^n x_j \le \sqrt{\frac{1}{S}\sum_{j=1}^n \frac{x_j^2}{s_j}}.
    \end{align}
    The inequality holds as an equality if and only if for all $i=1,\cdots,n$ and $j=1,\cdots,n$, $$\frac{x_i}{s_i}=\frac{x_j}{s_j}.$$ 
\end{lemma}
\begin{proof}
    We first multiply $S$ by both sides and take the square, then the right-hand side minus the left-hand side will be
    \begin{align*}
        \sum_{j=1}^n \frac{S}{s_j}x_j^2 - \left( \sum_{j=1}^n x_j \right)^2 
        =& \sum_{j=1}^n \sum_{i=1}^n \frac{s_i}{s_j}x_j^2 - \sum_{j=1}^n\sum_{i=1}^n x_ix_j \\
        =&
        \sum_{j\neq i} \frac{s_i}{s_j}x_j^2 - \sum_{j\neq i} x_ix_j \\
        =&
        \sum_{j\neq i} \left( \sqrt{\frac{s_i}{s_j}}x_j - \sqrt{\frac{s_j}{s_i}}x_i \right)^2
        \geq 0,
    \end{align*}
    which concludes the proof. Note that this result is an application of the Cauchy-Schwarz inequality.
\end{proof}

\begin{lemma}[A basic inequality]\label{lem:basic_inequality}
    For $c\ge 0$ and $x,y\in\RR$, it holds that
    \begin{align}\label{eq:lem_basic_inequality}
        \Abs{xy} \le \frac{c}{2} x^2 + \frac{1}{2c} y^2.
    \end{align}
\end{lemma}

\begin{lemma}[sum of ratios with the denominator being the sum of past numerators]\label{lem:sum_ratio}
    Assume a non-negative sequence $(a_n)$ and $\epsilon>0$. We define $b_n=\sum_{i=1}^n a_i$. Then it holds that
    \begin{align}\label{eq:lem_sum_ratio}
        \sum_{t=1}^{N} \frac{a_t}{b_t + \epsilon} \le \ln \left( \frac{b_N+\epsilon}{\epsilon} \right).
    \end{align}
\end{lemma}
\begin{proof}
    It holds that
    \begin{align*}
        \frac{a_t}{b_t + \epsilon} \le& - \ln\left( 1 - \frac{a_t}{\epsilon+ b_t} \right) \\
        =&
        \ln( b_t + \epsilon ) - \ln (b_t-a_t+\epsilon) \\
        =&
        \ln( b_t + \epsilon ) - \ln (b_{t-1} + \epsilon) ,
    \end{align*}
    where the inequality is based on the fact that $x \le \ln(1+x)$ for all $x > -1$ and we set $b_0=0$. Then by summing up for $1$ to $N$ we finish the proof.
\end{proof}

\begin{lemma}[Solving inequality $x\le C_0 + C_1 \ln x$]\label{lem:inequality_log}
    Assume $C_1 \ge C_0 \ge 0$ and $x > 0$. If $x \le C_0 + C_1 \ln x$ (where $\ln$ denotes the natural logarithm), it holds that for $\zeta \ge 5$,
    \begin{align}\label{eq:lem_inequality_log}
        x \le 2C_0 + \zeta C_1 \ln (C_1 + e) .
    \end{align}
\end{lemma}
\begin{proof}
    Denote $g(x) = x - C_1 \ln x - C_0$. Then we have $g$ is a convex function and attains uniform lower bound at $C_1$. For $x\ge C_1$, $g$ is a monotonically increasing function. Therefore, to verify \eqref{eq:lem_inequality_log}, it is equivalent to verify that if $x = 2C_0 + \zeta C_1 \ln (C_1 + e)$, where $\zeta=5$, then $x \ge C_0 + C_1\ln x$. We begin the verification then.

    Denote $z = 2C_0 + \zeta C_1 \ln (C_1 + e)$. Then we have
    \begin{align*}
        z - ( C_0 + C_1 \ln z ) =& z - \left( C_0 + C_1 \ln \left( 2C_0 + \zeta C_1 \ln (C_1 + e) \right) \right) \\
        \ge&
        z - \left( C_0 + \frac{1}{2} \cdot 2C_0 + C_1 \ln\left( \zeta C_1 \ln (C_1 + e) \right) \right) \\
        =&
         C_1 \left( \zeta \ln (C_1 + e) - \ln\left( \zeta C_1 \ln (C_1 + e) \right) \right) ,
    \end{align*}
    where the second inequality is based on the fact that 
    $$ C_1 \ln\left( 2C_0 + \zeta C_1 \ln (C_1 + e) \right) \le \frac{1}{2} C_0 + C_1 \ln\left( \zeta C_1 \ln (C_1 + e) \right) $$
    for $C_0 \ge 0$ and $\zeta \ge 2$.
    Then let us consider function $h(y) = \zeta \ln (y + e) - \ln\left( \zeta y \ln (y + e) \right)$ for $y \ge 0$ and we have $h(1) \ge 0$. We also have 
    \begin{align*}
        h'(y) = \frac{\zeta}{y+e} - \frac{1}{\zeta y \ln(y+e)} \left( \zeta \ln (y + e) + \frac{\zeta y}{y+e} \right) = 
        \frac{\zeta}{y+e} - \frac{1}{y} - \frac{1}{(y+e) \ln (y + e)} ,
    \end{align*}
    which implies that if $y \ge 1$, we have $h'(y) \ge 0$ and thus $h(y) \ge 0$ for $y \ge 1$. For $y \in [0,1)$, it is also straightforward to obtain that $h(y) \ge \zeta - \ln(\zeta \ln(1+e)) \ge 0$. Therefore, we conclude that $h(y) \ge 0$. By substituting $y = C_1$, we have
    \begin{align*}
        z - ( C_0 + C_1 \ln z ) \ge 0
    \end{align*}
    and thus $x \le C_0 + C_1 \ln x$ implies $x \le z$, which finishes the proof.
\end{proof}

\begin{lemma}[comparison of measures]\label{lem:comparison_measure}
    For a sequence $\{a_{t,j}\}$ with $t=1,...,T$ and $j=1,...,d$, it holds that
    \begin{align}\label{eq:lem_comparison_measure}
        \left( \sum_{j=1}^d \sqrt{\sum_{t=1}^T a_{t,j}^2} \right)^2 \ge \sum_{t=1}^T \left( \sum_{j=1}^d \Abs{a_{t,j}} \right)^2.
    \end{align}
\end{lemma}
\begin{proof}
    It holds that
    \begin{align*}
        \left( \sum_{j=1}^d \sqrt{\sum_{t=1}^T a_{t,j}^2} \right)^2 \ge& \sum_{t=1}^T \left( \sum_{j=1}^d \Abs{a_{t,j}} \right)^2 \\
        \Longleftrightarrow\quad 
        \sum_{j_1\neq j_2}^d \sqrt{\sum_{t=1}^T a_{t,j_1}^2} \cdot \sqrt{\sum_{t=1}^T a_{t,j_2}^2} \ge&
        \sum_{t=1}^T \sum_{j_1\neq j_2}^d \Abs{a_{t,j_1}a_{t,j_2}} \\
        \Longleftrightarrow \quad
        \sum_{t=1}^T a_{t,j_1}^2 \cdot \sum_{t=1}^T a_{t,j_2}^2 \ge&
        \left( \sum_{t=1}^T \Abs{a_{t,j_1}a_{t,j_2}} \right)^2, \quad \forall j_1\neq j_2.
    \end{align*}
    The last inequality follows from the Cauchy-Schwarz Inequality. Thus we finish the proof.
\end{proof}

%%%%%%%%%%%%%%%%%%%%%%%%%%%%%%%%%%%%%%%%%%%%%%%%%%%%%%%%%%%%

\end{document}